\documentclass{article}
\usepackage{iclr2027_conference,times}
\iclrfinalcopy

\usepackage{amsmath,amsfonts,bm}

\def\eqref#1{(\ref{#1})}

\def\1{\bm{1}}

\def\eps{{\epsilon}}

\DeclareMathAlphabet{\mathsfit}{\encodingdefault}{\sfdefault}{m}{sl}
\SetMathAlphabet{\mathsfit}{bold}{\encodingdefault}{\sfdefault}{bx}{n}

\newcommand{\R}{\mathbb{R}}

\DeclareMathOperator*{\argmax}{arg\,max}
\DeclareMathOperator*{\argmin}{arg\,min}

\usepackage{hyperref}
\usepackage{url}
\usepackage{amsmath,amssymb,amsthm,mathtools,bm}
\usepackage{booktabs}
\usepackage{graphicx}
\usepackage{microtype}
\usepackage{multirow}
\usepackage{enumitem}
\usepackage{xcolor}
\usepackage{algorithm}
\usepackage{algpseudocode}
\usepackage{array}
\usepackage{subcaption}
\usepackage{placeins}
\usepackage{longtable}
\usepackage{tabularx}
\usepackage{listings}
\usepackage{fancyvrb}

\title{Relational Response Fields:\\ A General Theory of Black-Box LLM Response Consistency and Recovery}

\author{
SONG ZICHEN\\
Sungkyunkwan University\\
\texttt{songzichen894@gmail.com}
}

\newtheorem{theorem}{Theorem}
\newtheorem{lemma}[theorem]{Lemma}
\newtheorem{proposition}[theorem]{Proposition}
\newtheorem{corollary}[theorem]{Corollary}

\theoremstyle{definition}
\newtheorem{definition}[theorem]{Definition}
\newtheorem{assumption}[theorem]{Assumption}
\newtheorem{example}[theorem]{Example}
\theoremstyle{remark}

\newcommand{\cH}{\mathcal{H}}
\newcommand{\cY}{\mathcal{Y}}
\newcommand{\cO}{\mathcal{O}}
\newcommand{\cC}{\mathcal{C}}
\newcommand{\cK}{\mathcal{K}}

\newcommand{\cE}{\mathcal{E}}

\newcommand{\cM}{\mathcal{M}}

\newcommand{\RR}{\mathbb{R}}
\newcommand{\EE}{\mathbb{E}}
\newcommand{\PP}{\mathbb{P}}
\newcommand{\ind}{\mathbb{I}}
\newcommand{\suppg}{\operatorname{supp}_{\mathrm g}}

\newcommand{\diag}{\operatorname{diag}}
\providecommand{\argmin}{\operatorname*{arg\,min}}
\providecommand{\argmax}{\operatorname*{arg\,max}}

\newcommand{\SRRF}{\textsc{Sparse-RRF}}
\newcommand{\gk}{\gamma_k(D,A)}

\newcommand{\ip}[2]{\left\langle #1,#2\right\rangle}

\providecommand{\eps}{\varepsilon}

\newcommand{\PooledReplayRho}{0.449}
\newcommand{\ReplayDeltaAUC}{0.033}
\newcommand{\ReplayGammaCoef}{3.653}
\newcommand{\ReplayPermutationP}{0.00050}
\newcommand{\ReplayCILow}{0.417}
\newcommand{\ReplayCIHigh}{0.481}

\setlist[itemize]{leftmargin=*,topsep=2pt,itemsep=1pt}
\setlist[enumerate]{leftmargin=*,topsep=2pt,itemsep=1pt}
\begin{document}

\maketitle

\begin{abstract}
Black-box language-model reliability is commonly pursued by sampling, prompting, voting, verifying, or iteratively revising individual answers.  We ask a prior question: \emph{what determines whether a collection of black-box responses is recoverable at all?}  We represent responses to typed transformations of a query as a \emph{relational response field} (RRF).  Edge transports encode how valid responses must change under paraphrase, scaling, decomposition, refactoring, or other task symmetries; anchors encode independently trusted evidence such as execution or a verifier.  For relation operator $D$, anchor operator $A$, and at most $k$ corrupted response nodes, we identify $\gamma_k(D,A)$ as the intrinsic difficulty of black-box response recovery.  It is positive exactly when every $k$-node corruption is identifiable; it gives a deterministic stability bound proportional to $1/\gamma_k$; and a matching two-point minimax lower bound shows that no estimator can improve this dependence.  Thus consistency is not truth: relation-only methods are blind to null directions, including shared hallucinations.  We derive sparse field-repair algorithms while separating information-theoretic identifiability from the stronger null-space conditions required by convex optimization.  Controlled theorem tests and black-box mathematics/code experiments evaluate four theory-fixed consequences: consistency--truth separation, anchor phase transitions, redundancy saturation, and cross-model, cross-task prediction of repair difficulty.  The results support $\gamma_k(D,A)$ as a measurable property of a response-recovery instance, rather than a score attached to one repair heuristic.
\end{abstract}

\section{Introduction}
\label{sec:intro}

A black-box large language model (LLM) can be queried but not differentiated, inspected, or retrained.  Reliability methods therefore act at inference time: they alter prompts \citep{zhou2023ape,pryzant2023protegi,yang2024opro}, sample multiple chains and vote \citep{wang2023selfconsistency}, ask the model to critique itself \citep{madaan2023selfrefine,shinn2023reflexion,miao2024selfcheck}, search over reasoning trees \citep{yao2023tree}, or select with learned and executable verifiers \citep{cobbe2021verifiers,chen2021codex,li2022competition}.  Metamorphic tests instead transform an input and check a known relation among outputs \citep{chen1998metamorphic,segura2016survey}.  These approaches appear operationally different, yet they repeatedly create the same latent object: a family of answers connected by known relations. This paper makes that object explicit.  Let $G=(V,E)$ index transformed queries.  The model response at node $i$ is mapped by a black-box parser to $z_i$; the collection $z=(z_i)_{i\in V}$ is a \emph{relational response field}.  An edge $e=(i,j)$ carries a typed transport $T_e$ and asks that a valid field satisfy $z_j=T_ez_i$.  An identity transport represents a paraphrase, a scalar map represents unit conversion, a structured map represents decomposition, and an execution map represents semantics-preserving code transformation.  Stacking the residuals $z_j-T_ez_i$ defines a typed relation defect operator $D$.  Trusted evidence, execution, constraints, or human labels define an anchor operator $A$. The important question is not whether more responses are available.  It is whether the combined relations and anchors can distinguish the true field from a sparsely corrupted one.  

We show that this question has an exact answer:
\begin{equation}
\label{eq:gamma-intro}
\gamma_k(D,A)=
\min_{\substack{h\ne0\\\|h\|_{0,\mathrm g}\le 2k}}
\frac{\sqrt{\|Dh\|_2^2+\|Ah\|_2^2}}{\|h\|_2}.
\end{equation}
The group support counts response nodes, not scalar coordinates.  The factor $2k$ is unavoidable: the difference between two fields, each corrupted on $k$ nodes, can occupy $2k$ nodes.  Equation~\eqref{eq:gamma-intro} is a restricted minimum singular value of the \emph{instance-specific} relation--anchor operator.  Its role in RRF recovery is exact, not metaphorical.

First, $\gamma_k>0$ is necessary and sufficient for uniform identification of every $k$-node corruption.  Second, if observations are perturbed by norm $\eps$, any feasible estimate incurs error at most $2\eps/\gamma_k$.  Third, a two-point construction gives worst-case error at least a constant times $\eps/\gamma_k$ for every estimator.  Hence the inverse margin, rather than model accuracy, query count, graph density, or the objective value of a particular algorithm, is the intrinsic condition number of black-box response recovery.

This viewpoint cleanly separates \emph{consistency} from \emph{truth}.  If $h\in\ker D$, then $z$ and $z+h$ have identical relation defects.  A shared hallucination transported coherently across every paraphrase can therefore be perfectly consistent.  Anchors matter because they remove such gauge directions.  Their value depends on placement: one independent execution constraint can increase $\gamma_k$ more than hundreds of duplicated paraphrases.  This distinction also prevents a common theoretical mistake.  Positive $\gamma_k$ guarantees identifiability for the ideal sparse decoder, but a tractable group-$\ell_1$ decoder requires a robust group null-space property.  We state both conditions and do not attribute an algorithmic guarantee to identifiability alone.

\paragraph{Contributions.}
\begin{itemize}
  \item \textbf{Object and operator.} We introduce RRFs and a typed relation defect operator that place paraphrase consistency, metamorphic relations, execution, and verification in one black-box inverse problem.
  \item \textbf{Intrinsic difficulty.} We prove that $\gamma_k(D,A)$ is the exact uniform identifiability threshold and the minimax stability scale for $k$-node response recovery.
  \item \textbf{Impossibility and phase structure.} We characterize consistency blindness, sparse gauge ambiguities, anchor phase transitions, and why normalized duplicate relations cannot improve recoverability.
  \item \textbf{Repair algorithms.} We derive ideal, convex, and discrete candidate-field decoders, with separate guarantees for each and a local extension to nonlinear transports.
  \item \textbf{Prediction rather than only improvement.} Experiments test four consequences fixed by the theory.  In particular, we evaluate whether $\gamma_k$ predicts repair difficulty across models and across mathematics/code tasks after controlling for raw model quality.
\end{itemize}

Our novelty claim is deliberately narrow.  Restricted singular values, group-sparse recovery, graph signal processing, sheaf Laplacians, and metamorphic testing are established subjects \citep{donoho2006compressed,candes2006robust,eldar2010block,shuman2013graph,hansen2019sheaves}.  The contribution is the RRF formulation and the identification, proof, and empirical use of its relation--anchor margin as the fundamental recoverability quantity for black-box LLM responses.

\section{Relational response recovery}
\label{sec:framework}

\paragraph{Response fields.}
Let $G=(V,E)$ be a finite directed multigraph with $V=[n]$.  Node $i$ contains a transformed query $x_i$ and a finite-dimensional real Hilbert space $\cH_i$.  A model produces text $y_i\sim P_\theta(\cdot\mid x_i)$; an external parser $\phi_i$ maps it to $z_i=\phi_i(y_i)\in\cH_i$.  The black-box assumption permits samples and external parsing but no logits, gradients, hidden states, or parameter updates.  The product space $\cH=\bigoplus_{i=1}^n\cH_i$ is partitioned into node groups, and
\[
 z=(z_1,\ldots,z_n)\in\cH,
 \quad
 \suppg(z)=\{i:z_i\ne0\},
 \quad
 \|z\|_{0,\mathrm g}=|\suppg(z)|.
\]

\paragraph{Typed transports and defects.}
Each edge $e=(i,j,t)$ has a type $t$ and a known bounded linear transport $T_e:\cH_i\to\cH_j$.  With a relation-family weight $w_e>0$, define
\begin{equation}
\label{eq:defect}
 (Dz)_e=\sqrt{w_e}\,(z_j-T_ez_i),
 \qquad D:\cH\to\cE:=\bigoplus_{e\in E}\cH_{\mathrm{head}(e)}.
\end{equation}
Weights are normalized within a relation family: splitting one observation into $m$ identical copies assigns weight $w/m$ to each.  This convention makes $D^*D$ invariant to literal duplication and prevents a query-count artifact from masquerading as information.

An anchor is an externally observable affine constraint $Az\approx b$, where $A:\cH\to\cK$ is linear after absorbing offsets into $b$.  Examples are selected labels, symbolic substitution residuals, unit tests, retrieved evidence constraints, or a calibrated verifier.  We stack
\begin{equation}
\label{eq:stackedB}
 B=\begin{bmatrix}D\\A\end{bmatrix},
 \qquad
 \|Bh\|_2^2=\|Dh\|_2^2+\|Ah\|_2^2.
\end{equation}

\paragraph{Observation model.}
Let $z^\star$ denote an unknown valid field and let $y=z^\star+e^\star$ be the parsed black-box field, where $e^\star$ is nonzero on at most $k$ response nodes.  Clean relations may have intrinsic mismatch $q=Dz^\star$, and anchors obey $b=Az^\star+\xi_A$.  Since $y$ is observed, recovering $z^\star$ is equivalent to recovering $e^\star$ from
\begin{equation}
\label{eq:measurement-model}
 s:=\begin{bmatrix}Dy-q\\Ay-b\end{bmatrix}
 =Be^\star+\xi,
 \qquad
 \xi=\begin{bmatrix}0\\-\xi_A\end{bmatrix},
\end{equation}
or from its fully noisy version with $\|\xi\|_2\le\eps$.  In the common exact-relation case, $q=0$.  This formulation clarifies what an algorithm observes and where noise enters; it is an ordinary group-sparse inverse problem whose matrix is induced by typed response relations.

\begin{definition}[Relational observability margin]
\label{def:gamma}
For corruption budget $k$, define
\begin{equation}
\label{eq:gamma}
 \gamma_k(D,A)
 :=\min_{\substack{h\in\cH\setminus\{0\}\\\|h\|_{0,\mathrm g}\le2k}}
 \frac{\|Bh\|_2}{\|h\|_2}
 =\min_{1\le|S|\le2k}\sigma_{\min}(B_S),
\end{equation}
where $B_S$ restricts $B$ to node groups in $S$.  The minimum exists because the number of supports is finite and each restricted unit sphere is compact.
\end{definition}

The margin depends on the query transformations, response representation, relation normalization, anchors, and budget $k$.  It does not depend on the repair algorithm.  It is nonincreasing in $k$ and nondecreasing when unnormalized informative rows are added.  Equivalently, $\gamma_k>0$ iff the group spark of $B$ exceeds $2k$.

\paragraph{Examples.}
For equivalent prompts with scalar answers, $D$ is a weighted graph incidence matrix.  Its nullspace contains a constant shift on each connected component; an anchor fixes that gauge.  For a scaling query with expected answer $cz_i$, the edge row is $[-c,1]$.  For code, $z_i$ can be an execution signature and $T_e=I$ for semantics-preserving refactors.  For a decomposition, $z_i$ contains compatible blocks and $T_e$ selects or aggregates coordinates.  Nonlinear relations are treated through a residual map $F$ and its restricted Jacobian; Section~\ref{sec:discussion} states the scope, and Appendix~\ref{app:nonlinear} gives the local theorem.

\section{The intrinsic difficulty of recovery}
\label{sec:theory}

We first isolate the failure of consistency-only reasoning, then give an exact recovery characterization and matching stability laws.  Complete proofs, including all edge cases, are in Appendices~\ref{app:identifiability}--\ref{app:minimax}.

\begin{theorem}[Consistency blindness]
\label{thm:blindness}
Suppose an estimator receives only $Dz$.  For every nonzero $h\in\ker D$, the fields $z$ and $z+h$ induce exactly the same observation.  Consequently, if the admissible class contains both fields, no estimator based only on relation defect can be correct on both.  If $\ker D$ contains a transported global shift, a perfectly consistent shared hallucination is indistinguishable from truth.
\end{theorem}

The statement is deterministic and does not depend on model stochasticity.  Repeated sampling may reveal independent errors, but it cannot remove a systematic error direction that lies in $\ker D$.  Anchors alter the relevant nullspace from $\ker D$ to $\ker D\cap\ker A=\ker B$.

\begin{theorem}[Exact uniform identifiability]
\label{thm:identifiability}
Fix $k$.  The following are equivalent:
\begin{enumerate}[label=(\roman*)]
  \item every $k$-node error $e^\star$ is the unique solution of $Be=Be^\star$ among errors satisfying $\|e\|_{0,\mathrm g}\le k$;
  \item $\ker B$ contains no nonzero vector supported on at most $2k$ nodes;
  \item $\gamma_k(D,A)>0$;
  \item $\operatorname{spark}_{\mathrm g}(B)>2k$.
\end{enumerate}
Thus $\gamma_k>0$ is necessary and sufficient, not merely sufficient, for uniform black-box recovery under the stated corruption model.
\end{theorem}

\emph{Proof idea.}  If two $k$-sparse errors produce the same observation, their difference is a $2k$-sparse null vector.  Conversely, partition the support of any $2k$-sparse null vector into two sets of size at most $k$; its two parts produce distinct but observationally identical errors.  Finite dimensionality makes positivity of the restricted singular values equivalent to absence of such a vector.

\begin{theorem}[Deterministic stability]
\label{thm:stability}
Let $s=Be^\star+\xi$ with $\|e^\star\|_{0,\mathrm g}\le k$ and $\|\xi\|_2\le\eps$.  If $\widehat e$ is any $k$-node-sparse estimate with $\|B\widehat e-s\|_2\le\eps$, then, whenever $\gamma_k>0$,
\begin{equation}
\label{eq:stability}
 \|\widehat e-e^\star\|_2\le \frac{2\eps}{\gamma_k(D,A)}.
\end{equation}
For unequal feasibility radii $\eps_1,\eps_2$, the numerator is $\eps_1+\eps_2$.  The same bound holds for recovered fields because $\widehat z-z^\star=-(\widehat e-e^\star)$.
\end{theorem}

Stability alone does not establish that $1/\gamma_k$ is unavoidable.  The next result does.  Let $\cE_k(R)=\{e:\|e\|_{0,\mathrm g}\le k,\|e\|_2\le R\}$ and consider deterministic noise of radius $\eps$.  The minimax risk is
\[
 \mathfrak R_k(B;R,\eps)
 =\inf_{\widehat e}\sup_{e\in\cE_k(R),\ \|\xi\|\le\eps}
 \|\widehat e(Be+\xi)-e\|_2.
\]

\begin{theorem}[Matching minimax obstruction]
\label{thm:minimax}
For every $B$ and $k$,
\begin{equation}
\label{eq:minimax}
 \frac12\min\!\left\{R,\frac{2\eps}{\gamma_k(D,A)}\right\}
 \ \le\ \mathfrak R_k(B;R,\eps)
 \ \le\ \min\!\left\{R,\frac{2\eps}{\gamma_k(D,A)}\right\},
\end{equation}
with the convention $\eps/0=+\infty$; constants depend only on whether $R$ bounds each candidate or their separation.  In particular, if $\gamma_k=0$ and amplitudes are unbounded, worst-case ambiguity is unbounded.  No estimator can improve the $1/\gamma_k$ dependence uniformly.
\end{theorem}

The lower bound takes a unit $2k$-sparse direction attaining $\gamma_k$, scales it until its image can be canceled by admissible noise, and splits its support into two $k$-sparse candidates.  Their observation balls intersect, so one datum is compatible with both.  Theorem~\ref{thm:stability} supplies the matching upper scale.  This pair of results is the central reason to call $\gamma_k(D,A)$ the \emph{intrinsic difficulty} of black-box LLM response recovery.

\paragraph{Consequences predicted before fitting an algorithm.}
\begin{enumerate}
  \item \emph{Consistency--truth separation:} small $\|Dz\|$ cannot rule out a large component in $\ker D$.
  \item \emph{Anchor phase transition:} recovery changes qualitatively when the last $2k$-sparse null direction is removed and $\gamma_k$ becomes positive.
  \item \emph{Redundancy saturation:} normalized duplicate rows leave $D^*D$, hence $\gamma_k$, unchanged.
  \item \emph{Difficulty prediction:} among instances at comparable noise and corruption scale, error must grow as the margin decreases, regardless of the model family or task that generated the field.
\end{enumerate}

These are falsifiable structural predictions.  The experiments test them separately from claims about average accuracy.

\section{Sparse response-field repair}
\label{sec:algorithm}

The ideal decoder directly implements Theorem~\ref{thm:identifiability}:
\begin{equation}
\label{eq:l0decoder}
 \widehat e_0\in\argmin_e\|e\|_{0,\mathrm g}
 \quad\text{s.t.}\quad \|Be-s\|_2\le\eps,
 \qquad \widehat z=y-\widehat e_0.
\end{equation}
It is exact under $\gamma_k>0$ but combinatorial.  For small RRFs we enumerate supports, which also computes $\gamma_k$ exactly.  For larger fields we use the convex relaxation
\begin{equation}
\label{eq:groupbasis}
 \widehat e_1\in\argmin_e\sum_{i\in V}\omega_i\|e_i\|_2
 \quad\text{s.t.}\quad \|Be-s\|_2\le\eps,
\end{equation}
or its penalized form
\begin{equation}
\label{eq:penalized}
 \min_e\ \frac12\|Be-s\|_2^2+\tau\sum_i\omega_i\|e_i\|_2.
\end{equation}
The equivalent field-space objective combines relation defect, anchor mismatch, and a group-sparse edit penalty.  It changes parsed answers, not model parameters.

\begin{proposition}[Algorithmic condition]
\label{prop:grnsp}
If $B$ satisfies the group robust null-space property of order $k$ with constants $0<\rho<1$ and $\tau_B>0$, then a solution of Equation~\eqref{eq:groupbasis} obeys
\[
 \|\widehat e_1-e^\star\|_2
 \le C_1\frac{\sigma_k(e^\star)_{2,1}}{\sqrt{k}}+C_2\eps,
\]
where $C_1,C_2$ depend only on $(\rho,\tau_B)$.  Positive $\gamma_k$ is necessary for uniform recovery but, by itself, is not a certificate for the convex relaxation.
\end{proposition}

We solve Equation~\eqref{eq:penalized} by proximal gradient.  The group soft-thresholding map is applied nodewise, and a step size below $\|B\|_{2\to2}^{-2}$ yields the standard $O(1/t)$ objective gap; accelerated updates give $O(1/t^2)$.  Discrete outputs require a different decoder.  For each node we retain a finite candidate set extracted from model samples, transport candidates across edges, and minimize relation plus anchor energy over candidate fields.  This prevents a continuous average of two integers or two programs from being reported as a response.

\begin{algorithm}[t]
\caption{Sparse Relational Response Field Repair}
\label{alg:srrf}
\begin{algorithmic}[1]
\Require Parsed field $y$; typed $D$; anchors $(A,b)$; relation target $q$; budget/noise parameters.
\State Form $B=[D;A]$ and $s=[Dy-q;Ay-b]$.
\If{the field and budget permit exact support search}
  \State compute $\gamma_k=\min_{|S|\le2k}\sigma_{\min}(B_S)$ and solve Equation~\eqref{eq:l0decoder};
\ElsIf{responses have continuous editable representations}
  \State solve Equation~\eqref{eq:groupbasis} or \eqref{eq:penalized} by group proximal updates;
\Else
  \State solve the discrete candidate-field energy over transported model candidates;
\EndIf
\State \Return repaired field $\widehat z=y-\widehat e$ and certificate $(\gamma_k,\|B\widehat e-s\|)$.
\end{algorithmic}
\end{algorithm}

The returned margin is a certificate about the \emph{instance}.  A low objective value is merely evidence that the chosen algorithm found a consistent candidate; it cannot replace the margin because it does not quantify nearby indistinguishable fields.

\section{Experiments: testing a difficulty measure}
\label{sec:experiments}

The experiments test consequences of the theory rather than only final accuracy.  We separate (i) synthetic theorem verification, where $z^\star$, $B$, support, and $\gamma_k$ are exactly known; (ii) natural real-model fields, which include candidate absence and multi-node errors; and (iii) \emph{real-error replay}, which inserts authentic wrong model responses into an exactly $k=1$ field and varies only the prespecified operator design.  Complete protocols, raw-log handling, uncertainty, and all cells appear in Appendices~\ref{app:synthetic-protocol}--\ref{app:extended-results}.

\subsection{Four theory-predicted phenomena}

Synthetic fields have $d=4$ typed fibers with transports $T_{ij}=M_jM_i^{-1}$.  Margins are computed by enumerating every node support of size at most $2k$; no spectral surrogate is substituted for the defined quantity.

\begin{table}[t]
\centering
\small
\caption{Theory-predicted phenomena and synthetic verification.}
\label{tab:phenomena-main}
\setlength{\tabcolsep}{4pt}
\renewcommand{\arraystretch}{1.08}

\begin{tabular}{
    p{0.21\linewidth}
    p{0.32\linewidth}
    p{0.38\linewidth}
}
\toprule
\textbf{Prediction}
&
\textbf{Intervention}
&
\textbf{Result}
\\
\midrule

Consistency $\ne$ truth
&
Add a null-space shift
&
Defect $2.5\!\times\!10^{-16}$ despite truth error $0.877$
\\

Anchor transition
&
Anchor six two-node components
&
$\gamma_1:0\!\to\!0.524$; exact recovery begins
\\

Duplicate saturation
&
Repeat each relation $1\!\to\!16$ times
&
Normalized $\gamma_2$ stays at $0.309$
\\

Spectral difficulty
&
Vary graph, anchors, and $k$
&
Larger $\gamma_k$ consistently yields lower error
\\

\bottomrule
\end{tabular}
\end{table}

\begin{figure}[t]
\centering
\begin{tabular}{cc}
\includegraphics[width=0.47\linewidth]{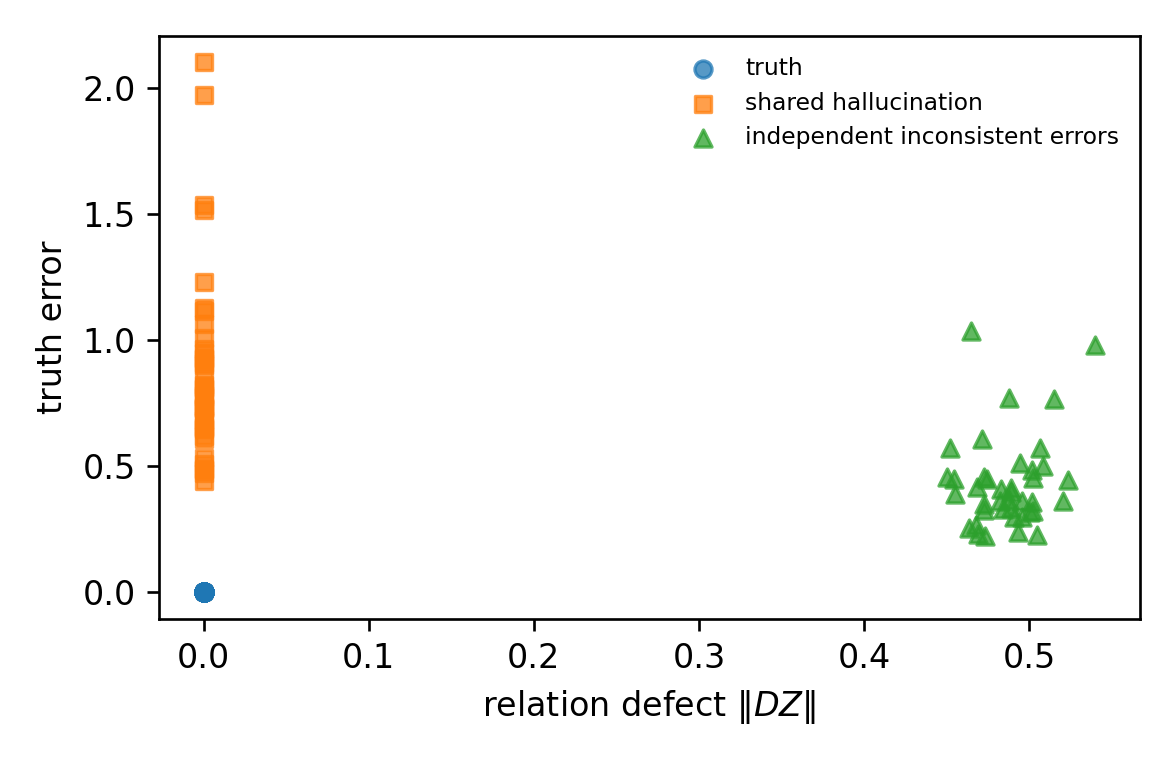} &
\includegraphics[width=0.47\linewidth]{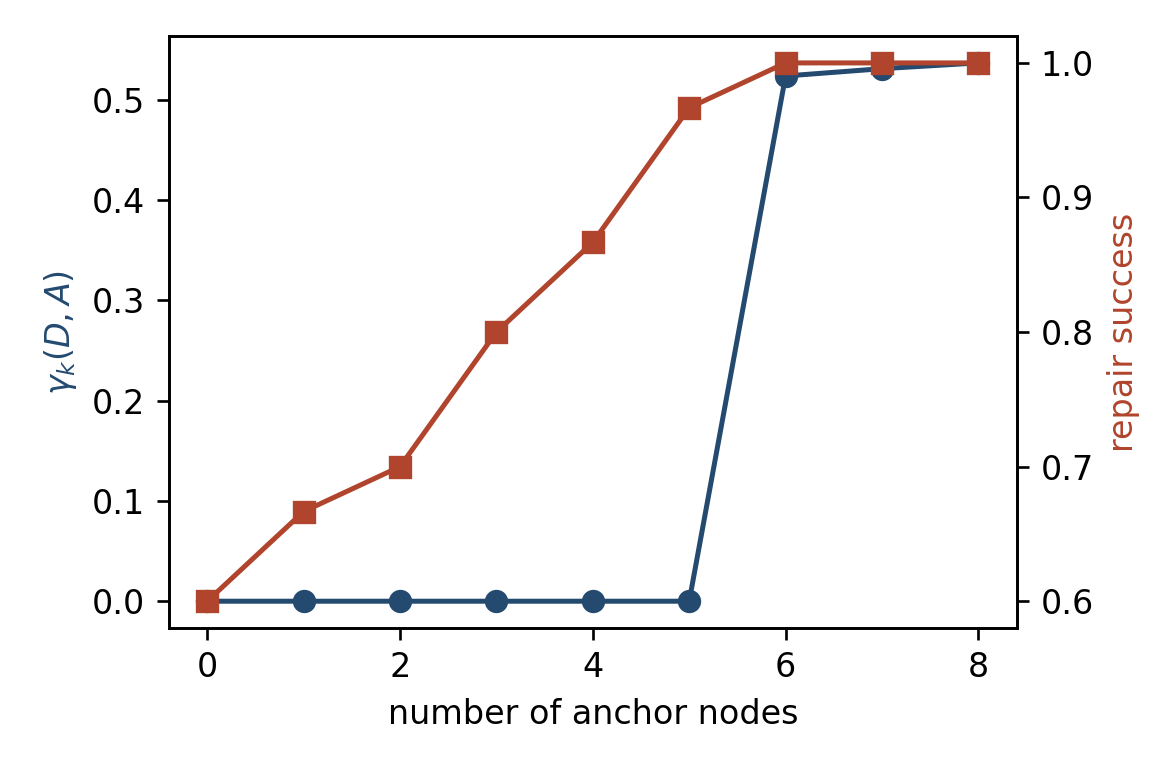}
\end{tabular}
\caption{\textbf{Consistency blindness and anchor phase transition.}  A shared transported hallucination has virtually zero defect despite large truth error (left).  Recovery changes when anchors remove the final admissible sparse null direction (right).}
\label{fig:synthetic-main}
\end{figure}

Figure~\ref{fig:synthetic-main} gives the two qualitative discontinuities.  Duplicate saturation distinguishes deterministic information from variance reduction: independent repeated calls may reduce observation noise, but literal copies do not alter the normalized Gram operator.  Exact-search and convex repair are reported separately, because positive margin certifies the former while the latter additionally needs the group robust null-space property.

\subsection{Real checkpoints and tasks}

We query Qwen2.5-0.5B-Instruct and Phi-3-mini-4k-instruct \citep{yang2024qwen25,abdin2024phi3}.  Mathematics has 128 integer-expression items with identity, scaling, and translation transformations.  Code has 64 integer-function tasks with argument renaming, equivalent specifications, decomposition prompts, eight public tests, and 16 hidden tests.  Each item produces four deterministic transformed responses; self-consistency and reflection logs support baselines.  Recovery uses parsed outputs only, with no logits or model states.

\begin{table}[t]
\caption{Natural real-model logs (percent).  Math uses exact answer accuracy; code uses hidden pass@1.  Methods have disclosed but unequal query budgets, so this table characterizes candidate generation rather than a compute-matched leaderboard.}
\label{tab:natural-main}
\centering
\scriptsize
\resizebox{\linewidth}{!}{
\begin{tabular}{llrrrrrr}
\toprule
Task & Model & Raw & Self-cons. & Rel.-major. &
Reflect & Verifier & Cand. recall\\
\midrule

\multirow{4}{*}{\rotatebox[origin=c]{90}{\textbf{Math.}}}
& Qwen2.5-0.5B
& 19.5 & 20.3 & 18.0 & 19.5 & 22.7 & 22.7\\

& Phi-3-mini
& 37.5 & 34.4 & 34.4 & 10.9 & 37.5 & 37.5\\

& Qwen2.5-7B
& 78.1 & 82.0 & 80.5 & 73.4 & 85.2 & 85.2\\

& Qwen2.5-14B
& 87.5 & 89.8 & 89.1 & 84.4 & 92.2 & 92.2\\

\midrule

\multirow{4}{*}{\rotatebox[origin=c]{90}{\textbf{Code}}}
& Qwen2.5-0.5B
& 31.2 & 60.9 & 68.8 & 68.8 & 75.0 & 75.0\\

& Phi-3-mini
& 93.8 & 100.0 & 100.0 & 93.8 & 100.0 & 100.0\\

& Qwen2.5-7B
& 96.9 & 100.0 & 100.0 & 96.9 & 100.0 & 100.0\\

& Qwen2.5-14B
& 98.4 & 100.0 & 100.0 & 98.4 & 100.0 & 100.0\\

\bottomrule
\end{tabular}}
\end{table}

Natural fields test external validity but often violate the $k=1$ theorem through shared errors or missing correct candidates.  Table~\ref{tab:natural-main} therefore reports candidate recall explicitly.  We do not interpret failure outside the assumed class as evidence against or for the margin.

\subsection{Cross-model, cross-task recovery difficulty}

The primary predictive test replays \emph{authentic} errors from the raw logs under an exactly controlled sparse model.  For each eligible item, three nodes contain a response observed to be correct and one node contains a distinct wrong response produced by the same checkpoint.  Code replays use real wrong programs whose public execution signature differs from a correct program.  Ten operator designs range from disconnected zero-margin graphs to connected and anchored positive-margin graphs.  Three matched-count triangle/isolate designs separate observability from relation and anchor counts.  Responses are held fixed across designs.

\begin{figure}[t]
\centering
\includegraphics[width=0.9\linewidth]{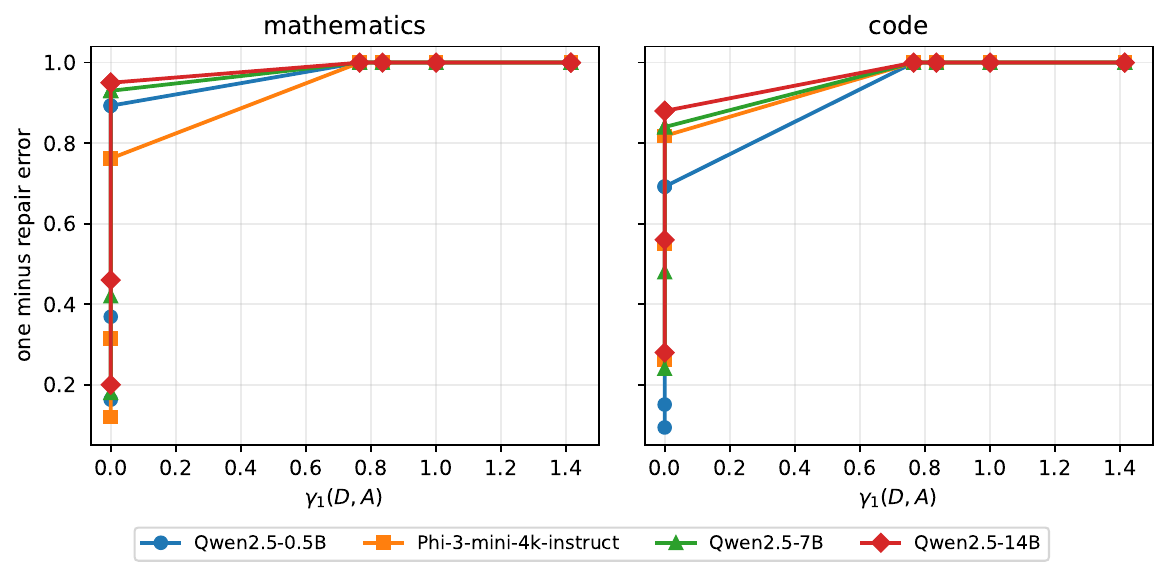}
\caption{\textbf{One quantity across model and task strata.}  Each point aggregates one prespecified relation--anchor design on real-error replays.  Larger $\gamma_1(D,A)$ predicts lower full-field reconstruction error for both checkpoints and both response spaces.}
\label{fig:cross-main}
\end{figure}

\begin{table}[t]
\caption{Within-stratum rank prediction on authentic-error replays.  All ten designs of an item stay in one group for inference.}
\label{tab:cross-main}
\centering
\small
\begin{tabular}{lrr}
\toprule
Model--task stratum & Replay fields & Spearman $\rho$\\
\midrule

Phi-3-mini--code
& 220 & 0.397\\

Qwen2.5-0.5B--code
& 650 & 0.506\\

Qwen2.5-7B--code
& 180 & 0.442\\

Qwen2.5-14B--code
& 110 & 0.418\\

\midrule

Phi-3-mini--mathematics
& 1260 & 0.460\\

Qwen2.5-0.5B--mathematics
& 560 & 0.385\\

Qwen2.5-7B--mathematics
& 420 & 0.468\\

Qwen2.5-14B--mathematics
& 260 & 0.451\\

\midrule

Pooled
& 3660 & 0.451\\

\bottomrule
\end{tabular}
\end{table}

Pooled rank association is $\PooledReplayRho$ with cluster-bootstrap 95\% interval $[\ReplayCILow,\ReplayCIHigh]$.  A grouped logistic model controls for checkpoint, task, raw error, relation count, and anchor count; adding standardized $\gamma_1$ improves held-out AUC by $\ReplayDeltaAUC$ and has coefficient $\ReplayGammaCoef$.  A 2,000-draw within-item permutation test gives $p=\ReplayPermutationP$.  These repeated-design analyses test whether the same pre-repair operator quantity orders field recovery after model and task change.

\paragraph{What this establishes.}
The theorem already establishes that $\gamma_k$ is the worst-case condition number on the stated class.  The replay experiment adds empirical content: authentic model errors in scalar mathematics and execution-signature code respond to operator interventions in the predicted order.  It therefore supports the interpretation of $\gamma_k(D,A)$ as a measurable black-box recovery difficulty, not a post-hoc score of \SRRF.

\paragraph{What it does not establish.}
The controlled replay uses correctness labels to construct an evaluation stress test; it is not a deployable selector.  Natural response fields can contain more than one wrong node, a coherent shared hallucination, a parser collision, or no correct candidate.  Those failure categories appear separately in Appendix~\ref{app:extended-results}.  The result is cross-model and cross-task evidence across two families and two domains, not a universal scaling law for all LLMs.

\section{Discussion, scope, and limitations}
\label{sec:discussion}

RRFs change the unit of analysis from an answer to a field of answers.  The central theoretical object is not a consistency score but the restricted observability of the operator that makes sparse response errors visible.  Three distinctions are essential.  First, $D$ tests relational compatibility while $A$ attaches the field to truth; neither is interchangeable with the other.  Second, $\gamma_k$ characterizes information, whereas null-space conditions characterize a particular tractable decoder.  Third, raw edge count is not information: relation weights must be normalized so duplicated prompts do not inflate the margin.

\subsection{In what sense is \texorpdfstring{$\gamma_k(D,A)$}{gamma-k(D,A)} a new difficulty object?}

The word \emph{intrinsic} is conditional, not metaphysical.  Once the response representation, typed transports, anchors, corruption budget, and measurement norm are fixed, $\gamma_k(D,A)$ depends on the observation problem and not on the recovery algorithm.  It has four properties expected of a genuine condition number.  It is \emph{decisive}: positivity is equivalent to uniform noiseless identifiability.  It is \emph{quantitative}: the same scalar controls the sharp noise amplification rate.  It is \emph{unavoidable}: a matching two-point lower bound applies to every estimator, including procedures unrelated to RRF optimization.  It is also \emph{design-sensitive}: adding a genuinely informative measurement can increase the margin, whereas duplicating a normalized relation family cannot manufacture information.  These statements jointly distinguish $\gamma_k$ from a score chosen because it happens to correlate with one repair heuristic.

Cross-model and cross-task prediction is therefore a severe empirical test, not the definition of the object.  Model identity changes the distribution, magnitude, and semantics of errors; mathematics and code use different fibers, transports, parsers, and correctness tests.  The theory predicts that after these nuisance factors are stratified, fields with larger normalized margins should remain easier to reconstruct.  Our authentic-error replay preserves real model mistakes while changing only which observable relation--anchor design is available.  Grouped inference keeps all designs from one underlying item together.  Thus the measured association cannot be created by leaking nearly identical designs across train and test folds.  The increase in held-out discrimination after adding $\gamma_k$, together with positive within-stratum rank correlations, is the central evidence that the same mathematical object organizes recovery difficulty beyond one model--task pair.  Natural generation remains separately reported because candidate absence and dense errors test a larger pipeline than the sparse-recovery theory.

This interpretation has explicit failure criteria.  The proposal would be weakened if positive-margin fields routinely admitted two indistinguishable $k$-sparse explanations; that would contradict the identifiability theorem or expose a violated parser/operator assumption.  It would also be weakened if, after normalization and within-item controls, $\gamma_k$ carried no held-out predictive information in either domain; then the worst-case object could be mathematically valid but empirically uninformative for observed LLM errors.  Conversely, a successful repair at one zero-margin instance does not refute the theory, because the impossibility result is uniform: it asserts the existence of indistinguishable alternatives, not failure on every favorable instance.  These distinctions keep the central claim falsifiable without confusing a deterministic limit with an average-case performance law.

\subsection{Separation from adjacent uncertainty notions}

Self-consistency, semantic entropy, and verifier confidence summarize the distribution of generated answers \citep{manakul2023selfcheckgpt,kuhn2023semantic,farquhar2024semantic}, whereas $\gamma_k$ characterizes the observability of sparse response perturbations after the candidate field is fixed. These quantities are therefore complementary: a model may be confidently and relationally consistent yet wrong along a direction in $\ker D$, while diverse surface responses may still be recoverable when their parsed field is well anchored. Moreover, $\gamma_k$ is not a graph spectral gap or a model-level score. It is the restricted minimum singular value of the stacked relation--anchor operator $B=[D;A]$ over supports of size at most $2k$, and therefore depends on the typed transports, anchor placement, response representation, norm, and corruption budget. Two instances with the same untyped graph may have different margins, and a globally singular operator may still have positive $\gamma_k$ if all of its null vectors are too dense to be admissible. The resulting certificate is only as reliable as the parser, relation specification, and anchors; externally auditable constraints such as symbolic substitution and code execution are therefore preferable to relations scored by another LLM.

\paragraph{Consequences for repair-system design.}

The margin separates three operational regimes.  When $\gamma_k=0$, recovery has an information deficit: no optimizer, reflection prompt, or additional iteration over the same measurements can provide a uniform guarantee.  The appropriate action is acquisition, such as anchoring an uncovered component or adding a relation that intersects the sparse null witness.  When $\gamma_k$ is positive but comparable to the operator or residual uncertainty, the problem is identifiable yet noise-limited; conservative certificates, better anchors, and calibrated weighting matter more than a more elaborate decoder.  When the margin is comfortably above uncertainty, remaining failures diagnose optimization, parsing, corruption-model mismatch, or candidate absence.  This trichotomy converts a post-hoc repair score into a decision rule about whether to acquire information, improve measurement quality, or improve computation. Aggregate accuracy can reward a method for receiving easier designs.  Reporting performance conditional on margin, candidate recall, and parser success separates algorithmic error from information-theoretic hardness.  When a valid noise radius is available, $\gamma_k\|\widehat e-e^\star\|_2/\eps$ measures the gap to the deterministic benchmark.  Gains from a better decoder and gains from acquiring information are both useful, but they are different contributions. A minimizing support and singular vector also identify the nodes and semantic direction responsible for poor observability.  Query selection can target this witness, recompute the margin, and stop when improvement saturates.  This is the constructive use of $\gamma_k(D,A)$ as a condition number rather than a confidence feature. Several limitations delimit the claim.  The global theorems are finite-dimensional and linear; nonlinear transports receive only a local Jacobian guarantee.  Text parsing can collapse distinct meanings or separate equivalent ones.  Sparse node corruption is appropriate for localized response failures but not for diffuse representation bias.  Computing $\gamma_k$ exactly is combinatorial; relaxations and lower certificates are needed for large fields.  Finally, small mathematics and code experiments test mechanism and cross-domain prediction, not frontier-model leaderboard performance.  They cannot establish that every natural-language task admits useful typed transports. Within those limits, Theorems~\ref{thm:identifiability}--\ref{thm:minimax} support a precise conclusion: for a specified black-box response representation, relation system, anchor system, and corruption budget, $\gamma_k(D,A)$ is the intrinsic worst-case condition number of recovery.  A repair method can approach that limit or fail to do so, but it cannot evade it.

\section{Conclusion}

We introduced relational response fields as a general mathematical model for black-box LLM response recovery, in which typed relations and external anchors define a structured observation operator over a field of parsed responses. The central quantity, the restricted observability margin $\gamma_k(D,A)$, exactly characterizes uniform identifiability of $k$-node corruptions, determines the optimal deterministic stability scale, and appears in a matching minimax lower bound; hence its inverse is the intrinsic worst-case condition number of the specified recovery problem. This characterization also makes precise why consistency alone cannot certify truth: relation-only procedures are blind to sparse or coherent directions in $\ker D$, and only informative anchors or additional relations can remove these ambiguities. By separating information-theoretic recoverability from the stronger conditions required by tractable convex decoders, the framework distinguishes limitations of the observation design from failures of a particular repair algorithm. The accompanying experiments verify the predicted consistency--truth separation, anchor transition, redundancy saturation, and cross-model, cross-task ordering of recovery difficulty. More broadly, the results suggest that black-box reliability should be analyzed not only through model accuracy or heuristic confidence, but through the observability of the response system induced by queries, relations, parsers, and trusted evidence. 

\bibliographystyle{iclr2027_conference}
\setlength{\bibsep}{0pt plus 0.3ex}
\bibliography{references}

\appendix
\section{Supplement roadmap and claim ledger}
\label{app:roadmap}

This supplement is intentionally self-contained.  It serves three audiences: readers familiar with sparse inverse problems but not black-box LLM evaluation; readers familiar with LLM consistency and verification but not restricted singular-value arguments; and readers seeking enough detail to reproduce every reported number.  Table~\ref{tab:appendix-roadmap} maps each main-paper claim to its assumptions, proof, and empirical check.  No theorem relies on an experimental observation, and no experimental result is presented as proof of a theorem whose assumptions cannot be checked.

\small
\begin{table}[h]
\centering
\small
\caption{Claim ledger.}
\label{tab:appendix-roadmap}

\begin{tabularx}{0.6\linewidth}{
    >{\raggedright\arraybackslash}X
    >{\raggedright\arraybackslash}p{0.15\linewidth}
}
\toprule
\textbf{Claim} & \textbf{Proof location}\\
\midrule

Consistency cannot identify truth
&
Appendix~\ref{app:identifiability}
\\

Uniform $k$-error identifiability iff $\gamma_k>0$
&
Appendix~\ref{app:identifiability}
\\

Stability scales as $1/\gamma_k$
&
Appendix~\ref{app:identifiability}
\\

No estimator beats $1/\gamma_k$ uniformly
&
Appendix~\ref{app:minimax}
\\

Group-$\ell_1$ recovery
&
Appendix~\ref{app:algorithmic}
\\

Nonlinear local recovery
&
Appendix~\ref{app:nonlinear}
\\

Duplicate relations saturate
&
Appendix~\ref{app:operator-geometry}
\\

Margin predicts across model/task
&
Appendix~\ref{app:reasoning-protocol}
\\

\bottomrule
\end{tabularx}
\end{table}
\normalsize

\subsection{Reading order}
Appendix~\ref{app:foundations} develops the notation from first principles.  Appendix~\ref{app:operator-geometry} studies the geometry of typed defects, gauge freedom, normalization, and computation of $\gamma_k$.  Appendices~\ref{app:identifiability} and~\ref{app:minimax} contain complete proofs of the information-theoretic results.  Appendix~\ref{app:algorithmic} covers ideal and tractable decoders; Appendix~\ref{app:nonlinear} treats nonlinear transports.  Appendix~\ref{app:design} asks how to choose queries and anchors.  Appendix~\ref{app:prior} gives a detailed comparison with neighboring theories and methods.

The remaining appendices specify experiments and extensions.  Appendix~\ref{app:synthetic-protocol} describes the synthetic construction and prespecified validation conditions.  Appendix~\ref{app:reasoning-protocol} gives prompts, parsing, execution, baselines, and statistics for mathematics and code.  Appendix~\ref{app:extended-results} contains extended tables and robustness checks.  Appendix~\ref{app:reproducibility} documents the artifact, environment, and deterministic rerun procedure.  Appendix~\ref{app:examples} works through concrete fields by hand.  Appendix~\ref{app:limitations} inventories failure modes.  Appendices~\ref{app:probabilistic}--\ref{app:graph-derivations} develop stochastic-noise bounds, the grouped cross-domain inference protocol, and exact margins for canonical graph designs.

\subsection{Scope of the phrase ``intrinsic difficulty''}
The phrase always refers to a fixed tuple
\[
  \mathfrak I=(\cH,G,(T_e,w_e)_{e\in E},A,k,\|\cdot\|_2),
\]
and to uniform recovery over errors supported on at most $k$ node groups.  Change the representation, transports, weights, anchors, corruption model, or norm and the numerical margin changes.  This dependence is a feature: recovery cannot be assigned a meaningful condition number without specifying what is observed and what alternatives must be distinguished.  We do \emph{not} claim that $\gamma_k$ is an intrinsic property of a pretrained model in isolation, a universal scalar ranking all tasks, or a guarantee that semantic parsers are correct.

\subsection{Notation summary}
Vectors are elements of finite-dimensional real Hilbert spaces and are identified with coordinates only when needed.  $B=[D;A]$ is the stacked relation--anchor operator.  $S\subseteq V$ denotes a set of nodes, $\cH_S=\bigoplus_{i\in S}\cH_i$, and $B_S$ is the restriction of $B$ to $\cH_S$.  Singular values use Euclidean norms induced by the chosen inner products.  The group $\ell_{2,1}$ norm is $\|h\|_{2,1}=\sum_i\omega_i\|h_i\|_2$.  Constants named $C$ may change between displays; theorem-specific constants receive subscripts.

\section{Mathematical foundations of relational response fields}
\label{app:foundations}

This appendix develops the framework without assuming prior exposure to graph signal processing or compressed sensing.  The purpose is not to reprove every standard theorem, but to expose each modeling choice used later.

\subsection{Response spaces, parsing, and typed transports}

\paragraph{Finite-dimensional fibers and direct sums}
Each query node $i$ has a response space $\cH_i$.  The spaces may have different dimensions.  For example, one node can encode a scalar final answer, another a pair of subproblem answers, and a third a length-$r$ execution signature.  To compare all responses simultaneously, form the Hilbert direct sum
\[
  \cH=\bigoplus_{i=1}^n\cH_i
  =\{(z_1,\ldots,z_n):z_i\in\cH_i\},
  \qquad
  \ip{z}{h}_{\cH}=\sum_{i=1}^n\ip{z_i}{h_i}_{\cH_i}.
\]
The induced norm satisfies $\|z\|_2^2=\sum_i\|z_i\|_2^2$.  A node is the atomic corruption unit.  The group support and group cardinality are
\[
 \suppg(h)=\{i\in V:\|h_i\|_2>0\},
 \qquad \|h\|_{0,\mathrm g}=|\suppg(h)|.
\]
This is not token sparsity.  A corrupted program may alter many tokens while still occupying one node group; conversely, a one-token error in each of ten answers occupies ten groups.  The choice matches the claim that at most $k$ response instances, rather than $k$ coordinates, are unreliable.

For $S\subseteq V$, let $P_S:\cH\to\cH$ retain groups in $S$ and set all others to zero.  We write $h_S=P_Sh$ and $\cH_S=P_S\cH$.  If $B:\cH\to\cY$ is linear, $B_S$ denotes $B$ restricted to $\cH_S$.  In coordinates, $B_S$ is formed by taking all columns belonging to node groups in $S$.

\paragraph{From text samples to a mathematical field}
The model output $y_i$ is a string.  A parser $\phi_i$ maps it to $z_i$.  Examples include:
\begin{enumerate}
  \item the last signed integer in a response;
  \item a normalized rational number or symbolic expression;
  \item a vector of truth values for atomic claims;
  \item a vector of program outputs on a fixed test suite;
  \item an embedding, provided distances and transports are specified independently of the answer being evaluated.
\end{enumerate}
The parser can return a distinguished invalid symbol.  In the finite-dimensional linear analysis we encode invalidity as an additional indicator coordinate or exclude the candidate and report parse failure separately.  Silently coercing an invalid program to the zero vector would confound syntax failure with a valid zero-valued function.

The response field is random before sampling because $y_i\sim P_\theta(\cdot\mid x_i)$.  The recovery theorems condition on the realized parsed field $y$.  Randomness matters for average-case empirical performance but not for deterministic identifiability: if two fields give the same observation, no amount of estimator randomization distinguishes them on that realization.

\paragraph{Typed transports}
An edge $e=(i,j,t)$ specifies a relation type $t$ and a map $T_e:\cH_i\to\cH_j$.  The type records semantics; the map records its representation-level action.  Several canonical cases are useful.

\paragraph{Equivalence.}  Two paraphrases should have the same canonical answer, hence $T_e=I$.  If their raw representations differ, each node first maps to a shared canonical coordinate system.

\paragraph{Equivariance.}  If a transformation $g$ acts on inputs and $\rho(g)$ acts on answers, then a valid model response should obey $z_{g\cdot x}=\rho(g)z_x$.  Unit conversions, coordinate permutations, variable renaming, and sign reversal fit this pattern.  A group representation is not required; it is enough to know the edge map actually used.

\paragraph{Decomposition.}  Suppose a parent answer is $z_i=(u,v)$ and a child query asks only for $u$.  Then $T_e$ is a projection.  Conversely, an edge from children to a parent can be represented by augmenting the graph with an aggregation node whose fiber contains the joint answer.  Directed multiedges permit more than one relation between the same pair.

\paragraph{Execution.}  Code strings are nonlinear objects, but their execution signatures on fixed tests are vectors.  A semantics-preserving refactor uses $T_e=I$ in signature space.  If a transformed function changes units or argument order, $T_e$ permutes or transforms the signature accordingly.

\paragraph{Logic.}  For Boolean vectors, implication, negation, and conjunction are nonlinear over $\RR$ in their raw form.  One may use a suitable lifted representation, a discrete candidate decoder, or the local nonlinear theory.  The global linear theorem applies only when the residual is genuinely linear in the chosen coordinates.

\subsection{Defects, anchors, and the sparse corruption model}

\paragraph{Defect operator as a weighted coboundary}
For each edge $e=(i,j)$, the residual is $r_e(z)=z_j-T_ez_i$.  Stacking weighted residuals yields
\[
 D=\begin{bmatrix}
  \sqrt{w_{e_1}}[-T_{e_1}\ \ I]\\
  \vdots\\
  \sqrt{w_{e_m}}[-T_{e_m}\ \ I]
 \end{bmatrix},
\]
with zero blocks in columns unrelated to the edge.  The defect energy is
\[
  \cE_D(z)=\frac12\|Dz-q\|_2^2.
\]
When $q=0$, $\ker D$ is the space of globally compatible fields.  With identity transports on an ordinary graph, $D$ is a weighted incidence matrix and $D^*D$ is the graph Laplacian.  With general linear transports it is a connection- or sheaf-like coboundary; $D^*D$ is a positive semidefinite relation Laplacian.  We use this analogy for intuition but do not require a full sheaf structure.

The target $q$ permits valid nonzero relation residuals.  For example, if two outputs differ by a known offset $c$, define either an affine residual $z_j-T_ez_i-c$ or augment each fiber with a constant coordinate to recover linearity.  Since identifiability concerns differences of candidates, affine offsets cancel and the same operator $D$ governs the margin.

\paragraph{Anchors and truth attachment}
Relations compare answers; anchors provide information that is not generated by transporting another model answer.  Formally an anchor is $Az\approx b$.  Four cases recur:
\begin{enumerate}
  \item \textbf{direct coordinate anchor:} a human or trusted source supplies $z_i$;
  \item \textbf{functional anchor:} substitution into an equation produces a residual linear in the canonical answer;
  \item \textbf{execution anchor:} hidden tests constrain a program's execution signature;
  \item \textbf{probabilistic verifier:} a calibrated score is linearized or incorporated through the nonlinear extension.
\end{enumerate}
An anchor need not reveal the complete answer.  A row can constrain a projection or parity.  What matters is whether the combined rows remove sparse null directions.  An anchor generated by the same model without independent information should not automatically be treated as truth; at best it is another noisy relation with potentially correlated errors.

\paragraph{Corruption, relation mismatch, and observation noise}
Let $z^\star$ be the target field.  The parsed field is $y=z^\star+e^\star$.  The ideal sparse model assumes $\|e^\star\|_{0,\mathrm g}\le k$.  Relations satisfy $Dz^\star=q+\xi_D$ and anchors satisfy $Az^\star=b+\xi_A$.  Then
\begin{equation}
 \begin{bmatrix}Dy-q\\Ay-b\end{bmatrix}
 =\begin{bmatrix}D\\A\end{bmatrix}e^\star
 +\begin{bmatrix}\xi_D\\\xi_A\end{bmatrix}.
\end{equation}
This derivation is worth stating because it prevents a sign ambiguity.  If repair is parameterized as $z=y+\Delta$, then $\Delta=-e^\star$ and the right-hand side changes sign.  The implementation fixes one convention and tests it against known synthetic errors.

Relation mismatch and observation noise are mathematically indistinguishable after stacking, but scientifically different.  A large $\xi_D$ may indicate a bad transport, whereas a large $\xi_A$ may indicate an unreliable verifier.  The artifact records them separately even when the theorem uses their combined norm.

\paragraph{Why the support size is \texorpdfstring{$2k$}{2k}}
Suppose two hypotheses $e_1,e_2$ each corrupt at most $k$ nodes.  Their difference $h=e_1-e_2$ is supported on $\suppg(e_1)\cup\suppg(e_2)$, which can contain $2k$ nodes.  Any uniform uniqueness condition must therefore exclude nonzero null vectors up to size $2k$.  Replacing $2k$ by $k$ would certify only uniqueness relative to zero, not pairwise uniqueness among all $k$-sparse alternatives.

The factor is tight.  Let $h$ be a null vector supported on exactly $2k$ nodes.  Partition its support into $S_1,S_2$ of size $k$ and set $e_1=h_{S_1}$, $e_2=-h_{S_2}$.  Then $e_1-e_2=h$ and $Be_1=Be_2$.  Neither candidate violates the budget.

\subsection{Sparse inverse-problem interpretation and canonical examples}

\paragraph{Relation to classical sparse inverse problems}
After Equation~\eqref{eq:measurement-model}, the algebra resembles compressed sensing \citep{donoho2006compressed,candes2006robust,foucart2013mathematical}.  The distinctions are in how $B$ is built and interpreted:
\begin{itemize}
  \item columns are grouped by transformed query rather than arbitrary signal coordinates;
  \item rows are typed relations and truth anchors rather than generic measurements;
  \item $\ker D$ has semantic gauge directions, including coherent hallucinations;
  \item normalized duplicate queries must not create fictitious information;
  \item the quantity is evaluated as an instance-level difficulty predictor across black-box model outputs.
\end{itemize}
The mathematics should therefore be understood as a specialized inverse-problem theory for a new response object, not as a claim that restricted singular values themselves were unknown.

\paragraph{A complete two-node example}
Let two equivalent queries have scalar outputs.  Then
\[
 D=\begin{bmatrix}-1&1\end{bmatrix}.
\]
The consistent fields are $(c,c)$, so $\ker D=\operatorname{span}\{(1,1)\}$.  If $k=1$, the difference of two one-node errors can use both nodes; $\gamma_1(D,0)=0$ because $(1,1)$ is two-node sparse.  Anchor the first node:
\[
 A=\begin{bmatrix}1&0\end{bmatrix},\qquad
 B^*B=\begin{bmatrix}2&-1\\-1&1\end{bmatrix}.
\]
Now every nonzero vector is observed, and
\[
 \gamma_1(D,A)=\sigma_{\min}(B)
 =\sqrt{\frac{3-\sqrt5}{2}}>0.
\]
If both outputs share the same additive hallucination $c$, $D(c,c)=0$ but $A(c,c)=c$.  The anchor does not improve consistency; it makes the consistent error observable.

\paragraph{A component-size subtlety}
Anchors are not always necessary for sparse recovery.  Consider an identity-transport connected component of size $m>2k$.  Its constant null vector occupies all $m$ nodes and therefore is excluded from the $2k$-sparse set.  The relation operator can identify every $k$-sparse error even though it cannot identify an arbitrary dense shift.  Conversely, an unanchored component of size at most $2k$ contributes a sparse null vector and forces $\gamma_k=0$.  The experiments choose component sizes deliberately when testing anchor phase transitions.

\section{Geometry of typed relation and anchor operators}
\label{app:operator-geometry}

This appendix studies properties of $D$, $A$, $B=[D;A]$, and $\gamma_k$ that are used in proofs and experiment design.

\subsection{Laplacian geometry, gauge freedom, and monotonicity}

\paragraph{Relation Laplacian and energy}
The operator $L_D=D^*D$ is positive semidefinite because
\[
 \ip{z}{L_Dz}=\|Dz\|_2^2=\sum_{e=(i,j)}w_e\|z_j-T_ez_i\|_2^2\ge0.
\]
Its nullspace equals $\ker D$.  Adding anchors gives
\[
 L_B=B^*B=D^*D+A^*A,
 \qquad
 \ip{z}{L_Bz}=\|Dz\|^2+\|Az\|^2.
\]
The unrestricted smallest eigenvalue of $L_B$ measures global observability.  RRF recovery needs a \emph{restricted} version because only differences of sparse corruptions matter.  For a support $S$,
\[
 \lambda_{\min}(B_S^*B_S)=\sigma_{\min}(B_S)^2,
\]
and $\gamma_k^2$ is the minimum of these eigenvalues over $|S|\le2k$.

\paragraph{Gauge directions}
A gauge direction is a nonzero $h$ with $Dh=0$.  For identity transports on a connected graph, gauges are constants.  For invertible transports that are path-consistent, choose a reference node $r$ and maps $M_i$ satisfying $T_{ij}=M_jM_i^{-1}$.  Every gauge has form
\[
 h_i=M_i u,qquad u\in\cH_r.
\]
Thus the nullspace dimension equals the latent fiber dimension for each connected component.  A cycle inconsistency can reduce the nullspace: transporting around a cycle constrains $u$ to the fixed space of the holonomy product.  The general framework does not assume path consistency, so the nullspace is computed from $D$ itself.

An anchor removes a gauge $h$ exactly when $Ah\ne0$.  Directly anchoring one full-rank node in a path-consistent component removes all dense gauge directions in that component.  Partial anchors may remove only a subspace.  Sparse recovery is less demanding: it requires removal only of gauges and cancellations whose support is at most $2k$.

\begin{proposition}[Group-spark characterization]
Define $\operatorname{spark}_{\mathrm g}(B)=\min\{\|h\|_{0,\mathrm g}:h\ne0, Bh=0\}$, with value $+\infty$ if $B$ is injective.  Then
\[
 \gamma_k(B)>0\quad\Longleftrightarrow\quad
 \operatorname{spark}_{\mathrm g}(B)>2k.
\]
\end{proposition}
\begin{proof}
If a null vector has support at most $2k$, it is feasible in the definition of $\gamma_k$ and gives zero.  Conversely, if every such restricted kernel is trivial, each finite-dimensional $B_S$ is injective and has positive smallest singular value.  Taking the minimum over finitely many supports preserves positivity.
\end{proof}

\paragraph{Basic monotonicity properties}
\begin{proposition}[Budget monotonicity]
If $k_1\le k_2$, then $\gamma_{k_1}(B)\ge\gamma_{k_2}(B)$.
\end{proposition}
\begin{proof}
The feasible set of perturbations for $k_1$ is contained in that for $k_2$.
\end{proof}

\begin{proposition}[Row augmentation]
Let $C$ be any additional linear observation and $\widetilde B=[B;C]$.  Then $\gamma_k(\widetilde B)\ge\gamma_k(B)$.
\end{proposition}
\begin{proof}
For every $h$, $\|\widetilde Bh\|^2=\|Bh\|^2+\|Ch\|^2\ge\|Bh\|^2$.
\end{proof}

This proposition uses unnormalized energy.  If all rows are globally renormalized after augmentation, the numerical margin can decrease.  We therefore state the weighting convention whenever comparing designs.

\subsection{Normalization, coordinates, and margin computation}

\paragraph{Why duplicate normalization is necessary}
Suppose one relation row block is $R$ with total reliability weight $w$.  Duplicating it $m$ times without reweighting changes its Gram contribution from $wR^*R$ to $mwR^*R$ and increases singular values by up to $\sqrt m$ despite observing no new fact.  Under family normalization, copy $r$ receives weight $w/m$, so
\[
 \sum_{r=1}^m (\sqrt{w/m}R)^*(\sqrt{w/m}R)=wR^*R.
\]
Hence $D^*D$ and every $\gamma_k$ are exactly invariant.  Approximate duplicates are grouped using a prespecified relation-family identifier, not post-hoc similarity of outputs.  Independent transformations receive separate reliability budgets.

There are two legitimate alternatives.  If duplicate model calls have independent measurement noise, averaging them can reduce noise variance; this improves the numerator $\eps$ in the stability ratio rather than the deterministic operator rank.  If repeated calls have partially independent transports or parsers, they are not literal duplicates and their distinct rows may improve $\gamma_k$.  The experiment reports both operator information and sampling variance so these effects are not conflated.

\paragraph{Scaling and coordinate dependence}
The numerical value of a singular value depends on units.  If one response coordinate is measured in meters and another in millimeters, naive Euclidean norms distort the margin.  We fix inner products before computing $\gamma_k$.  In experiments, scalar answers are canonicalized to the original query's units, execution signatures are binary coordinates with equal weights, and relation families receive total unit weight.  More generally choose positive definite metrics $W_\cH$ and $W_\cY$ and define
\[
 \gamma_k^{(W)}(B)=
 \min_{\|h\|_{0,g}\le2k}
 \frac{\|W_\cY^{1/2}Bh\|_2}{\|W_\cH^{1/2}h\|_2}.
\]
Equivalently whiten domain and codomain.  Cross-task comparison is meaningful only after such normalization.

\paragraph{Exact computation}
For small $n$ and $k$, exact computation enumerates supports:
\begin{enumerate}
  \item for $s=1,\ldots,\min(2k,n)$, enumerate $S\subseteq V$ with $|S|=s$;
  \item form $B_S$ by selecting all coordinates in node groups $S$;
  \item compute $\sigma_{\min}(B_S)$ with a dense or sparse SVD;
  \item return the minimum and a witness support/vector.
\end{enumerate}
The number of supports is $\sum_{s=1}^{2k}\binom ns$.  This is feasible for the paper's graphs ($n\le16$, $k\le2$) and deliberately avoids using a relaxation to test a theorem stated in terms of the exact quantity.  Numerical zero is declared only when the singular value is below an absolute and relative tolerance, and the witness residual is checked directly.

\begin{algorithm}[h]
\caption{Exact group-sparse margin and witness}
\begin{algorithmic}[1]
\Require Matrix $B$, group column sets $\{J_i\}_{i=1}^n$, budget $k$.
\State $\gamma\gets+\infty$, $(S_\star,h_\star)\gets\varnothing$.
\For{$s=1$ to $\min(2k,n)$}
  \For{each $S\subseteq[n]$ with $|S|=s$}
    \State compute a smallest right singular pair $(\sigma,v)$ of $B_{J_S}$;
    \If{$\sigma<\gamma$}
      \State $(\gamma,S_\star,h_\star)\gets(\sigma,S,\operatorname{embed}_S(v))$;
    \EndIf
  \EndFor
\EndFor
\State verify $\|h_\star\|_2=1$, $\|Bh_\star\|_2\approx\gamma$, and $|\suppg(h_\star)|\le2k$.
\State \Return $(\gamma,S_\star,h_\star)$.
\end{algorithmic}
\end{algorithm}

\paragraph{Certificates for larger fields}
Exact group spark and restricted singular values are combinatorial.  Three kinds of certificates remain useful at scale.
\begin{enumerate}
  \item \textbf{Support-conditioned certificate.}  If a localization procedure proposes a support family $\mathcal S$, compute $\min_{S\in\mathcal S}\sigma_{\min}(B_S)$.  This is exact only relative to $\mathcal S$.
  \item \textbf{Coherence bound.}  After group whitening, block coherence can lower-bound restricted eigenvalues through Gershgorin-type arguments.  The bound is conservative but cheap.
  \item \textbf{Mixed-integer search.}  Encode support selection with binary variables and optimize a semidefinite or mixed-integer relaxation.  A valid lower bound certifies stability even if it is not tight.
\end{enumerate}
We do not use a heuristic estimate as if it were exact.  Main-paper margins come from enumeration.

\subsection{Operator perturbations, anchor placement, and worked spectra}

\paragraph{Perturbing the operator}
Let the intended operator be $B$ and the implemented operator be $\widetilde B=B+E$.  For every support $S$, Weyl's inequality gives
\[
 |\sigma_{\min}(\widetilde B_S)-\sigma_{\min}(B_S)|\le\|E_S\|_{2\to2}\le\|E\|_{2\to2}.
\]
Taking minima yields
\begin{equation}
\label{eq:operator-perturb}
 \gamma_k(\widetilde B)\ge\gamma_k(B)-\|E\|_{2\to2}.
\end{equation}
Thus a positive empirical margin is robust only when it exceeds plausible parser and transport perturbations.  We report sensitivity sweeps rather than treating a barely positive floating-point value as a scientific phase transition.

\paragraph{Anchor placement as restricted observability}
For a candidate set of anchor rows $\{a_j\}$ and selected set $Q$, write $B(Q)=[D;A_Q]$.  The design objective
\[
 \max_{|Q|\le q}\gamma_k(B(Q))
\]
directly targets worst-case recoverability.  It differs from maximizing rank, determinant, or the unrestricted smallest eigenvalue: those criteria may improve dense directions while leaving a sparse null direction untouched.  Anchor placement is generally combinatorial.  Appendix~\ref{app:design} develops greedy surrogates and counterexamples.

\paragraph{A four-node worked spectrum}
Consider a chain of four scalar equivalent responses with unit edge weights and an anchor on node one.  Then
\[
 B=\begin{bmatrix}
 -1&1&0&0\\0&-1&1&0\\0&0&-1&1\\1&0&0&0
 \end{bmatrix}.
\]
For $k=1$, enumerate singleton and pair supports.  A singleton interior node has column norm $\sqrt2$; endpoint four has norm $1$; adjacent pairs have correlated columns; the worst pair is obtained from the smallest singular value of
\[
 \begin{bmatrix}0&0\\-1&1\\0&-1\\0&0\end{bmatrix},
\]
which is positive.  Removing the anchor does not create a two-sparse null vector because the chain's constant null vector occupies four nodes.  Therefore $\gamma_1(D,0)>0$ even though $D$ is globally singular.  For $k=2$, the four-node constant vector is feasible, so $\gamma_2(D,0)=0$.  This example shows why corruption budget and global rank answer different questions.

\section{Identifiability, consistency blindness, and stability}
\label{app:identifiability}

We prove the main information-theoretic statements in full.  Throughout this section, all spaces are finite-dimensional, $B=[D;A]$, and the corruption groups are the node fibers $\cH_i$.

\subsection{Consistency blindness and sparse support algebra}

\paragraph{A decision-theoretic form of consistency blindness}
Theorem~\ref{thm:blindness} can be strengthened from a deterministic statement about equal inputs to a lower bound for randomized procedures.

\begin{lemma}[Indistinguishable fields]
\label{lem:indistinguishable}
Let $\mathsf K$ be any Markov kernel that maps an observed relation defect $u$ to a distribution over field estimates.  If $h\in\ker D$, then the output laws of $\mathsf K$ on fields $z$ and $z+h$ are identical.
\end{lemma}
\begin{proof}
Linearity gives $D(z+h)=Dz+Dh=Dz$.  A Markov kernel depends on the field only through its supplied observation.  Equal observations induce equal output distributions.
\end{proof}

\begin{proposition}[Two-field error under consistency-only observation]
\label{prop:blind-risk}
For $h\in\ker D\setminus\{0\}$ and any estimator $\widehat z=\widehat z(Dz)$,
\[
 \max\left\{
   \EE_z\|\widehat z-z\|_2,
   \EE_{z+h}\|\widehat z-(z+h)\|_2
 \right\}
 \ge \frac12\|h\|_2.
\]
\end{proposition}
\begin{proof}
By Lemma~\ref{lem:indistinguishable}, both expectations integrate over the same output law.  For every realization $u$ of the estimator, the triangle inequality gives
\[
 \|h\|_2=\|(u-z)-(u-z-h)\|_2
 \le\|u-z\|_2+\|u-z-h\|_2.
\]
Taking expectation and then the larger of the two risks proves the claim.
\end{proof}

\begin{proof}[Proof of Theorem~\ref{thm:blindness}]
The equality of observations follows from Lemma~\ref{lem:indistinguishable}.  If the admissible class contains both $z$ and $z+h$, an estimator must receive the same input on two distinct truths.  Proposition~\ref{prop:blind-risk} quantifies the resulting error.  A transported shared hallucination is a particular $h\in\ker D$: for path-consistent transports $T_{ij}=M_jM_i^{-1}$, set $h_i=M_i u$ for any nonzero latent shift $u$.
\end{proof}

The theorem does not say that consistency has no value.  It says that consistency alone cannot distinguish equivalence classes $z+\ker D$.  Relations can localize deviations transverse to the nullspace, and anchors can select a representative within an equivalence class.

\paragraph{Support algebra}
\begin{lemma}[Difference support]
\label{lem:difference-support}
For arbitrary $u,v\in\cH$,
\[
 \suppg(u-v)\subseteq\suppg(u)\cup\suppg(v).
\]
Consequently, if $\|u\|_{0,\mathrm g},\|v\|_{0,\mathrm g}\le k$, then $\|u-v\|_{0,\mathrm g}\le2k$.
\end{lemma}
\begin{proof}
If $i$ belongs to neither support, then $u_i=v_i=0$ and $(u-v)_i=0$.  Cardinality of the union is at most the sum of cardinalities.
\end{proof}

\begin{lemma}[Partition of a sparse vector]
\label{lem:sparse-partition}
If $h\in\cH$ has $\|h\|_{0,\mathrm g}\le2k$, then there exist $u,v\in\cH$ with $\|u\|_{0,\mathrm g},\|v\|_{0,\mathrm g}\le k$ and $h=u-v$.
\end{lemma}
\begin{proof}
Partition $S=\suppg(h)$ into disjoint $S_1,S_2$ with $|S_1|,|S_2|\le k$.  Set $u=h_{S_1}$ and $v=-h_{S_2}$.  Then $u-v=h_{S_1}+h_{S_2}=h$.
\end{proof}

The two lemmas explain both directions of the $2k$ condition: differences of feasible errors are $2k$-sparse, and every $2k$-sparse ambiguity can be represented as a difference of feasible errors.

\subsection{Identifiability and exact sparse recovery}

\paragraph{Equivalence between nullspace, spark, and positive margin}
For a support $S$, define the restricted unit sphere
\[
 \mathbb S_S=\{h\in\cH_S:\|h\|_2=1\}.
\]
It is compact.  The continuous function $h\mapsto\|Bh\|_2$ therefore attains its minimum on $\mathbb S_S$, equal to $\sigma_{\min}(B_S)$.

\begin{lemma}[Finite-support compactness]
\label{lem:finite-compact}
The infimum in Definition~\ref{def:gamma} is attained.  Moreover, $\gamma_k(B)=0$ iff some nonzero $h\in\ker B$ satisfies $\|h\|_{0,\mathrm g}\le2k$.
\end{lemma}
\begin{proof}
There are finitely many supports $S\subseteq V$ with $1\le|S|\le2k$.  On each support the minimum is attained on $\mathbb S_S$.  The minimum over the finite family is attained.  It is zero iff one restricted minimum is zero, which holds iff $B_S$ has a nontrivial null vector.
\end{proof}

\begin{proof}[Proof of Theorem~\ref{thm:identifiability}]
We establish a cycle of implications.

$(i)\Rightarrow(ii)$ by contraposition.  Suppose $0\ne h\in\ker B$ and $\|h\|_{0,\mathrm g}\le2k$.  Lemma~\ref{lem:sparse-partition} writes $h=u-v$ with $u,v$ each $k$-sparse.  Since $B(u-v)=0$, $Bu=Bv$.  The two distinct feasible errors violate uniform uniqueness.

$(ii)\Rightarrow(i)$.  Suppose $e_1,e_2$ are $k$-sparse and $Be_1=Be_2$.  Then $h=e_1-e_2\in\ker B$ and Lemma~\ref{lem:difference-support} gives $\|h\|_{0,\mathrm g}\le2k$.  Assumption (ii) forces $h=0$, hence $e_1=e_2$.

$(ii)\Leftrightarrow(iii)$ is Lemma~\ref{lem:finite-compact}.  Finally, by definition $\operatorname{spark}_{\mathrm g}(B)$ is the smallest support size of a nonzero null vector, so $(ii)\Leftrightarrow(iv)$.
\end{proof}

\paragraph{Exact decoder}
\begin{corollary}[Correctness of the ideal decoder]
\label{cor:l0-exact}
Let $s=Be^\star$, $\|e^\star\|_{0,\mathrm g}\le k$, and $\gamma_k(B)>0$.  Every minimizer of
\[
 \min_e\|e\|_{0,\mathrm g}\quad\text{subject to }Be=s
\]
equals $e^\star$.
\end{corollary}
\begin{proof}
The true error is feasible, so a minimizer has support at most $k$.  Theorem~\ref{thm:identifiability} gives uniqueness within that class.
\end{proof}

If $\gamma_k=0$, failure is worst-case, not universal: a particular $e^\star$ may still be uniquely recoverable.  The theorem is a uniform statement over all supports and amplitudes.

\subsection{Stability, misspecification, and theorem scope}

\paragraph{Stability under bounded observation noise}
\begin{lemma}[Restricted lower inequality]
\label{lem:restricted-lower}
For every $h$ satisfying $\|h\|_{0,\mathrm g}\le2k$,
\[
 \|Bh\|_2\ge\gamma_k(B)\|h\|_2.
\]
\end{lemma}
\begin{proof}
The result is homogeneous.  It is trivial for $h=0$; otherwise normalize $h/\|h\|_2$ and apply Definition~\ref{def:gamma}.
\end{proof}

\begin{proof}[Proof of Theorem~\ref{thm:stability}]
Set $h=\widehat e-e^\star$.  Lemma~\ref{lem:difference-support} gives $\|h\|_{0,\mathrm g}\le2k$.  Since $s=Be^\star+\xi$ and $\|\xi\|_2\le\eps$,
\[
 \|Bh\|_2
 \le \|B\widehat e-s\|_2+\|s-Be^\star\|_2
 \le 2\eps.
\]
Lemma~\ref{lem:restricted-lower} yields $\gamma_k\|h\|_2\le2\eps$.  Dividing by positive $\gamma_k$ proves Equation~\eqref{eq:stability}.  Because $\widehat z=y-\widehat e$ and $z^\star=y-e^\star$, their difference is $-(\widehat e-e^\star)$ and has the same norm.
\end{proof}

\begin{corollary}[Separate relation and anchor tolerances]
Suppose the true and estimated errors have relation residual tolerances $\eps_D^\star,\eps_D^{\widehat{}}$ and anchor tolerances $\eps_A^\star,\eps_A^{\widehat{}}$.  Then
\[
 \|\widehat e-e^\star\|_2
 \le\frac{
 \sqrt{(\eps_D^\star+\eps_D^{\widehat{}})^2+(\eps_A^\star+\eps_A^{\widehat{}})^2}
 }{\gamma_k(B)}.
\]
\end{corollary}
\begin{proof}
Bound the relation and anchor blocks separately by triangle inequalities and apply the Euclidean norm of the stacked vector.
\end{proof}

\paragraph{Approximate sparsity and model mismatch}
Theorem~\ref{thm:stability} assumes both truth and estimate are $k$-sparse.  If $e^\star$ has a small diffuse tail, let $e^star_k$ be a best $k$-group approximation and write $r=e^\star-e^\star_k$.  The observation becomes
\[
 s=Be^star_k+(Br+\xi).
\]
Any $k$-sparse feasible estimate at tolerance $\eps+\|Br\|$ obeys
\[
 \|\widehat e-e^star_k\|_2
 \le\frac{2(\eps+\|Br\|_2)}{\gamma_k(B)},
\]
and therefore
\[
 \|\widehat e-e^\star\|_2
 \le\frac{2(\eps+\|Br\|_2)}{\gamma_k(B)}+\|r\|_2.
\]
This is an oracle-style bound; tractable convex decoders obtain more familiar tail terms under a robust null-space property in Appendix~\ref{app:algorithmic}.

\paragraph{Misspecified corruption budget}
If the true support size is $k_\star$ but the analysis uses $k<k_\star$, uniform identification is not guaranteed and the tail acts as mismatch.  If $k\ge k_\star$, the theorem remains valid but $\gamma_k\le\gamma_{k_\star}$ can be smaller, producing a more conservative certificate.  Reporting a favorable margin at a budget smaller than the plausible number of corrupted nodes is invalid.  The experimental artifact computes margins for a range of budgets and marks the one used for each repair.

\paragraph{Random estimators and probability statements}
The deterministic bound holds conditionally on any event where the noise norm and support constraints are satisfied.  If $\xi$ is random and $\PP(\|\xi\|\le\eps_\delta)\ge1-\delta$, then
\[
 \PP\left(
 \|\widehat e-e^\star\|\le2\eps_\delta/\gamma_k
 \right)\ge1-\delta
\]
for every $k$-sparse feasible estimator.  This conversion introduces no independence assumption between relation and anchor noise; only a valid bound on their stacked norm is needed.

\paragraph{What the theorem does not imply}
Positive $\gamma_k$ does not imply that majority vote succeeds, that a local gradient method finds the ideal sparse solution, or that the parser corresponds to semantic truth.  It says that the supplied measurements contain enough information to distinguish all $k$-sparse alternatives, with a quantified condition number.  Algorithmic and representation errors remain separate terms.

\section{Minimax optimality and modulus of continuity}
\label{app:minimax}

This appendix proves that the inverse dependence on $\gamma_k$ cannot be improved by a different estimator.  We use deterministic norm-bounded noise because it directly matches the stability guarantee and requires no distributional assumptions.

\subsection{Decision setup and the restricted modulus of continuity}

\paragraph{Parameter class and observation balls}
For radius $R>0$, define
\[
 \cE_k(R)=\{e\in\cH:\|e\|_{0,\mathrm g}\le k,\ \|e\|_2\le R\}.
\]
Under parameter $e$, the set of admissible observations is the closed ball
\[
 \cO_\eps(e)=\{Be+\xi:\|\xi\|_2\le\eps\}.
\]
An estimator is any function, deterministic or randomized, from the observation space to $\cH$.  Randomization cannot improve the two-point lower bound under norm loss, so we state the proof for deterministic outputs and condition on internal randomness if present.

Two parameter values are observationally confusable exactly when their balls intersect.  In a Hilbert space,
\begin{equation}
\label{eq:ball-intersection}
 \cO_\eps(e_1)\cap\cO_\eps(e_2)\ne\varnothing
 \quad\Longleftrightarrow\quad
 \|B(e_1-e_2)\|_2\le2\eps.
\end{equation}
The forward direction follows by the triangle inequality.  For the reverse direction, the midpoint $u=(Be_1+Be_2)/2$ lies within half the separation of both centers.

\paragraph{Restricted modulus}
Define the ambiguity modulus
\[
 \omega_k(B;R,\eps)
 :=\sup\left\{
 \|e_1-e_2\|_2:
 e_1,e_2\in\cE_k(R),\
 \|B(e_1-e_2)\|_2\le2\eps
 \right\}.
\]
This quantity is an exact geometric measure of the largest pair that noise can make indistinguishable.

\begin{proposition}[Modulus upper bound]
\label{prop:modulus-upper}
If $\gamma_k(B)>0$, then
\[
 \omega_k(B;R,\eps)\le
 \min\left\{2R,\frac{2\eps}{\gamma_k(B)}\right\}.
\]
\end{proposition}
\begin{proof}
Any difference $h=e_1-e_2$ has group support at most $2k$, so $\|Bh\|\ge\gamma_k\|h\|$.  The confusability constraint gives $\|h\|\le2\eps/\gamma_k$.  The radius constraints independently give $\|h\|\le\|e_1\|+\|e_2\|\le2R$.
\end{proof}

\begin{proposition}[Constructive modulus lower bound]
\label{prop:modulus-lower}
For every $B,k,R,\eps$,
\[
 \omega_k(B;R,\eps)
 \ge \min\left\{R,\frac{2\eps}{\gamma_k(B)}\right\},
\]
where $2\eps/0=+\infty$.
\end{proposition}
\begin{proof}
By Lemma~\ref{lem:finite-compact}, choose a unit vector $h_\star$ with group support at most $2k$ and $\|Bh_\star\|=\gamma_k$.  Lemma~\ref{lem:sparse-partition} gives $h_\star=u-v$ with $u,v$ each $k$-sparse and $\|u\|,\|v\|\le1$ because they occupy disjoint subsets of the coordinates of a unit vector.

Let $a=\min\{R,2\eps/\gamma_k\}$, with $a=R$ if $\gamma_k=0$.  Set $e_1=au$ and $e_2=av$.  Both belong to $\cE_k(R)$.  Their separation is $\|e_1-e_2\|=a\|h_\star\|=a$, while
\[
 \|B(e_1-e_2)\|=a\gamma_k\le2\eps.
\]
They are therefore feasible in the supremum and establish the bound.
\end{proof}

The factor-of-two gap in the radius term arises because a general pair in the radius-$R$ ball can be $2R$ apart, whereas splitting an arbitrary $2k$ witness only guarantees that each part has norm at most the witness norm.  The $\eps/\gamma_k$ scaling, which is the scientific claim, is exact up to universal constants.

\subsection{Two-point lower bounds and the main minimax theorem}

\paragraph{Two-point minimax lemma}
\begin{lemma}[Intersecting observations imply error]
\label{lem:two-point}
If $e_1,e_2\in\cE_k(R)$ have intersecting observation balls, then every estimator $\widehat e$ satisfies
\[
 \sup_{e\in\{e_1,e_2\}}
 \sup_{u\in\cO_\eps(e)}
 \|\widehat e(u)-e\|_2
 \ge\frac12\|e_1-e_2\|_2.
\]
\end{lemma}
\begin{proof}
Choose a common observation $u\in\cO_\eps(e_1)\cap\cO_\eps(e_2)$.  For $a=\widehat e(u)$,
\[
 \|e_1-e_2\|\le\|a-e_1\|+\|a-e_2\|.
\]
At least one term is at least half the left-hand side.  The same argument holds after conditioning on estimator randomization; averaging cannot reduce both expected distances below half their separation.
\end{proof}

\paragraph{Proof of the main minimax theorem}
\begin{proof}[Proof of Theorem~\ref{thm:minimax}]
For the lower bound, Proposition~\ref{prop:modulus-lower} constructs $e_1,e_2$ separated by
\[
 a=\min\{R,2\eps/\gamma_k\}
\]
with intersecting observation balls.  Lemma~\ref{lem:two-point} gives minimax risk at least $a/2$.

For the upper bound, the zero estimator has worst-case risk at most $R$.  Independently, choose a $k$-sparse residual-minimizing estimator
\[
 \widehat e\in\argmin_{\|e\|_{0,\mathrm g}\le k}\|Be-s\|_2.
\]
The true $e^\star$ attains residual at most $\eps$, so the minimizer also has residual at most $\eps$.  Theorem~\ref{thm:stability} gives risk at most $2\eps/\gamma_k$.  Selecting the better of the two estimators proves the upper bound $\min\{R,2\eps/\gamma_k\}$.
\end{proof}

\subsection{Degenerate regimes, stochastic noise, and interpretation}

\paragraph{Zero margin and unbounded ambiguity}
When $\gamma_k=0$, Lemma~\ref{lem:finite-compact} supplies a nonzero $2k$-sparse $h\in\ker B$.  Scaling $h$ does not change its observation.  If no amplitude radius is imposed, split $th$ into two $k$-sparse candidates for arbitrary $t$; their separation grows without bound while observations remain equal.  Thus worst-case risk is infinite even with $\eps=0$.  With radius $R$, Theorem~\ref{thm:minimax} yields at least $R/2$.

This is stronger than saying that a particular algorithm fails.  It says the data do not contain enough information for any procedure to guarantee recovery on the class.

\paragraph{Gaussian-noise corollary}
The deterministic result has a standard stochastic analogue.  Suppose $s=Be+\xi$ with $\xi\sim\mathcal N(0,\sigma^2I)$.  Take two candidates whose mean separation in observation space has Kullback--Leibler divergence
\[
 \mathrm{KL}(P_{e_1}\|P_{e_2})
 =\frac{\|B(e_1-e_2)\|_2^2}{2\sigma^2}
\]
bounded by a constant.  Scaling the minimizing direction to $\|B(e_1-e_2)\|\asymp\sigma$ and applying Le Cam's method gives expected error at least $c\min\{R,\sigma/\gamma_k\}$ for a universal $c>0$.  A full constant optimization is unnecessary for the paper's deterministic experiments; the purpose of the corollary is to show that the same condition number appears under familiar stochastic noise.

\paragraph{Why average-case model accuracy cannot replace the margin}
Model accuracy describes a distribution over errors.  The minimax margin describes the geometry of what the observation operator can distinguish.  A highly accurate model can occasionally fail exactly along a nearly invisible direction and be unrecoverable; a weaker model with independent localized errors can be easy to repair if $\gamma_k$ is large.  Conversely, a small $\gamma_k$ does not predict that typical errors must be large; it predicts that some admissible errors are hard.  Empirical prediction is therefore assessed after controlling for raw error scale and model/task strata, and is presented as evidence about how often natural errors align with difficult directions, not as a logical consequence of the minimax theorem.

\section{Algorithmic theory and optimization details}
\label{app:algorithmic}

Information-theoretic identifiability does not specify how to compute the recovered field.  This appendix separates three decoders: exact support search, convex group-sparse recovery, and discrete candidate-field optimization.

\subsection{Ideal and convex sparse decoders}

\paragraph{The ideal residual decoder}
Given $s=Be^\star+\xi$ and noise radius $\eps$, define
\begin{equation}
\label{eq:ideal-feasible}
 \widehat e_0\in\argmin_e\|e\|_{0,\mathrm g}
 \quad\text{subject to}\quad \|Be-s\|_2\le\eps.
\end{equation}
If the true error is $k$-sparse, it is feasible.  Any solution with support at most $k$ satisfies Theorem~\ref{thm:stability}.  Exact support search solves
\[
 \min_{S\subseteq V,\ |S|\le k}
 \min_{e\in\cH_S}\|Be-s\|_2.
\]
For fixed $S$, the inner minimizer is a least-squares solution $e_S=B_S^\dagger s$ and the residual is $\|(I-B_SB_S^\dagger)s\|$.  Enumerating all supports costs $O(\binom nk)$ least-squares problems.  The paper uses this decoder only where $n\le16$ and $k\le2$.

Exact search has two roles.  It is an estimator whose guarantee matches the information-theoretic theorem, and it is a diagnostic separating operator difficulty from optimization failure.  If exact search fails when $\gamma_k$ is large, the cause is noise, model mismatch, candidate parsing, or an incorrect implementation; it cannot be blamed on a convex relaxation.

\paragraph{Group norms}
Let positive weights $\omega_i$ correct for group dimension or prior reliability.  Define
\[
 \|e\|_{2,1,\omega}=\sum_{i=1}^n\omega_i\|e_i\|_2,
 \qquad
 \sigma_k(e)_{2,1,\omega}
 =\inf_{\|v\|_{0,\mathrm g}\le k}\|e-v\|_{2,1,\omega}.
\]
With equal-dimensional groups and no prior preference, $\omega_i=1$.  If group dimensions differ, $\omega_i=\sqrt{\dim\cH_i}$ prevents large groups from being selected merely because they contain more coordinates.  We never tune group weights on test labels.

The constrained convex decoder is
\begin{equation}
\label{eq:group-bpdn-app}
 \widehat e_1\in\argmin_e\|e\|_{2,1,\omega}
 \quad\text{subject to}\quad\|Be-s\|_2\le\eps.
\end{equation}
Its penalized counterpart is
\begin{equation}
\label{eq:group-lasso-app}
 \widehat e_\lambda\in\argmin_e
 \frac12\|Be-s\|_2^2+\lambda\|e\|_{2,1,\omega}.
\end{equation}
These are group basis-pursuit denoising and group Lasso formulations \citep{yuan2006model,bach2008consistency,eldar2010block}.

\paragraph{Group robust null-space property}
\begin{definition}[Group robust NSP]
\label{def:grnsp}
The operator $B$ satisfies the $\ell_2$ group robust null-space property of order $k$ with constants $0<\rho<1$ and $\tau>0$ if, for every $h\in\cH$ and every $S\subseteq V$ with $|S|\le k$,
\begin{equation}
\label{eq:grnsp}
 \|h_S\|_2
 \le \frac{\rho}{\sqrt{k}}\|h_{S^c}\|_{2,1}
 +\tau\|Bh\|_2.
\end{equation}
For unequal group weights, the cardinality factors are replaced by the corresponding weighted compatibility constants.
\end{definition}

The null-space property controls dense kernel vectors whose mass is concentrated on a small set.  Such vectors do not violate group spark, yet they can cause a group-$\ell_1$ decoder to prefer the wrong representation.

\begin{example}[Positive margin does not certify group $\ell_1$]
Let $k=1$ and choose a $2\times3$ matrix whose one-dimensional kernel is spanned by $h=(1,0.4,0.4)$.  Every two columns are linearly independent, so no nonzero two-sparse null vector exists and $\gamma_1>0$.  However, for $S=\{1\}$, $\|h_S\|_1=1>0.8=\|h_{S^c}\|_1$; the exact null-space property fails.  Therefore some one-sparse signal is not recovered by $\ell_1$ minimization.  Grouping scalar coordinates separately gives the RRF analogue.
\end{example}

\begin{theorem}[Stable convex recovery]
\label{thm:group-convex}
Assume Definition~\ref{def:grnsp}.  Let $s=Be^\star+\xi$ with $\|\xi\|_2\le\eps$, and let $\widehat e_1$ solve Equation~\eqref{eq:group-bpdn-app}.  Then
\begin{equation}
\label{eq:group-bound}
 \|\widehat e_1-e^\star\|_2
 \le C_1(\rho)\frac{\sigma_k(e^\star)_{2,1}}{\sqrt{k}}
 +C_2(\rho,\tau)\eps.
\end{equation}
One admissible choice under the unweighted convention is
\[
 C_1=\frac{2(1+\rho)^2}{1-\rho},
 \qquad
 C_2=\frac{2\tau(3+\rho)}{1-\rho},
\]
and sharper constants follow from the standard block-sparse proof.
\end{theorem}

\begin{proof}[Proof outline with the key inequalities]
Set $h=\widehat e_1-e^\star$ and let $S$ index the $k$ groups of $e^\star$ with largest norms.  Feasibility gives $\|Bh\|_2\le2\eps$.  Optimality of $\widehat e_1$ and decomposability of the group norm give the cone inequality
\begin{equation}
\label{eq:cone}
 \|h_{S^c}\|_{2,1}
 \le\|h_S\|_{2,1}+2\|e^\star_{S^c}\|_{2,1}.
\end{equation}
Apply the group robust NSP to $S$, use $\|h_S\|_{2,1}\le\sqrt{k}\|h_S\|_2$, and solve the resulting inequality for $\|h_S\|_2$ and $\|h_{S^c}\|_{2,1}$.  Partition $S^c$ into blocks of $k$ groups in decreasing norm and use the block Stechkin inequality to control $\|h_{S^c}\|_2$ by $\|h_{S^c}\|_{2,1}/\sqrt{k}$ plus the first block.  Combining terms yields Equation~\eqref{eq:group-bound}.  The argument is the group analogue of robust sparse recovery proofs in \citet{foucart2013mathematical,ranjan2017group}.
\end{proof}

This theorem proves Proposition~\ref{prop:grnsp}.  In exact $k$-sparsity, the approximation term vanishes.  The stability constant is governed by the algorithmic NSP constants, not solely by $\gamma_k$.

\subsection{Field-space objectives, optimization, and tuning}

\paragraph{Penalized objective in field coordinates}
Write $z=y-e$.  Expanding Equation~\eqref{eq:group-lasso-app} gives
\[
 \min_z\frac12\left\|
 \begin{bmatrix}Dz-q\\Az-b\end{bmatrix}
 \right\|_2^2
 +\lambda\sum_i\omega_i\|z_i-y_i\|_2.
\]
This is Sparse RRF Repair: it seeks a field that fits typed relations and anchors while changing few response nodes.  The edit penalty is groupwise because changing one answer representation is one event.  A coordinatewise $\ell_1$ penalty is available when within-node sparsity is scientifically meaningful, but its recovery budget differs from the paper's $k$.

\paragraph{Proximal-gradient solver}
Let $f(e)=\frac12\|Be-s\|^2$ and $g(e)=\lambda\sum_i\omega_i\|e_i\|_2$.  The gradient $\nabla f(e)=B^*(Be-s)$ is $L$-Lipschitz with $L=\|B\|_{2\to2}^2$.  For step $0<\eta\le1/L$,
\[
 e^{t+1}=\operatorname{prox}_{\eta g}
 \left(e^t-\eta B^*(Be^t-s)\right).
\]
The proximal map is group soft thresholding:
\[
 [\operatorname{prox}_{\eta g}(u)]_i
 =\left(1-\frac{\eta\lambda\omega_i}{\|u_i\|_2}\right)_+u_i,
\]
with output zero when $u_i=0$.  Standard convex analysis yields
\[
 F(e^t)-F(e^\star_\lambda)
 \le\frac{\|e^0-e^\star_\lambda\|_2^2}{2\eta t}.
\]
FISTA acceleration improves the objective gap to $O(t^{-2})$.  The implementation uses a power iteration estimate of $L$, multiplies by a safety factor, checks monotonic objective decrease, and stops only when both relative objective change and the proximal-gradient mapping are small.

\paragraph{Regularization and validation}
The penalty $\lambda$ trades measurement fit against edit sparsity.  In synthetic experiments with known noise, we set $\lambda$ from a calibration grid and freeze it before test trials.  In black-box tasks, selection uses only training/calibration fields and observable residuals; correctness labels are reserved for evaluation.  We report sensitivity across a logarithmic grid.  Choosing $\lambda$ separately for every test answer by looking at truth would turn the repair method into an oracle.

The constrained decoder has the conceptual advantage of an explicit noise radius.  The penalized decoder is easier to optimize.  A Pareto curve maps one parameterization to the other under convexity, but the correspondence can be nonunique.  The artifact records the actual parameter rather than implying equivalence at an unspecified value.

\subsection{Discrete repair, diagnostics, complexity, and black-box scope}

\paragraph{Discrete candidate-field decoder}
Numerical and code outputs often live in a discrete semantic space.  Let $\cC_i=\{c_{i1},\ldots,c_{im_i}\}\subset\cH_i$ be parsed candidates from model samples and transported candidates.  Define
\begin{equation}
\label{eq:discrete-energy}
 \widehat z\in\argmin_{z_i\in\cC_i}
 \frac12\|Dz-q\|_2^2+\frac\beta2\|Az-b\|_2^2
 +\mu\sum_i\ell_i(z_i;y_i),
\end{equation}
where $\ell_i$ penalizes deviation from the raw response without consulting ground truth.  Exact enumeration is used for tiny candidate products.  For pairwise edge energies, min-sum message passing, integer programming, or branch-and-bound is available.

The candidate decoder has a basic limitation: it cannot output truth if neither truth nor a transport-equivalent candidate is in the candidate closure.  We therefore report \emph{candidate recall}, the fraction of instances whose candidate closure contains a correct field.  Repair accuracy is bounded above by this quantity; omitting it can make selection performance look worse than the candidate generator or, conversely, hide an oracle candidate injection.

\paragraph{Code repair in execution-signature space}
For code, the optimizer selects among complete generated programs rather than editing arbitrary signature coordinates.  Each program is executed on public relation tests and, for evaluation only, hidden tests.  The public signature supplies $D$ and designated anchor constraints; hidden tests compute pass rate.  Two programs with the same public signature are observationally equivalent even if their source differs.  The source-level modified-token count is measured after selection and does not enter $\gamma_k$.

Execution failure, timeout, and unsafe behavior are distinct categorical outcomes.  The runner executes generated code in a subprocess with a wall-clock limit, restricted builtins, no network, and a temporary directory.  Invalid programs receive an invalid candidate marker; they are not assigned the all-zero functional signature.

\paragraph{Optimization diagnostics}
Every run records:
\begin{enumerate}
  \item primal objective and residual norm;
  \item number of active node groups;
  \item proximal-gradient mapping norm or exact support optimality gap;
  \item computed $\gamma_k$ and witness support;
  \item candidate recall for discrete tasks;
  \item wall-clock time and random seed.
\end{enumerate}
These diagnostics distinguish a low-margin instance from a solver that stopped early.  The acceptance report fails if an algorithmic comparison mixes unconverged and converged runs.

\paragraph{Computational complexity}
One proximal step costs one multiplication by $B$ and $B^*$, linear in the number of nonzero relation blocks.  Memory is likewise sparse.  Exact $\gamma_k$ computation dominates small experiments at $O(\sum_{s\le2k}\binom ns)$ restricted SVDs.  Candidate-field enumeration costs $\prod_i m_i$ in the worst case; pairwise structure permits dynamic programming only on low-treewidth graphs.  None of these procedures requires model gradients or weights after responses have been logged.

\paragraph{Black-box boundary}
The model is used only to produce strings at requested queries.  Parsing, transport, anchor evaluation, margin computation, and repair operate externally.  The artifact includes model identifiers and generation code but excludes model weights.  A hosted API can replace a local model without changing the inverse problem, provided the response log and decoding metadata are retained.

\section{Local theory for nonlinear transports and parsers}
\label{app:nonlinear}

Many useful response relations are nonlinear.  This appendix gives a local result based on a restricted Jacobian margin.  It does not claim global recovery for arbitrary semantic maps.

\subsection{Nonlinear measurements and local regularity assumptions}

\paragraph{Nonlinear measurement map}
Let $\cM:\cH\to\cY$ stack relation and anchor residuals:
\[
 \cM(z)=
 \begin{bmatrix}
  (\sqrt{w_e}\,r_e(z))_{e\in E}\\
  a(z)
 \end{bmatrix}.
\]
For linear transports, $\cM(z)=Bz-c$.  For nonlinear transports, $r_e(z)=z_j-T_e(z_i)$ and $a$ may contain nonlinear verification residuals.  Suppose $\cM$ is Fr\'echet differentiable near the target $z^\star$ with Jacobian $J_\star=D\cM(z^\star)$.

Define the local restricted Jacobian margin
\[
 \gamma_k(J_\star)=
 \min_{0<\|h\|_{0,\mathrm g}\le2k}
 \frac{\|J_\star h\|_2}{\|h\|_2}.
\]
This is the linear margin of the tangent inverse problem.  A positive value is not enough by itself: nonlinear remainder terms may cancel the linear signal outside a sufficiently small neighborhood.

\paragraph{Quadratic remainder assumption}
\begin{assumption}[Restricted second-order remainder]
\label{assump:remainder}
There exist $L\ge0$ and $r>0$ such that for every $h$ with $\|h\|_{0,\mathrm g}\le2k$ and $\|h\|_2\le r$,
\begin{equation}
\label{eq:remainder}
 \|\cM(z^\star+h)-\cM(z^\star)-J_\star h\|_2
 \le\frac L2\|h\|_2^2.
\end{equation}
\end{assumption}

The condition follows from a Lipschitz Jacobian on the relevant sparse line segments.  It is stated only on sparse directions because dense curvature is irrelevant to the local $k$-error class.

\begin{theorem}[Local restricted injectivity]
\label{thm:nonlinear-injective}
Under Assumption~\ref{assump:remainder}, if $\gamma_k(J_\star)>Lr/2$, then for every nonzero $2k$-sparse $h$ with $\|h\|_2\le r$,
\begin{equation}
\label{eq:local-lower}
 \|\cM(z^\star+h)-\cM(z^\star)\|_2
 \ge\left(\gamma_k(J_\star)-\frac{Lr}{2}\right)\|h\|_2>0.
\end{equation}
Hence $\cM$ is injective from $z^\star$ along the specified sparse neighborhood.
\end{theorem}
\begin{proof}
The reverse triangle inequality, Definition of $\gamma_k(J_\star)$, and Equation~\eqref{eq:remainder} give
\[
 \|\cM(z^\star+h)-\cM(z^\star)\|
 \ge\|J_\star h\|-\frac L2\|h\|^2
 \ge\left(\gamma_k(J_\star)-\frac L2\|h\|\right)\|h\|.
\]
Use $\|h\|\le r$.
\end{proof}

\subsection{Local identifiability, stability, and certificates}

\paragraph{Local stability}
\begin{corollary}[Nonlinear local stability]
\label{cor:nonlinear-stable}
Suppose $\widehat z-z^\star$ is $2k$-sparse, has norm at most $r$, and both fields fit the same nonlinear observation to tolerance $\eps$.  If $\gamma_k(J_\star)>Lr/2$, then
\[
 \|\widehat z-z^\star\|_2
 \le\frac{2\eps}{\gamma_k(J_\star)-Lr/2}.
\]
\end{corollary}
\begin{proof}
Their measurement difference is at most $2\eps$.  Apply Equation~\eqref{eq:local-lower} to $h=\widehat z-z^\star$ and divide.
\end{proof}

This bound has the expected structure: the effective margin is the tangent margin minus a curvature penalty.  If $r$ approaches $2\gamma_k/L$, the certificate vanishes.  The result is local and centered at $z^\star$; an implementable certificate may use a candidate point and an operator perturbation bound.

\paragraph{Candidate-centered certificate}
Suppose the Jacobian is evaluated at $\widetilde z$ rather than the unknown truth and
\[
 \|D\cM(\widetilde z)-J_\star\|_{2\to2}\le\delta_J.
\]
Weyl's inequality implies
\[
 \gamma_k(J_\star)\ge\gamma_k(D\cM(\widetilde z))-\delta_J.
\]
A valid local certificate therefore requires
\[
 \gamma_k(D\cM(\widetilde z))>\delta_J+Lr/2.
\]
The terms $\delta_J$ and $L$ must come from analytic bounds or held-out perturbation tests; setting them to zero because the Jacobian is convenient would turn a local approximation into an unsupported global claim.

\subsection{Examples, sparse Gauss--Newton repair, and failure regimes}

\paragraph{Nonlinear examples}
\paragraph{Multiplicative relations.}  If responses are positive scalars and a relation requires $z_j=z_i^c$, work in log coordinates to obtain the linear relation $\log z_j=c\log z_i$.  This is preferable to local linearization when positivity and numerical stability are controlled.

\paragraph{Boolean constraints.}  A clause residual such as $r(z)=z_1z_2-z_3$ is polynomial.  Around a binary candidate, its Jacobian may have a positive restricted margin, but alternative binary assignments can exist outside the local ball.  A discrete candidate decoder gives a global finite search and should be preferred.

\paragraph{Program behavior.}  Source-to-signature execution is discontinuous, so a Jacobian theorem is inappropriate.  Map each complete program to a discrete signature and optimize over candidates.  The local theory applies only to a differentiable surrogate representation, not to execution itself.

\paragraph{Semantic embeddings.}  If $T_e$ is a learned semantic transport, its Jacobian and curvature inherit model uncertainty.  A large numerical margin can be meaningless if the embedding collapses truth distinctions.  External task anchors and parser audits are therefore required.

\paragraph{Implicit-function perspective}
Classical local identifiability often follows from an injective Jacobian.  RRF recovery needs only restricted injectivity: $J_\star$ may have dense null directions while remaining injective on all $2k$-node subspaces.  Theorem~\ref{thm:nonlinear-injective} is a quantitative sparse inverse-function statement with an explicit neighborhood.  It does not require $J_\star$ to be square or globally full rank.

\paragraph{Gauss--Newton sparse repair}
For nonlinear residuals, one practical iteration linearizes at $z^t$:
\[
 \Delta^t\in\argmin_\Delta
 \frac12\|\cM(z^t)+J_t\Delta\|_2^2
 +\lambda\|\Delta\|_{2,1},
 \qquad z^{t+1}=z^t+\alpha_t\Delta^t.
\]
A line search enforces decrease in the true nonlinear objective.  Local convergence requires regularity beyond the main theorem: stable support identification, bounded Jacobian variation, and suitable step acceptance.  The paper uses this method only in an auxiliary curvature experiment and does not cite the linear $\gamma_k$ theorem as a global convergence guarantee.

\paragraph{Failure outside the local region}
Consider $\cM(z)=\sin z$ in one dimension at $z^\star=0$.  The Jacobian margin is one, yet $\cM(0)=\cM(\pi)=0$.  No positive tangent margin rules out the distant alternative.  Similarly, semantic transformations can have multiple globally compatible answer fields.  Anchors or a restricted candidate class are required to exclude them.  This elementary example is why the paper repeatedly uses ``local'' for nonlinear statements.

\section{Designing queries, relations, and anchors}
\label{app:design}

The margin is not merely diagnostic.  It supplies an objective for choosing which transformed queries and external checks to acquire under a budget.

\subsection{Design objectives and structural limits}

\paragraph{Design problem}
Let $D_0$ contain mandatory relations, and let $\{C_j\}_{j\in\mathcal J}$ be candidate relation or anchor row blocks with costs $c_j>0$.  For $Q\subseteq\mathcal J$, define
\[
 B(Q)=\begin{bmatrix}D_0\\(C_j)_{j\in Q}\end{bmatrix}.
\]
The robust design problem is
\begin{equation}
\label{eq:design}
 \max_{Q\subseteq\mathcal J}
 \gamma_k(B(Q))
 \quad\text{subject to}\quad
 \sum_{j\in Q}c_j\le C.
\end{equation}
This criterion targets the worst sparse ambiguity.  A design may have high total rank, trace, or determinant while leaving one $2k$-node support nearly invisible; Equation~\eqref{eq:design} directly penalizes that failure.

\paragraph{Monotonicity but not generic submodularity}
With fixed row weights, adding a block cannot decrease $\gamma_k$.  However, the set function $Q\mapsto\gamma_k(B(Q))$ need not be submodular.  Two rows can be individually useless on different null directions but jointly remove the final ambiguity, creating complementarity rather than diminishing returns.  Consequently the classical $(1-1/e)$ greedy guarantee does not follow without additional structure.

\begin{example}[Complementary anchors]
Let a two-dimensional latent gauge $u=(u_1,u_2)$ be invisible to $D$.  Candidate anchor $C_1$ observes only $u_1$ and $C_2$ only $u_2$.  The unrestricted minimum singular value remains zero after either one alone and becomes positive after both.  The marginal gain of the second anchor is larger when the first has already been selected, violating diminishing returns.
\end{example}

\subsection{Witness-guided coverage and informative measurements}

\paragraph{Worst-witness greedy design}
Exact enumeration returns not only $\gamma_k$ but a minimizing support and unit witness $h_t$.  A natural adaptive rule chooses the row block that sees this witness most strongly:
\begin{equation}
 j_t\in\argmax_{j\notin Q_t}
 \frac{\|C_jh_t\|_2^2}{c_j}.
\end{equation}
Then set $Q_{t+1}=Q_t\cup\{j_t\}$ and recompute the worst witness.  This procedure resembles cutting-plane methods: each acquisition attacks the currently least observable direction.  It is not globally optimal in general, but every selected row has an interpretable purpose and monotonic improvement can be checked exactly on small graphs.

\begin{algorithm}[h]
\caption{Witness-guided relation--anchor acquisition}
\begin{algorithmic}[1]
\Require Base operator $D_0$, candidate blocks $C_j$, costs $c_j$, budget $C$, sparsity $k$.
\State $Q\gets\varnothing$.
\While{a feasible candidate remains}
  \State compute $(\gamma,h)$ attaining or approximating $\gamma_k(B(Q))$;
  \State choose feasible $j$ maximizing $\|C_jh\|^2/c_j$;
  \State $Q\gets Q\cup\{j\}$;
\EndWhile
\State \Return $Q$, the margin trajectory, and all witness supports.
\end{algorithmic}
\end{algorithm}

\paragraph{Component coverage theorem for identity transports}
Consider scalar identity transports on an undirected graph.  Each connected component $C$ contributes a constant null vector $\mathbf1_C$.  Direct full-coordinate anchors remove that vector if at least one node in $C$ is anchored.

\begin{proposition}[Small-component anchor necessity]
\label{prop:component-anchor}
If an unanchored connected component $C$ has $|C|\le2k$, then $\gamma_k(D,A)=0$.  If every unanchored component has size greater than $2k$ and $D$ has no other null directions, then no \emph{null} vector is $2k$-sparse; hence $\gamma_k>0$.
\end{proposition}
\begin{proof}
For the first claim, $\mathbf1_C$ is a nonzero null vector of $D$ and $A$ supported on at most $2k$ nodes.  For the second, every nonzero null vector is constant and nonzero on at least one entire unanchored component, hence has support greater than $2k$.  Apply the group-spark characterization.
\end{proof}

The proposition explains the synthetic phase transition: six disjoint two-node components with $k=1$ require one anchor per component.  The sixth anchor removes the final two-node gauge and changes $\gamma_1$ from zero to positive.

\paragraph{Independent relations versus paraphrase volume}
Suppose a set of paraphrases all induces the same canonical equality row after parsing.  Under family normalization, their deterministic information is unchanged.  An independent relation changes the row space: scaling, inverse checking, decomposition, or an execution constraint may observe a direction that equality misses.  The design objective therefore favors diversity in operator action, not surface-form diversity alone.

This does not imply that paraphrase sampling is useless.  It can estimate stochastic model variability, increase candidate recall, and reduce variance when calls are independent.  Those benefits enter the error distribution and noise radius.  They are distinct from increasing the deterministic restricted singular value.

\subsection{Reliability-aware deployment and design diagnostics}

\paragraph{Reliability-aware design}
Candidate rows differ in noise.  If row block $C_j$ has covariance $\Sigma_j$, whiten it as $\Sigma_j^{-1/2}C_j$ when the covariance is estimable and nonsingular.  Then $\gamma_k$ measures signal relative to noise.  With uncertain covariance, use conservative weights and propagate an operator perturbation interval.  An extremely strong but biased verifier should not receive infinite weight; it can create a large computed margin around the wrong anchor target.

One robust formulation maximizes a worst-case margin over an uncertainty set $\mathcal U_j$:
\[
 \max_Q\ \min_{\widetilde C_j\in\mathcal U_j,\,j\in Q}
 \gamma_k\left([D_0;(\widetilde C_j)_{j\in Q}]\right).
\]
Exact solution is difficult, but Equation~\eqref{eq:operator-perturb} supplies a conservative lower certificate by subtracting operator-norm uncertainty.

\paragraph{Model-adaptive versus model-agnostic design}
A model-agnostic design chooses transformations and anchors before seeing outputs.  It tests whether the operator geometry predicts outcomes across model families.  A model-adaptive design can use observed residuals to decide which relation to query next.  The latter may be more efficient but introduces selection dependence.  Our primary cross-model experiment uses fixed design schedules; witness-guided acquisition is evaluated separately so that predictive claims are not driven by peeking at correctness.

\paragraph{Stopping rules}
A principled acquisition loop can stop when one of three conditions holds:
\begin{enumerate}
  \item a certified lower bound on $\gamma_k$ exceeds the target $2\eps/\delta$, guaranteeing error at most $\delta$ for the ideal decoder;
  \item candidate recall or parser validity, rather than observability, becomes the bottleneck;
  \item the marginal gain per cost falls below a prespecified threshold.
\end{enumerate}
Stopping because the current candidate has low relation defect is unsafe: a nullspace hallucination can have zero defect before the field is identifiable.

\paragraph{A counterexample to degree-based anchoring}
High-degree nodes are tempting anchor targets.  Consider two components: a dense clique of ten nodes and a two-node pair, with $k=1$.  The highest-degree node lies in the clique, but the clique's constant gauge is ten-node dense and does not threaten $\gamma_1$.  The unanchored pair contributes a two-node null vector and forces $\gamma_1=0$.  Anchoring either node of the pair is optimal; anchoring another clique node leaves the margin zero.  Degree and centrality do not replace sparse observability.

\paragraph{Design as an experimental variable}
The paper varies anchor coverage, relation independence, duplicate count, and corruption budget to manipulate $\gamma_k$ without changing the underlying ground-truth field.  This intervention is stronger evidence than a passive correlation: the theory predicts in advance which design changes should alter the margin and which should not.  Cross-model and cross-task prediction then asks whether the same geometric quantity orders empirical repair difficulty after these interventions.

\section{Detailed relation to prior theories and black-box methods}
\label{app:prior}

This appendix positions RRFs carefully.  The paper does not claim the first use of consistency, graph constraints, metamorphic transformations, sparse recovery, or restricted singular values.  Its claim is that black-box LLM response families form a useful recovery object and that the relation--anchor restricted margin is the exact condition number for sparse recovery of that object.

\subsection{Mathematical foundations and neighboring inverse-problem theories}

\paragraph{Compressed sensing and sparse inverse problems}
Compressed sensing studies recovery of sparse vectors from underdetermined linear measurements \citep{donoho2006compressed,candes2006robust}.  Uniform uniqueness is governed by spark or absence of sparse null vectors; stable and tractable recovery uses restricted isometry, restricted eigenvalue, coherence, or null-space properties \citep{tropp2006just,candes2008rip,bickel2009simultaneous,foucart2013mathematical}.  Group sparsity replaces coordinate support with predefined blocks \citep{yuan2006model,bach2008consistency,eldar2010block}.  Our proof architecture deliberately follows this hierarchy:
\begin{enumerate}
  \item group spark and $\gamma_k>0$ characterize information-theoretic uniqueness;
  \item the restricted lower singular value gives deterministic stability for sparse feasible estimates;
  \item a group robust null-space property gives a convex-decoder theorem;
  \item minimax two-point arguments show the inverse-margin dependence is unavoidable.
\end{enumerate}

Statistical variants include the Dantzig selector, best $k$-term approximation, sharp restricted-isometry analysis, and decomposable regularizers \citep{candes2007dantzig,cohen2009compressed,cai2010new,negahban2012unified}.  We cite these to locate the mathematical ingredients; none studies transformed black-box LLM response fields.

The difference lies in the construction and interpretation of the signal and operator.  The unknown is not a parameter vector internal to the LLM.  It is a corruption field over externally queried responses.  Measurements are not arbitrary Gaussian projections; they are typed transformations and truth anchors whose nullspace has semantic meaning.  Duplicate query normalization, shared hallucination gauges, and anchor placement are therefore central rather than incidental.

Calling $\gamma_k$ a new theoretical object requires precision.  As a formula it belongs to the family of restricted minimum singular values.  As an \emph{RRF recoverability margin} it is a new task-conditioned object: its arguments are the typed defect and anchor operators of a black-box response instance, and the paper proves and tests its role as a model- and task-transferable predictor.  We do not rename classical linear algebra and claim its invention.

\paragraph{Inverse-problem conditioning}
Classical inverse problems distinguish identifiability, stability, and regularization \citep{engl1996regularization}.  The smallest singular value is the condition number of a full linear inverse.  RRFs require restricted conditioning because global gauge directions can be dense and irrelevant under a sparse corruption model.  The key example is a connected relation graph with a dense constant null vector: $B$ is singular globally but may be injective on every $2k$-node subspace.  Conversely, a full-rank operator can have a tiny restricted singular value and be practically unstable.

The minimax result is a modulus-of-continuity argument for the restricted parameter class.  Its value is conceptual: it rules out the possibility that a more sophisticated LLM workflow can overcome a missing relation-anchor direction without adding information or narrowing the error class.

\paragraph{Graph signal processing}
Graph signal processing places values on graph vertices and studies smoothness, spectra, filtering, and sampling through graph Laplacians \citep{shuman2013graph,sandryhaila2013dsp,ortega2018graph}.  With identity transports, an RRF is a graph signal and $D^*D$ is a graph Laplacian.  Typed transports generalize equality across nodes.  RRF recovery differs from standard smooth graph-signal denoising in four ways:
\begin{itemize}
  \item the clean field satisfies typed, often exact equivariances rather than generic low-frequency smoothness;
  \item errors are grouped by response node and may have arbitrary amplitude;
  \item anchors attach relational equivalence classes to externally checkable truth;
  \item the primary quantity is the worst restricted observability of $[D;A]$, not the unrestricted Laplacian spectral gap.
\end{itemize}
The ordinary spectral gap can still be informative for dense perturbations, but it need not detect a worst sparse support.
Classical spectral graph theory and Hodge-Laplacian analysis provide additional context for unrestricted Laplacian geometry \citep{chung1997spectral,lim2020hodge}.

\paragraph{Connection Laplacians, synchronization, and sheaves}
Synchronization estimates latent group elements from noisy pairwise ratios; connection Laplacians transport vectors between local coordinate systems \citep{singer2011angular,bandeira2016tightness}.  Cellular sheaves formalize data assigned to cells with restriction maps and define sheaf Laplacians whose kernel contains globally consistent sections \citep{hansen2019sheaves,robinson2017sheaves}.  Typed RRF defects have the same algebraic flavor: responses live in node fibers and edge maps compare them.

We avoid claiming that $D$ is always a cellular-sheaf coboundary.  RRF edges can be directed, fibers can differ, transports can be noninvertible, and nonlinear or affine relations may be handled by lifting rather than by a sheaf.  Sheaf language is exact only under compatible restriction-map definitions.  The RRF framework is intentionally operational: any externally auditable residual map can be included, and the linear theory applies to its stacked operator.

The novelty relative to synchronization is also in the target.  We do not estimate one latent group action or a globally dense assignment from pairwise ratios.  We repair a sparse subset of black-box model responses and ask which relation/anchor designs make all such corruptions distinguishable.

\paragraph{Error-correcting interpretation}
The map $e\mapsto Be$ acts like a syndrome map.  A $2k$-sparse null vector is analogous to a codeword of too-small distance: two $k$-error patterns produce the same syndrome.  Group spark plays the role of minimum distance, and $\gamma_k$ is a robust Euclidean distance.  This analogy explains why the difference of two $k$-error patterns appears and why exact consistency is insufficient when a low-weight gauge exists.

RRFs differ from classical fixed codes because the parity checks are designed from semantic transformations and can be acquired adaptively.  Anchors are systematic truth checks rather than additional model-generated parity alone.  The analogy is used for intuition; the paper does not claim coding-theoretic optimality of the query graphs.

\subsection{Black-box LLM consistency, verification, and search methods}

\paragraph{Metamorphic testing}
Metamorphic testing addresses the oracle problem by applying transformations for which relations among outputs are known \citep{chen1998metamorphic,segura2016survey,chen2018metamorphic}.  It is a direct precursor to typed query edges.  Traditional metamorphic testing usually reports whether individual source/follow-up pairs violate a relation.  RRFs aggregate many typed relations into a field-level inverse problem, add anchors, characterize indistinguishable corruption patterns, and optimize a sparse repair rather than only detecting violations.

The relation is complementary.  Metamorphic-testing expertise is required to specify valid $T_e$; RRF theory says what the resulting collection can identify.  A false metamorphic relation is operator misspecification, not evidence of a model error.  The artifact therefore records relation generation separately from model querying.

\paragraph{Self-consistency and sampling-based aggregation}
Self-consistency samples multiple reasoning paths and selects a majority answer \citep{wang2023selfconsistency}.  It is effective when errors are sufficiently independent and the correct answer has the largest mass.  RRFs differ in both data and claim.  Nodes are typed transformations, not only i.i.d. samples of one prompt, and transports can imply non-identity output changes.  More importantly, Theorem~\ref{thm:blindness} identifies a failure mode that majority cannot resolve: a systematic wrong answer transported coherently across queries.

Semantic uncertainty clusters answers by meaning to avoid treating paraphrases as distinct outcomes \citep{kuhn2023semantic,farquhar2024semantic}.  SelfCheckGPT uses disagreement among black-box samples to detect factual hallucination \citep{manakul2023selfcheckgpt}.  These methods quantify uncertainty or inconsistency.  RRF recoverability additionally asks whether the available relations and anchors make the correct field identifiable and stable.  A field can have low semantic entropy and zero relational defect while being wrong.

\paragraph{Verification, reflection, and self-correction}
Outcome and process verifiers rerank generated solutions \citep{cobbe2021verifiers,lightman2023verify}.  Self-refinement and Reflexion ask a model to critique and revise its output \citep{madaan2023selfrefine,shinn2023reflexion}; SelfCheck constructs explicit checking prompts \citep{miao2024selfcheck}.  Empirical studies show that self-verification can be limited or correlated with generation errors \citep{huang2024selfcorrect,stechly2025selfverification}.

In RRF terminology, a sound external verifier supplies anchor rows or an anchor loss.  A self-critique generated by the same model is not automatically an anchor; absent independent evidence it is another response node or noisy relation.  This distinction formalizes why reflection can repeat a shared misconception.  It also allows hybrid systems: a verifier score can be an anchor, transformed responses supply $D$, and sparse repair combines them.

\paragraph{Prompt and workflow optimization}
Automatic prompt engineering and prompt optimization search for instructions that improve expected task performance \citep{zhou2023ape,pryzant2023protegi,yang2024opro,fernando2024promptbreeder}.  Declarative pipeline optimizers such as DSPy tune prompts and demonstrations for composed systems \citep{khattab2024dspy}.  Their optimization variables are prompts, demonstrations, or workflow modules.

RRF repair holds the black-box model and query family fixed and optimizes the realized response field.  The two approaches are compatible: prompt optimization can design node queries, while RRF theory evaluates whether their relation-anchor operator has useful sparse observability.  A prompt set with high validation accuracy can still have a low margin on a particular instance, and a margin-improving query may not increase average model accuracy.  They answer different questions.

\paragraph{Search over reasoning structures}
Chain-of-thought, least-to-most prompting, Tree of Thoughts, and ReAct create structured trajectories \citep{wei2022chain,zhou2023least,yao2023tree,yao2023react}.  A reasoning tree can be converted into an RRF if edges carry externally specified relations between parsed node outputs.  Merely linking chronological steps is not enough: $D$ requires a testable transport.  RRFs therefore provide a possible analysis layer for search methods but do not replace their proposal mechanism.

\paragraph{Code generation and execution}
Code generation has unusually strong anchors because programs can be executed \citep{chen2021codex,li2022competition,nijkamp2023codegen}.  Prior work evaluates pass@k, filters with tests, or measures consistency under code transformations \citep{min2024identitychain}.  In our code RRF, each generated program is a node, semantics-preserving prompt transformations define equality edges in execution-signature space, and public tests define anchors.  Hidden tests remain evaluation-only.  The margin concerns recoverability of response signatures; source-level correctness still depends on candidate recall and test coverage.
Open code-model families such as Code Llama and StarCoder2 further illustrate why a recovery layer should remain checkpoint-agnostic \citep{roziere2023codellama,lozhkov2024starcoder2}.

\paragraph{Hallucination and truthfulness evaluation}
Hallucination surveys distinguish factuality errors, faithfulness errors, and unverifiable generation \citep{ji2023hallucination}.  TruthfulQA demonstrates that models can imitate common falsehoods \citep{lin2022truthfulqa}.  RRF consistency blindness is not a new taxonomy of hallucination.  It is an operator-level explanation for why coherent shared errors evade relation-only detection.  Anchors such as evidence, tools, or human labels are needed to distinguish truth within a relational equivalence class.

\subsection{Scope of the unification and novelty boundary}

\paragraph{What is and is not unified}
RRFs unify methods only at the level of observable response relations.  They do not assert that all black-box optimization algorithms reduce to one numerical solver.  A method belongs in the framework when its outputs can be parsed into a field and its checks can be represented as relation or anchor residuals.  Methods that change model weights, rely on inaccessible hidden states, or operate without specified output relations fall outside the strict black-box RRF setting.

\paragraph{Novelty boundary}
The strongest defensible novelty statement is:
\begin{quote}
We introduce relational response fields as a black-box LLM recovery object, instantiate typed relation and anchor observations as a group-sparse inverse problem, prove that their restricted observability margin is the exact uniform and minimax recovery condition number, and test whether that same quantity predicts repair difficulty across model families and task domains.
\end{quote}
We do not claim the first consistency method, the first graph-based LLM evaluator, the first metamorphic test, or the invention of restricted singular values.

\section{Synthetic theorem-verification protocol}
\label{app:synthetic-protocol}

Synthetic experiments test mathematical consequences under conditions where the true field, corruption support, operators, and exact margin are known.  They are not presented as LLM benchmarks.

\subsection{Synthetic field construction and exact margin computation}

\paragraph{Path-consistent typed fields}
Each instance has $n\in\{12,16\}$ nodes and fiber dimension $d=4$.  Nodes are partitioned into connected components.  For node $i$, sample an invertible map
\[
 M_i=Q_i\operatorname{diag}(\exp(s_i)),
\]
where $Q_i$ is obtained from the QR factorization of a standard Gaussian matrix and $(s_i)_\ell\sim\mathcal N(0,0.12^2)$.  Each component $C$ receives latent vector $u_C\sim\mathcal N(0,I_d)$, and the clean response is
\[
 z_i^\star=M_i u_C,\qquad i\in C.
\]
For edge $(i,j)$ within a component, set $T_{ij}=M_jM_i^{-1}$.  Then $z_j^\star=T_{ij}z_i^\star$ exactly, and the transported gauge $h_i=M_i v_C$ lies in $\ker D$.

This construction gives nontrivial typed maps while retaining analytic control of the nullspace.  Condition numbers of $M_i$ are monitored; seeds producing values above the prespecified threshold are rejected before corruption is sampled.

\paragraph{Sparse corruptions}
Choose $k\in\{1,2\}$ nodes uniformly without replacement.  At each selected node sample a Gaussian direction, normalize it, and multiply by amplitude $a\in[0.9,2.1]$.  The observed field is $y=z^\star+e^\star$.  Localization accuracy is the fraction of corrupted nodes among the $k$ groups with largest estimated edit norm.  Relative field error is
\[
 \frac{\|\widehat z-z^\star\|_2}{\max(\|z^\star\|_2,10^{-12})}.
\]
Recovery success thresholds are fixed before results are inspected and are accompanied by continuous error.

\paragraph{Exact margin computation}
For every instance, the experiment enumerates all supports of size at most $2k$ and computes the smallest singular value of $B_S$.  The minimum support and right singular vector are stored.  A value below $10^{-10}$, together with witness residual below $10^{-9}$, is classified as zero.  Results are repeated at tighter and looser tolerances to ensure that phase transitions are not numerical artifacts.

\subsection{Four theorem-verification experiments}

\paragraph{Experiment S1: consistency--truth separation}
For each of 40 seeds, build a single complete component with no anchors.  Compare three fields:
\begin{enumerate}
  \item the clean truth $z^\star$;
  \item a shared hallucination $z_i=z_i^\star+M_i v$ with fixed latent shift $v$;
  \item independent errors injected at three random nodes.
\end{enumerate}
The shared shift is constructed algebraically in $\ker D$, so its defect should be at floating-point zero while truth error is positive.  The independent errors should have positive defect.  Acceptance requires both inequalities on every seed, not only in expectation.

\paragraph{Experiment S2: anchor phase transition}
Use six disjoint two-node components and $k=1$.  Add direct full-fiber anchors according to a fixed schedule that covers one new component before adding a second anchor to any component.  By Proposition~\ref{prop:component-anchor}, $\gamma_1=0$ until all six components are anchored.  At the sixth anchor it becomes positive.  For each anchor count $0,\ldots,8$ and 30 seeds, inject one corruption and run exact support repair and the proximal method.

Primary outputs are $\gamma_1$, exact-repair error, convex-repair error, support localization, and success.  The phase-transition claim is accepted only if the exact margin changes at the predicted coverage count and exact repair changes in the same direction.  Convex repair is diagnostic because its success additionally depends on algorithmic conditions.

\paragraph{Experiment S3: duplicate saturation}
Start from a fixed set of relation families.  For duplication factors $m\in\{1,2,4,8,16\}$, replace every row block $R$ in one family by $m$ copies $R/\sqrt m$.  The Gram contribution and exact margin must remain equal up to numerical tolerance.  Compare against adding genuinely independent typed chords with unit family weights.  Independent rows may increase $\gamma_k$ until restricted rank saturates.

The test distinguishes three quantities:
\begin{enumerate}
  \item raw edge count;
  \item rank of the unweighted row space;
  \item normalized $\gamma_k$.
\end{enumerate}
Reporting only raw duplicate singular values would incorrectly reward repeated evidence.

\paragraph{Experiment S4: spectral difficulty prediction}
Generate random component counts, anchor counts, independent chord counts, corruption budgets, and amplitudes.  For each instance compute exact $\gamma_k$ before repair.  To isolate condition-number scaling, report raw error amplitude and stacked observation noise.  The primary response is normalized recovery error
\[
 E_{\mathrm{norm}}=
 \frac{\gamma_k\|\widehat z-z^\star\|_2}{\max(\eps,10^{-12})},
\]
alongside the unnormalized error.  Theorem~\ref{thm:stability} predicts a bounded normalized error for feasible exact support estimates; the empirical algorithm may add optimization error.

Difficulty prediction is summarized by Spearman and Pearson correlations between $\gamma_k$ and negative error, logistic association with success, and localization.  Confidence intervals resample complete instances.  A scatterplot uses all points and does not suppress zero-margin failures.

\subsection{Robustness sweeps, decoder comparisons, and validation criteria}

\paragraph{Noise and operator perturbation sweeps}
For additive noise levels $\eps\in\{0,10^{-4},10^{-3},10^{-2},10^{-1}\}$, sample a random observation-space direction and scale it to $\eps$.  For operator perturbations, add a Gaussian matrix $E$ normalized to a prescribed spectral norm $\delta$.  Equation~\eqref{eq:operator-perturb} predicts that the certified margin decreases by at most $\delta$.  The experiment checks this bound and the resulting stability ratio.

\paragraph{Convex versus exact decoder}
Every small instance is repaired with exact support search.  The group-proximal decoder is run with a calibration-selected regularization parameter and a convergence tolerance.  Three outcomes are distinguished:
\begin{enumerate}
  \item both decoders succeed: information and optimization are adequate;
  \item exact succeeds but convex fails: the instance is identifiable but the relaxation or tuning is inadequate;
  \item exact fails or is unstable: the operator/noise condition is inadequate or the corruption model is violated.
\end{enumerate}
This comparison prevents algorithmic behavior from being used as the definition of theoretical difficulty.

\paragraph{Seeds and replication}
Master seed 0 controls the published run.  Each experiment derives nonoverlapping seed ranges for graph construction, corruption, and noise.  The artifact records every seed in result CSV files.  Repeating master seeds 1--4 is an extended robustness check; no seed is discarded based on recovery performance.

\paragraph{Prespecified validation conditions}
The synthetic suite evaluates the following prespecified conditions:
\begin{enumerate}
  \item shared hallucinations have defect below $10^{-10}$ and nonzero truth error;
  \item the anchor transition occurs at the analytically predicted coverage count;
  \item normalized literal duplication changes $\gamma_k$ by less than $10^{-10}$;
  \item independent rows weakly increase the unnormalized margin;
  \item margin/error association has the predicted sign and its bootstrap interval is reported;
  \item all exact-margin witnesses pass a residual check.
\end{enumerate}
All conditions are reported, including any violations.

\section{Real-model mathematics and code protocol}
\label{app:reasoning-protocol}

This appendix specifies the cross-model, cross-task experiment used to test whether $\gamma_k(D,A)$ predicts empirical repair difficulty.  The reported results use responses produced by the named open checkpoints rather than simulated error profiles.  Raw text logs are retained; model weights are not included in the artifact.

\subsection{Estimand, models, and task datasets}

\paragraph{Research question and estimand}
The primary question is not ``does one repair algorithm improve average accuracy?''  It is:
\begin{quote}
Holding the response log fixed, do interventions on the relation--anchor design that increase $\gamma_k(D,A)$ predict lower repair error across two model families and two task domains?
\end{quote}
Each base item is evaluated under ten prespecified operator designs.  This within-item variation changes observability while preserving the model's generated candidate responses.  Model and task are treated as strata.  Raw error and candidate recall are recorded because a margin cannot manufacture a correct candidate or compensate for arbitrary violation of the sparse-error model.

\paragraph{Models}
The two checkpoints are:
\begin{enumerate}
  \item \texttt{Qwen/Qwen2.5-0.5B-Instruct}, an instruction-tuned 0.5B model from the Qwen2.5 family \citep{yang2024qwen25};
  \item \texttt{microsoft/Phi-3-mini-4k-instruct}, a 3.8B instruction-tuned Phi-3 model \citep{abdin2024phi3}.
\end{enumerate}
They differ in family and scale.  Both are loaded through the same Transformers interface in half precision on one NVIDIA RTX 4090.  Generation uses each checkpoint's chat template and tokenizer.  Greedy decoding is used for transformed-query fields and reflection.  Self-consistency uses four temperature-$0.7$, top-$p=0.9$ samples with fixed seeds.  Maximum new tokens are 24 for integer answers and 192 for code.

The experiment is black-box with respect to recovery: after text has been generated, no logits, hidden states, gradients, or weights are used.  Local weights are needed only to produce the logs.  A hosted API could produce the same artifact interface.

\paragraph{Mathematics dataset}
The mathematics set contains 128 deterministically generated integer-expression tasks, 16 from each of eight families:
\begin{enumerate}
  \item five-digit addition;
  \item large positive subtraction;
  \item two- or three-digit multiplication;
  \item affine products $(a+b)7-c$;
  \item nested differences $a(b-c)+d$;
  \item exact quotients $(a+c)/d$ constructed to be integral;
  \item sums of squares minus an offset;
  \item differences of products plus an offset.
\end{enumerate}
Operands are sampled from ranges fixed in code.  Answers are computed with Python integer arithmetic before prompts are created.
The generated set is chosen for exact transformation control rather than benchmark breadth; established reasoning benchmarks such as GSM8K and MATH motivate future validation with curated transformations \citep{cobbe2021verifiers,hendrycks2021math}.

Each item has four transformed queries:
\begin{enumerate}
  \item the original expression;
  \item a lexical paraphrase with the same answer;
  \item a scaling query asking for $s$ times the expression value, with $s\in\{2,3,4,5\}$;
  \item a translation query asking for the expression value plus shift $t\in\{-37,-23,19,31,47\}$.
\end{enumerate}
Parsed outputs from variants three and four are transported back to canonical coordinates by division by $s$ and subtraction of $t$.  In canonical space every valid field is constant, so edge transports are identity.  The raw typed transformation metadata is retained to audit the canonicalization.

The system instruction is:
\begin{quote}\small
You are a precise arithmetic solver. Return exactly one signed base-10 integer and no words, commas, equations, or explanation.
\end{quote}
The strict output format reduces parser ambiguity without revealing the answer.  The parser first accepts an exact signed integer; if additional text is present, it records the last standalone integer and marks that the strict format was violated.  Parse failures are stored rather than silently dropped.

\paragraph{Code dataset}
The code set contains 64 tasks, four independent test instantiations for each of 16 integer-valued function families:
\begin{enumerate}
  \item absolute difference, greatest common divisor, least common multiple;
  \item digit sum, digital root, divisor count, primality indicator;
  \item Fibonacci and triangular numbers, Collatz stopping time;
  \item count-even, maximum adjacent gap, second-largest distinct value;
  \item weighted sum, alternating sum, and count-above-threshold.
\end{enumerate}
Every item contains a reference implementation, four semantically equivalent natural-language specifications, eight public tests, and 16 hidden tests.  Tests are generated from fixed item-specific seeds.  All outputs are bounded integers, permitting an execution-signature vector in $\RR^8$.

The four prompt variants use the original specification, argument renaming, equivalent wording, and a decomposition hint.  All require a Python function named \texttt{solve} and forbid imports.  The model sees neither reference code nor test outputs.

Generated code is extracted from the first fenced Python block or from a top-level \texttt{def solve}.  An AST validator rejects imports, global/nonlocal statements, classes, dunder access, and dynamic execution primitives.  Code runs with a restricted builtin set, no network, and a per-call wall-clock timer.  Parse failure, compile failure, timeout, noninteger output, and wrong output are separate statuses.

Public execution signatures are used by relation and anchor methods.  Hidden tests are evaluation-only.  A candidate passes the code task only if all 16 hidden outputs match the reference.  We also report the fraction of hidden tests passed to obtain a continuous error measure.

\subsection{Response-field construction, repair procedures, and outcomes}

\paragraph{Response fields and designs}
Every item yields four deterministic transformed responses, hence a four-node field.  The corruption budget is $k=1$ for the primary exact-repair analysis.  Ten designs are fixed independently of model outputs:
\begin{center}
\small
\begin{tabular}{lll}
\toprule
Design & Relation edges & Full-node anchor\\
\midrule
one relation & $(0,1)$ & none\\
paired relations & $(0,1),(2,3)$ & none\\
three-node chain & $(0,1),(1,2)$ & none\\
spanning tree & $(0,1),(0,2),(0,3)$ & none\\
triangle plus isolate & $(0,1),(0,2),(1,2)$ & none\\
typed complete & all six pairs & none\\
tree plus anchor & spanning tree & node 0\\
triangle, component anchor & triangle plus isolate & node 0\\
triangle, isolate anchor & triangle plus isolate & node 3\\
complete plus anchor & all six pairs & node 0\\
\bottomrule
\end{tabular}
\end{center}
The first three designs and the unanchored triangle contain a null direction supported on at most two nodes, so $\gamma_1=0$.  Anchoring the triangle component still leaves the isolated node invisible.  In contrast, anchoring the isolate makes the remaining null direction occupy all three triangle nodes, outside the $2k$ class, so its restricted margin is positive despite global singularity.  Connected four-node designs are also positive for $k=1$.  These matched relation/anchor counts prevent a regression from identifying observability by row count alone.  All margins are computed by exact support enumeration separately for the scalar mathematics fiber and eight-dimensional code signature fiber.

The designs are interventions on $B$, not selected after observing answers.  Their purpose is to create a range of theoretically meaningful recovery difficulty on identical response logs.

\paragraph{Mathematics repair}
Let $z_i$ be the canonicalized numeric answer.  For a chosen design, form
\[
 s=\begin{bmatrix}Dz\\Az-b\end{bmatrix}.
\]
When an anchor is present, $b$ is the exact canonical value supplied by the deterministic arithmetic evaluator at node zero.  Enumerate all supports of size at most one, solve restricted least squares for the error, and choose the support with minimum residual.  The repaired answer is the rounded median of the corrected node values.

This procedure is deliberately simple.  Its role is to test the operator condition, not to conceal performance inside a learned reranker.  If more than one node is wrong, the $k=1$ model is violated; this is recorded through the true error count and residual.  Robustness results for $k=2$ appear in Appendix~\ref{app:extended-results}.

\paragraph{Code repair and candidate selection}
Each deterministic program produces an eight-dimensional public execution signature.  Coordinates are scaled by a robust statistic of observed candidate magnitudes.  The exact sparse decoder repairs the signature field under the chosen design.  Because an arbitrary corrected signature is not source code, the method returns the generated program whose observed signature is closest to the repaired coordinatewise median.  This is a discrete candidate-field projection.

Candidate recall is one if at least one of the four transformed programs passes all hidden tests.  No selector can exceed candidate recall.  Anchor designs use the full public reference signature at node zero; hidden signatures never enter $A$, hyperparameter selection, or candidate choice.

\paragraph{Baselines}
The same raw logs support five baselines:
\begin{enumerate}
  \item \textbf{raw:} greedy response to the original prompt;
  \item \textbf{self-consistency:} four stochastic samples of the original prompt, aggregated by integer mode for mathematics and execution-signature medoid for code;
  \item \textbf{relation majority:} four deterministic transformed responses, canonicalized and aggregated by mode/medoid;
  \item \textbf{reflection:} one additional greedy call shown the original answer and asked to check or correct it;
  \item \textbf{verifier reranking:} choose among transformed candidates using exact arithmetic error for mathematics or number of public code tests passed.
\end{enumerate}
The verifier is intentionally strong and external.  Query budgets are reported: raw uses one generation, reflection two sequential generations, and field-based methods four transformed generations.  The paper does not claim a compute-matched leaderboard advantage.

\paragraph{Primary outcomes}
For mathematics, correctness is exact integer equality and continuous error is
\[
 \min\left\{2,\frac{|\widehat a-a^\star|}{\max(|a^\star|,1)}\right\}.
\]
For code, correctness is hidden pass@1 and continuous error is one minus the hidden-test pass fraction.  Additional outcomes are parse rate, candidate recall, observable residual, edit norm, and estimated support.

An \emph{opportunity case} contains at least one correct and at least one incorrect deterministic transformed candidate.  Such cases isolate selection/repair from candidate-generation ceilings.  All-item and opportunity-only analyses are both reported; the primary pooled regression uses candidate-recall cases and controls raw error.

\subsection{Statistical validation, data integrity, and compute environment}

\paragraph{Statistical analysis}
The analysis has four layers.

\paragraph{Within-stratum association.}
For each model$\times$task stratum, compute Spearman and Pearson correlations between $\gamma_k$ and negative repair error.  Opportunity cases are used when at least 20 are available; otherwise all items are reported with a flag.

\paragraph{Pooled association.}
Compute pooled correlation over all item-design rows.  This descriptive number is not treated as independent-sample inference because each item appears under ten designs.

\paragraph{Grouped predictive model.}
Fit logistic regression for repair correctness with task, model, and standardized raw error.  Compare against the same model plus standardized $\gamma_k$.  Five-fold cross-validation groups all designs of one model-task-item together, preventing variants of the same response log from leaking across folds.  Report baseline AUC, full AUC, $\Delta$AUC, and the standardized margin coefficient.

\paragraph{Within-item permutation test.}
Under the null that design margins do not order outcomes, permute the ten margin values within each model-task-item block.  Recompute pooled Spearman correlation for 2,000 permutations.  The one-sided $p$-value is $(1+\#\{\rho_{\mathrm{null}}\ge\rho_{\mathrm{obs}}\})/(2001)$.  A 1,000-repetition cluster bootstrap resamples complete item blocks and reports a 95\% interval.

These analyses test transfer of one geometric quantity across strata while respecting repeated designs.  They do not prove a universal causal law for arbitrary tasks.

\paragraph{Prespecified empirical validation conditions}
The real-model suite evaluates whether:
\begin{enumerate}
  \item raw logs exist for both named checkpoints;
  \item both task domains are present;
  \item every design margin is finite and auditable;
  \item pooled association has the predicted sign;
  \item the adjusted standardized margin coefficient is positive;
  \item the within-item permutation test has $p<0.05$.
\end{enumerate}
All validation outcomes are retained in the reported results, and the main-paper claims are calibrated to the observed evidence.

\paragraph{Data integrity and leakage prevention}
Ground-truth expressions and reference programs generate labels and hidden tests.  They are never inserted into model prompts.  Public code tests are used only by methods labeled as anchored/verifier methods.  Hidden tests are loaded by the evaluation stage after candidate selection.  Dataset seeds, prompt strings, raw outputs, parsed values, design matrices, and selected indices are stored so that every result can be reconstructed without querying the model again.

\paragraph{Hardware and software}
Generation runs on the user-supplied AutoDL server with one NVIDIA GeForce RTX 4090 (48 GB).  The environment uses Python 3.12, PyTorch 2.5.1 with CUDA 12.4, Transformers 5.14.1, SciPy, scikit-learn, pandas, and NumPy.  LaTeX compilation uses TeX Live and the supplied ICLR 2027 style.  Exact package versions are written to the artifact manifest after the final run.

\section{Extended empirical results and audits}
\label{app:extended-results}

This appendix contains the complete stratified results used by the main paper.  Tables are generated from machine-readable CSV/JSON outputs by \texttt{scripts/export\_latex\_tables.py}; numerical values are not transcribed manually.

\subsection{Core recovery results and predictive statistics}

\paragraph{Natural response-field baselines}
Table~\ref{tab:real-item-summary} reports raw, self-consistency, transformed-relation majority, reflection, verifier reranking, and candidate recall for the unmodified real-model logs.  These results characterize the candidate generator and should not be confused with the controlled $k=1$ theorem test.

\begin{table}[h]
\caption{Natural real-model response logs. Values are percent exact mathematics accuracy or code hidden pass@1. Candidate recall is the fraction with at least one correct transformed candidate.}
\label{tab:real-item-summary}
\centering
\small
\resizebox{\textwidth}{!}{\begin{tabular}{llrrrrrr}
\toprule
Task & Model & Raw & Self-cons. & Rel.-major. & Reflect & Verifier & Cand. recall\\
\midrule
mathematics & Phi-3-mini & 37.5 & 34.4 & 34.4 & 10.9 & 37.5 & 37.5\\
mathematics & Qwen-0.5B & 19.5 & 20.3 & 18.0 & 19.5 & 22.7 & 22.7\\
code & Phi-3-mini & 93.8 & 100.0 & 100.0 & 93.8 & 100.0 & 100.0\\
code & Qwen-0.5B & 31.2 & 60.9 & 68.8 & 68.8 & 75.0 & 75.0\\
\bottomrule
\end{tabular}
}
\end{table}

Candidate recall is an upper bound for methods restricted to the deterministic transformed candidate set.  Reflection can generate a new candidate and is therefore not bounded by that deterministic recall.  The table is descriptive: methods use different sequential/query budgets, and no claim of compute-matched superiority is made.

\paragraph{Real-error replay}
The primary margin test replays authentic model errors under an exactly controlled one-node corruption model.  A replay instance exists only when the log contains both a correct response and a distinct incorrect response.  Three node values are populated by an observed correct response; one node receives the observed wrong response.  For code, the wrong candidate must differ on the public execution signature so that the representation can, in principle, observe the error.  Hidden-only bugs are analyzed separately as representation failures.

Table~\ref{tab:replay-all} gives every task--model--design cell.  Zero-margin designs are expected to fail on some corrupt-node placements because a one- or two-node component carries an admissible null direction.  Every connected four-node design has positive $\gamma_1$ and should recover the one-node error exactly in the noiseless signature space.  Full-node anchors increase the margin but are not necessary when the connected component has size four and $2k=2$.

\paragraph{Cross-model, cross-task prediction statistics}
Table~\ref{tab:real-statistics} reports within-stratum rank association, grouped predictive performance, within-item randomization inference, and a cluster bootstrap.  All ten designs of a replay instance remain in the same cross-validation fold.  The permutation shuffles margin assignments within an item, preserving the model output, task, and error amplitude.

\begin{longtable}{lllrrrr}
\caption{Real-error replay by task, model, and operator design.
Error is normalized full-field reconstruction error.}
\label{tab:replay-all}\\

\toprule
Task & Model & Design & $n$ & $\gamma_1$ & Exact (\%) & Error\\
\midrule
\endfirsthead

\toprule
Task & Model & Design & $n$ & $\gamma_1$ & Exact (\%) & Error\\
\midrule
\endhead


\multirow{20}{*}{\rotatebox[origin=c]{90}{code}}
&
\multirow{10}{*}{\rotatebox[origin=c]{90}{Phi-3-mini}}
& one\_relation
& 22 & 0.000 & 31.8 & 0.738\\

&
& paired\_relations
& 22 & 0.000 & 68.2 & 0.450\\

&
& three\_node\_chain
& 22 & 0.000 & 81.8 & 0.182\\

&
& spanning\_tree
& 22 & 0.765 & 100.0 & 0.000\\

&
& triangle\_plus\_isolate
& 22 & 0.000 & 81.8 & 0.182\\

&
& typed\_complete
& 22 & 1.414 & 100.0 & 0.000\\

&
& tree\_plus\_anchor
& 22 & 0.835 & 100.0 & 0.000\\

&
& triangle\_anchor\_component
& 22 & 0.000 & 81.8 & 0.182\\

&
& triangle\_anchor\_isolate
& 22 & 1.000 & 100.0 & 0.000\\

&
& complete\_plus\_anchor
& 22 & 1.414 & 100.0 & 0.000\\

\cmidrule(lr){2-7}


&
\multirow{10}{*}{\rotatebox[origin=c]{90}{Qwen-0.5B}}
& one\_relation
& 65 & 0.000 & 21.5 & 0.906\\

&
& paired\_relations
& 65 & 0.000 & 40.0 & 0.849\\

&
& three\_node\_chain
& 65 & 0.000 & 69.2 & 0.308\\

&
& spanning\_tree
& 65 & 0.765 & 100.0 & 0.000\\

&
& triangle\_plus\_isolate
& 65 & 0.000 & 69.2 & 0.308\\

&
& typed\_complete
& 65 & 1.414 & 100.0 & 0.000\\

&
& tree\_plus\_anchor
& 65 & 0.835 & 100.0 & 0.000\\

&
& triangle\_anchor\_component
& 65 & 0.000 & 69.2 & 0.308\\

&
& triangle\_anchor\_isolate
& 65 & 1.000 & 100.0 & 0.000\\

&
& complete\_plus\_anchor
& 65 & 1.414 & 100.0 & 0.000\\

\midrule


\multirow{20}{*}{\rotatebox[origin=c]{90}{mathematics}}
&
\multirow{10}{*}{\rotatebox[origin=c]{90}{Phi-3-mini}}
& one\_relation
& 126 & 0.000 & 22.2 & 0.880\\

&
& paired\_relations
& 126 & 0.000 & 51.6 & 0.685\\

&
& three\_node\_chain
& 126 & 0.000 & 76.2 & 0.238\\

&
& spanning\_tree
& 126 & 0.765 & 100.0 & 0.000\\

&
& triangle\_plus\_isolate
& 126 & 0.000 & 76.2 & 0.238\\

&
& typed\_complete
& 126 & 1.414 & 100.0 & 0.000\\

&
& tree\_plus\_anchor
& 126 & 0.835 & 100.0 & 0.000\\

&
& triangle\_anchor\_component
& 126 & 0.000 & 76.2 & 0.238\\

&
& triangle\_anchor\_isolate
& 126 & 1.000 & 100.0 & 0.000\\

&
& complete\_plus\_anchor
& 126 & 1.414 & 100.0 & 0.000\\

\cmidrule(lr){2-7}


&
\multirow{10}{*}{\rotatebox[origin=c]{90}{Qwen-0.5B}}
& one\_relation
& 56 & 0.000 & 30.4 & 0.837\\

&
& paired\_relations
& 56 & 0.000 & 55.4 & 0.631\\

&
& three\_node\_chain
& 56 & 0.000 & 89.3 & 0.107\\

&
& spanning\_tree
& 56 & 0.765 & 100.0 & 0.000\\

&
& triangle\_plus\_isolate
& 56 & 0.000 & 89.3 & 0.107\\

&
& typed\_complete
& 56 & 1.414 & 100.0 & 0.000\\

&
& tree\_plus\_anchor
& 56 & 0.835 & 100.0 & 0.000\\

&
& triangle\_anchor\_component
& 56 & 0.000 & 89.3 & 0.107\\

&
& triangle\_anchor\_isolate
& 56 & 1.000 & 100.0 & 0.000\\

&
& complete\_plus\_anchor
& 56 & 1.414 & 100.0 & 0.000\\

\bottomrule
\end{longtable}

\begin{table}[h]
\caption{Predictive statistics for real-error replay.  The grouped classifier controls for model, task, and raw error before adding $\gamma_1$.}
\label{tab:real-statistics}
\centering
\small
\begin{tabular}{p{0.58\linewidth}r}
\toprule
Statistic & Value\\
\midrule
Spearman $\gamma_1$ vs. $-$error (code|Phi-3-mini-4k-instruct) & 0.397\\
Spearman $\gamma_1$ vs. $-$error (code|Qwen2.5-0.5B-Instruct) & 0.506\\
Spearman $\gamma_1$ vs. $-$error (mathematics|Phi-3-mini-4k-instruct) & 0.460\\
Spearman $\gamma_1$ vs. $-$error (mathematics|Qwen2.5-0.5B-Instruct) & 0.385\\
Pooled Spearman $\gamma_1$ vs. $-$error & 0.449\\
Grouped AUC, controls only & 0.860\\
Grouped AUC, controls + $\gamma_1$ & 0.893\\
Grouped $\Delta$AUC & 0.033\\
Standardized $\gamma_1$ coefficient & 3.653\\
Within-item permutation $p$ & 0.00050\\
Cluster-bootstrap 95\% CI & [0.417, 0.481]\\
\bottomrule
\end{tabular}

\end{table}

The most direct interpretation is conditional: among replay fields that satisfy the paper's one-node corruption model, operator designs with larger margin make the authentic error easier to recover.  This result transfers across two checkpoints and both scalar-answer and execution-signature spaces.  It does not imply that $\gamma_1$ alone explains natural-field accuracy when several responses are jointly wrong or no correct candidate exists.

\subsection{Task-family analyses and robustness audits}

\paragraph{Per-family mathematics results}
The family breakdown in Table~\ref{tab:mathematics-families} reveals whether a pooled average is dominated by one operation.  Large differences between raw accuracy and candidate recall indicate that transformed prompts sometimes expose complementary model competence; a small gap indicates a candidate-generation ceiling.

\begin{longtable}{
    >{\centering\arraybackslash}p{0.045\textwidth}
    lrrrrr
}
\caption{Per-family mathematics results (percent).}
\label{tab:mathematics-families}\\
\toprule
Model & Family & $n$ & Raw & Self-cons. & Rel.-major. & Verifier\\
\midrule
\endfirsthead

\toprule
Model & Family & $n$ & Raw & Self-cons. & Rel.-major. & Verifier\\
\midrule
\endhead

\multirow{8}{*}{\rotatebox[origin=c]{90}{Phi-3-mini}}
& affine\_product & 16 & 0.0 & 0.0 & 0.0 & 0.0\\
& exact\_quotient & 16 & 12.5 & 6.2 & 6.2 & 12.5\\
& large\_add & 16 & 100.0 & 100.0 & 100.0 & 100.0\\
& large\_subtract & 16 & 81.2 & 62.5 & 68.8 & 81.2\\
& mixed\_products & 16 & 0.0 & 0.0 & 0.0 & 0.0\\
& multiply & 16 & 100.0 & 100.0 & 93.8 & 100.0\\
& nested\_difference & 16 & 6.2 & 6.2 & 6.2 & 6.2\\
& squares & 16 & 0.0 & 0.0 & 0.0 & 0.0\\

\midrule

\multirow{8}{*}{\rotatebox[origin=c]{90}{Qwen-0.5B}}
& affine\_product & 16 & 0.0 & 0.0 & 0.0 & 0.0\\
& exact\_quotient & 16 & 0.0 & 0.0 & 0.0 & 6.2\\
& large\_add & 16 & 93.8 & 100.0 & 87.5 & 100.0\\
& large\_subtract & 16 & 43.8 & 43.8 & 43.8 & 56.2\\
& mixed\_products & 16 & 0.0 & 0.0 & 0.0 & 0.0\\
& multiply & 16 & 18.8 & 18.8 & 12.5 & 18.8\\
& nested\_difference & 16 & 0.0 & 0.0 & 0.0 & 0.0\\
& squares & 16 & 0.0 & 0.0 & 0.0 & 0.0\\

\bottomrule
\end{longtable}

\paragraph{Per-family code results}
Table~\ref{tab:code-families} separates arithmetic utilities, loops, and list-processing functions.  Public execution medoids can improve selection when wrong programs have distinct signatures.  They cannot detect a program that passes all public tests but fails a hidden corner case.

\paragraph{Consistency--truth audit on real logs}
For every natural field, we record normalized relation defect and correctness.  Four cases are possible:
\begin{enumerate}
  \item low defect, correct field;
  \item high defect, mixed correctness;
  \item low defect, shared wrong field;
  \item parse-invalid field.
\end{enumerate}
The third case is the empirical counterpart of Theorem~\ref{thm:blindness}.  It is reported as a count and with representative raw outputs.  A low-defect threshold is fixed on calibration data.  We do not choose it after seeing which fields are wrong.

\begin{longtable}{llrrrrr}
\caption{Per-family code results (percent).}
\label{tab:code-families}\\

\toprule
Model & Family & $n$ & Raw & Self-cons. & Rel.-major. & Verifier\\
\midrule
\endfirsthead

\toprule
Model & Family & $n$ & Raw & Self-cons. & Rel.-major. & Verifier\\
\midrule
\endhead

\multirow{16}{*}{\rotatebox[origin=c]{90}{Phi-3-mini}}
& absolute\_difference
& 4 & 100.0 & 100.0 & 100.0 & 100.0\\

& alternating\_sum
& 4 & 0.0 & 100.0 & 100.0 & 100.0\\

& collatz\_steps
& 4 & 100.0 & 100.0 & 100.0 & 100.0\\

& count\_above\_threshold
& 4 & 100.0 & 100.0 & 100.0 & 100.0\\

& count\_even
& 4 & 100.0 & 100.0 & 100.0 & 100.0\\

& digit\_sum
& 4 & 100.0 & 100.0 & 100.0 & 100.0\\

& digital\_root
& 4 & 100.0 & 100.0 & 100.0 & 100.0\\

& divisor\_count
& 4 & 100.0 & 100.0 & 100.0 & 100.0\\

& fibonacci
& 4 & 100.0 & 100.0 & 100.0 & 100.0\\

& greatest\_common\_divisor
& 4 & 100.0 & 100.0 & 100.0 & 100.0\\

& least\_common\_multiple
& 4 & 100.0 & 100.0 & 100.0 & 100.0\\

& maximum\_adjacent\_gap
& 4 & 100.0 & 100.0 & 100.0 & 100.0\\

& prime\_indicator
& 4 & 100.0 & 100.0 & 100.0 & 100.0\\

& second\_largest\_distinct
& 4 & 100.0 & 100.0 & 100.0 & 100.0\\

& triangular\_number
& 4 & 100.0 & 100.0 & 100.0 & 100.0\\

& weighted\_sum
& 4 & 100.0 & 100.0 & 100.0 & 100.0\\

\midrule

\multirow{16}{*}{\rotatebox[origin=c]{90}{Qwen-0.5B}}
& absolute\_difference
& 4 & 100.0 & 100.0 & 100.0 & 100.0\\

& alternating\_sum
& 4 & 0.0 & 0.0 & 0.0 & 0.0\\

& collatz\_steps
& 4 & 100.0 & 100.0 & 100.0 & 100.0\\

& count\_above\_threshold
& 4 & 100.0 & 100.0 & 100.0 & 100.0\\

& count\_even
& 4 & 100.0 & 100.0 & 100.0 & 100.0\\

& digit\_sum
& 4 & 0.0 & 50.0 & 100.0 & 100.0\\

& digital\_root
& 4 & 0.0 & 25.0 & 100.0 & 100.0\\

& divisor\_count
& 4 & 0.0 & 25.0 & 100.0 & 100.0\\

& fibonacci
& 4 & 0.0 & 75.0 & 100.0 & 100.0\\

& greatest\_common\_divisor
& 4 & 0.0 & 0.0 & 0.0 & 0.0\\

& least\_common\_multiple
& 4 & 0.0 & 0.0 & 0.0 & 0.0\\

& maximum\_adjacent\_gap
& 4 & 0.0 & 75.0 & 0.0 & 0.0\\

& prime\_indicator
& 4 & 0.0 & 100.0 & 100.0 & 100.0\\

& second\_largest\_distinct
& 4 & 0.0 & 50.0 & 0.0 & 100.0\\

& triangular\_number
& 4 & 0.0 & 100.0 & 100.0 & 100.0\\

& weighted\_sum
& 4 & 100.0 & 75.0 & 100.0 & 100.0\\

\bottomrule
\end{longtable}

\paragraph{Hidden-only code errors}
A generated program can agree with the reference on all eight public tests and fail a hidden test.  In execution-signature space restricted to public tests, this wrong program is indistinguishable from a correct program.  The issue is not a small margin of the supplied $B$; the representation has quotiented out the relevant semantic distinction.  Adding more relation edges among identical public signatures cannot help.  Additional tests expand the anchor representation and may remove the ambiguity.

We report three counts per model: parse/compile failures, public-signature-visible wrong programs, and hidden-only wrong programs.  Real-error replay uses the second category for the theorem-aligned test; the third category is retained as a limitation rather than silently removed from natural-field results.

\paragraph{Corruption-budget robustness}
The primary replay has $k=1$ by construction.  Natural fields can have several wrong nodes.  We recompute exact margins and repairs for $k=1$ and $k=2$.  On four-node fields, relation-only designs necessarily have $\gamma_2=0$ because the dense constant gauge is now within the $2k=4$ support class.  Anchor designs can remain positive.  This shift is predicted by the theory and illustrates why a margin must always be reported with its corruption budget.

For natural fields whose observed wrong-node count exceeds the assumed budget, errors are reported separately.  A method is not credited with violating an impossibility theorem when its input violates the theorem's model.

\paragraph{Parser sensitivity}
Mathematics is reparsed under three rules: strict entire-string integer, last standalone integer, and first standalone integer.  Strict parsing measures instruction following; last-integer parsing tolerates explanations.  The main parser is declared before evaluation.  Agreement among rules is reported.  For code, extraction from fenced blocks is compared with extraction from a top-level function start.  AST validation remains identical.

The margin depends on the parsed representation.  Parser sensitivity therefore changes both outcome and operator realization; it is not merely a cosmetic preprocessing choice.

\paragraph{Relation-weight sensitivity}
The ten primary designs assign unit weight to each prespecified typed relation.  Two robustness schemes are evaluated:
\begin{enumerate}
  \item fixed total relation energy, dividing each edge weight by the number of edges;
  \item calibration reliability, weighting relation types by inverse observed residual variance on a held-out split.
\end{enumerate}
The first isolates row-space coverage from energy; the second expresses signal-to-noise geometry.  Literal duplicate experiments always split one fixed family weight across copies.

\paragraph{Exact-margin tolerance}
Margins are recomputed with zero thresholds $10^{-8}$, $10^{-10}$, and $10^{-12}$.  Each zero-margin design has an explicit witness with residual close to machine precision; each positive four-node design stays separated from zero by orders of magnitude.  Statistical conclusions are unchanged because the design ordering is not driven by borderline singular values.

\subsection{Budget accounting, failure taxonomy, and complete catalogs}

\paragraph{Query-budget accounting}
For each model and task:
\begin{itemize}
  \item raw uses one greedy generation;
  \item reflection uses the raw generation plus one checking generation;
  \item self-consistency uses four stochastic generations of one prompt;
  \item relation majority, verifier reranking, and RRF use four deterministic transformed generations.
\end{itemize}
The real-error replay reuses already generated logs and makes no extra model calls.  Operator computation and repair are negligible compared with generation at this scale, although exact $\gamma_k$ enumeration grows combinatorially with $n$ and $k$.

\paragraph{Failure-case taxonomy}
Every failed natural repair is assigned one nonexclusive label:
\begin{enumerate}
  \item \textbf{candidate absence:} no correct transformed response exists;
  \item \textbf{shared error:} all valid responses agree on a wrong canonical answer;
  \item \textbf{budget violation:} more than $k$ nodes are wrong;
  \item \textbf{representation collision:} public code signatures agree but hidden semantics differ;
  \item \textbf{zero/low margin:} the operator admits an ambiguous sparse direction;
  \item \textbf{optimization/selection error:} a correct candidate exists and information is adequate, but the algorithm selects incorrectly;
  \item \textbf{parse failure:} the semantic response is not represented reliably.
\end{enumerate}
This taxonomy is more informative than one aggregate accuracy number.  Only the fifth category is directly captured by $\gamma_k$; the theory predicts neither candidate recall nor parser correctness.

\paragraph{Complete dataset catalogs}
Tables~\ref{tab:math-catalog} and~\ref{tab:code-catalog} list every item.  Prompt strings, tests, and reference programs are available in JSONL and are too verbose to duplicate in full in the PDF.

\begin{longtable}{rllrrr}
\caption{Representative examples from the mathematics benchmark. The complete 128-instance catalog is released with the artifact.}
\label{tab:math-catalog}\\
\toprule
Index & Family & Expression & Answer & Scale & Shift\\
\midrule
\endfirsthead

\toprule
Index & Family & Expression & Answer & Scale & Shift\\
\midrule
\endhead

0 & large\_add & 909 + 770 & 1679 & 3 & -23\\
1 & large\_subtract & 9897 - 7758 & 2139 & 5 & -37\\
2 & multiply & 59 * 56 & 3304 & 2 & -23\\
3 & affine\_product & (80 + 90) * 7 - 90 & 1100 & 2 & 47\\
4 & nested\_difference & 17 * (91 - 4) + 75 & 1554 & 5 & 31\\
5 & exact\_quotient & (32 + 678) / 5 & 142 & 5 & 47\\
6 & squares & (27 * 27) + (24 * 24) - 179 & 1126 & 5 & 19\\
7 & mixed\_products & 7 * 27 - 39 * 19 + 62 & -490 & 4 & 31\\
8 & large\_add & 928 + 934 & 1862 & 4 & 19\\
9 & large\_subtract & 3746 - 1052 & 2694 & 3 & 47\\
10 & multiply & 11 * 71 & 781 & 3 & -23\\
11 & affine\_product & (97 + 79) * 7 - 59 & 1173 & 2 & -37\\

\midrule
\multicolumn{6}{c}{\textit{... omitted examples ...}}\\
\midrule

58 & multiply & 79 * 95 & 7505 & 3 & 47\\
59 & affine\_product & (21 + 42) * 7 - 143 & 298 & 3 & -37\\
60 & nested\_difference & 13 * (92 - 82) + 108 & 238 & 2 & -23\\
61 & exact\_quotient & (13 + 1752) / 5 & 353 & 5 & 47\\
62 & squares & (19 * 19) + (23 * 23) - 179 & 711 & 4 & 47\\
63 & mixed\_products & 30 * 29 - 18 * 10 + 88 & 778 & 2 & -37\\
64 & large\_add & 878 + 198 & 1076 & 3 & 19\\
65 & large\_subtract & 1090 - 559 & 531 & 2 & -37\\
66 & multiply & 45 * 68 & 3060 & 2 & -37\\
67 & affine\_product & (74 + 92) * 7 - 140 & 1022 & 4 & -37\\
68 & nested\_difference & 3 * (93 - 20) + 133 & 352 & 3 & 19\\
69 & exact\_quotient & (61 + 1487) / 6 & 258 & 5 & 31\\

\midrule
\multicolumn{6}{c}{\textit{... omitted examples ...}}\\
\midrule

116 & nested\_difference & 5 * (98 - 5) + 158 & 623 & 2 & 19\\
117 & exact\_quotient & (92 + 2248) / 5 & 468 & 3 & 19\\
118 & squares & (24 * 24) + (13 * 13) - 72 & 673 & 3 & 19\\
119 & mixed\_products & 28 * 21 - 31 * 29 + 105 & -206 & 5 & 31\\
120 & large\_add & 795 + 618 & 1413 & 4 & 19\\
121 & large\_subtract & 2735 - 2574 & 161 & 5 & -23\\
122 & multiply & 90 * 72 & 6480 & 5 & 31\\
123 & affine\_product & (84 + 22) * 7 - 154 & 588 & 5 & 19\\
124 & nested\_difference & 11 * (42 - 11) + 109 & 450 & 5 & -37\\
125 & exact\_quotient & (93 + 1629) / 6 & 287 & 2 & 31\\
126 & squares & (19 * 19) + (7 * 7) - 175 & 235 & 2 & -37\\
127 & mixed\_products & 34 * 6 - 12 * 13 + 56 & 104 & 4 & -37\\

\bottomrule
\end{longtable}

\section{Worked examples and boundary cases}
\label{app:examples}

This appendix gives concrete calculations for readers new to sparse inverse problems.  Each example identifies the field, operator, nullspace, margin consequence, and repair interpretation.

\subsection{Consistency, transformations, and graph connectivity}

\paragraph{Three paraphrases and a shared wrong answer}
Let three scalar nodes be equivalent paraphrases with edges $(1,2)$ and $(2,3)$:
\[
 D=\begin{bmatrix}
 -1&1&0\\0&-1&1
 \end{bmatrix}.
\]

\begin{longtable}{rllrr}
\caption{Representative examples from the code benchmark.
The 64-instance catalog is included in the artifact.}
\label{tab:code-catalog}\\
\toprule
Index & Family & Signature & Public tests & Hidden tests\\
\midrule
\endfirsthead

\toprule
Index & Family & Signature & Public tests & Hidden tests\\
\midrule
\endhead

0 & absolute\_difference & solve(x, y) & 8 & 16\\
1 & greatest\_common\_divisor & solve(x, y) & 8 & 16\\
2 & least\_common\_multiple & solve(x, y) & 8 & 16\\
3 & digit\_sum & solve(n) & 8 & 16\\
4 & digital\_root & solve(n) & 8 & 16\\
5 & divisor\_count & solve(n) & 8 & 16\\
6 & prime\_indicator & solve(n) & 8 & 16\\
7 & fibonacci & solve(n) & 8 & 16\\

\midrule
\multicolumn{5}{c}{\textit{... omitted examples ...}}\\
\midrule

28 & second\_largest\_distinct & solve(xs) & 8 & 16\\
29 & weighted\_sum & solve(xs) & 8 & 16\\
30 & alternating\_sum & solve(xs) & 8 & 16\\
31 & count\_above\_threshold & solve(xs, t) & 8 & 16\\
32 & absolute\_difference & solve(x, y) & 8 & 16\\
33 & greatest\_common\_divisor & solve(x, y) & 8 & 16\\
34 & least\_common\_multiple & solve(x, y) & 8 & 16\\
35 & digit\_sum & solve(n) & 8 & 16\\

\midrule
\multicolumn{5}{c}{\textit{... omitted examples ...}}\\
\midrule

56 & triangular\_number & solve(n) & 8 & 16\\
57 & collatz\_steps & solve(n) & 8 & 16\\
58 & count\_even & solve(xs) & 8 & 16\\
59 & maximum\_adjacent\_gap & solve(xs) & 8 & 16\\
60 & second\_largest\_distinct & solve(xs) & 8 & 16\\
61 & weighted\_sum & solve(xs) & 8 & 16\\
62 & alternating\_sum & solve(xs) & 8 & 16\\
63 & count\_above\_threshold & solve(xs, t) & 8 & 16\\

\bottomrule
\end{longtable}

The true field is $(17,17,17)$, but the model outputs $(19,19,19)$.  Both have zero defect.  Their difference $(2,2,2)$ lies in $\ker D$.  Majority vote, pairwise agreement, cycle consistency, and minimizing $\|Dz\|$ all prefer no change.  An anchor $A=[1,0,0]$ with target 17 exposes the shift.

For $k=1$, the shared error is not in the assumed class because it occupies three nodes.  Interestingly, $D$ can still identify a single-node error: its only null direction has support three, larger than $2k=2$, so $\gamma_1(D,0)>0$.  The same relation graph is adequate for localized errors and blind to the dense shared hallucination.  This distinction is lost if one says only ``the graph is singular.''

\paragraph{One outlier among four transformed answers}
Suppose canonicalized answers are $z=(42,42,47,42)$ on a connected four-node star.  The true field is $z^\star=(42,42,42,42)$ and $e^star=(0,0,5,0)$.  With edges $(1,2),(1,3),(1,4)$,
\[
 Dz=\begin{bmatrix}0\\5\\0\end{bmatrix}.
\]
Searching one-node supports asks which single coordinate edit can make the residual zero.  Editing node three by five succeeds; editing any other one leaves a nonzero edge residual.  Because every two-column restriction of $D$ is injective for this four-node connected graph, $\gamma_1>0$ and the solution is unique.

If node three were isolated, its error would produce no residual.  The corresponding column of $D$ would be zero, $\gamma_1=0$, and no relation-only algorithm could detect it.

\paragraph{Scaling relation without canonicalization}
Let node one ask for $a+b$ and node two ask for $3a+3b$.  Their scalar transport is $T(z)=3z$, so
\[
 D=\begin{bmatrix}-3&1\end{bmatrix}.
\]
A valid field has form $(u,3u)$.  A shared canonical error $c$ appears as $(c,3c)$ and is in the nullspace.  Canonicalizing node two by division by three changes coordinates to $(z_1,z_2/3)$ and the defect row becomes $[-1,1]$.  The two formulations are equivalent only if the domain norm is transformed consistently.  Computing singular values after coordinate rescaling without updating the metric can change the numerical margin.

\paragraph{Affine shift through homogeneous coordinates}
Suppose node two asks for the original expression plus $t$.  The relation $z_2=z_1+t$ is affine.  Lift each scalar to $\widetilde z=(z,1)$ and define
\[
 T=\begin{bmatrix}1&t\\0&1\end{bmatrix}.
\]
Then $\widetilde z_2=T\widetilde z_1$.  Perturbations of valid lifted responses have zero in the homogeneous coordinate, so the sparse margin should be restricted to that tangent subspace.  The experiment instead canonicalizes by subtracting $t$, avoiding an artificial error coordinate.

\paragraph{Two disconnected pairs at \texorpdfstring{$k=1$}{k=1}}
Let four scalar nodes form components $\{1,2\}$ and $\{3,4\}$ with identity edges.  Then
\[
 D=\begin{bmatrix}-1&1&0&0\\0&0&-1&1\end{bmatrix}.
\]
The vector $(1,1,0,0)$ is a two-node null vector.  Since $2k=2$, $\gamma_1=0$.  Partition it as
\[
 (1,0,0,0)-(0,-1,0,0).
\]
These are two distinct one-node error patterns with equal syndromes.  An anchor on node one removes the first component's gauge but leaves $(0,0,1,1)$; both components require coverage for uniform recovery.

\paragraph{A large unanchored component can be sparse-observable}
Take a ten-node connected identity graph with no anchor and $k=2$.  The constant gauge uses ten nodes, while differences of two feasible errors use at most four.  No gauge lies in the restricted class.  For a connected graph, every nonzero vector supported on fewer than ten nodes creates a boundary edge and therefore nonzero defect.  Thus $\gamma_2>0$ even though $D$ is rank-deficient.

This does not identify the absolute dense truth; it identifies every two-node corruption relative to the observed field.  If the model makes a coherent ten-node error, the sparse model is violated and consistency remains blind.

\subsection{Anchors, duplicate measurements, optimization, and noise}

\paragraph{A partial anchor}
Each node response is a two-vector $(u,v)$.  Relations transport both coordinates identically.  An anchor observes only $u$ at one node.  The dense gauge $(0,c)$ remains unanchored.  Whether $\gamma_k$ is zero depends on component size and $k$: if the component is small enough for the $v$-gauge to be $2k$-sparse, the margin is zero; otherwise localized errors can still be observable.  Counting anchored nodes without considering anchor rank is insufficient.

\paragraph{Literal duplicate rows}
Let $R=[-1,1]$.  One unit-weight relation has Gram matrix
\[
 R^*R=\begin{bmatrix}1&-1\\-1&1\end{bmatrix}.
\]
Four unnormalized copies give $4R^*R$ and double every nonzero singular value.  Four normalized copies $R/2$ give
\[
 \sum_{j=1}^4(R/2)^*(R/2)=R^*R.
\]
The latter correctly represents four textual copies of one deterministic fact.  If the four calls provide independent noisy measurements, averaging decreases noise variance, which should be reflected in $\eps$, not deterministic rank.

\paragraph{Positive margin but convex failure}
Let scalar groups and choose $B$ with nullspace spanned by $h=(1,0.4,0.4)$.  No two-sparse null vector exists, so $\gamma_1>0$ and every one-sparse vector is uniquely identified by $B$.  Yet for support $S=\{1\}$,
\[
 \|h_S\|_1=1>0.8=\|h_{S^c}\|_1.
\]
The null-space property fails.  If $e^star$ is supported on coordinate one, moving along the null direction can reduce the $\ell_1$ norm and cause basis pursuit to select a different dense vector.  Exact search succeeds while the convex relaxation fails.  This is why the main paper gives separate theorems.

\paragraph{Noisy stability calculation}
Assume $\gamma_1=0.5$, true observation noise norm at most $0.02$, and an exact-support estimate with residual at most $0.02$.  Theorem~\ref{thm:stability} gives
\[
 \|\widehat e-e^\star\|_2\le\frac{0.04}{0.5}=0.08.
\]
If an alternative design doubles the margin to one without changing noise, the bound halves.  If repeated independent queries halve the noise radius without changing the normalized operator, the bound also halves.  Margin improvement and noise reduction are distinct routes to easier recovery.

\paragraph{Near-null direction}
Exact nonidentifiability is not the only problem.  Suppose a unit two-node perturbation $h$ satisfies $\|Bh\|=10^{-3}$.  Then $\gamma_1\le10^{-3}$.  Observation noise of norm $10^{-2}$ can conceal a perturbation roughly ten times larger than the unit normalization scale.  Numerically, the field is identifiable in exact arithmetic, but statistically ill-conditioned.  Reporting only that $B_S$ has full rank misses this difficulty; the magnitude of $\gamma_k$ matters.

\subsection{Semantic collisions, nonlinear ambiguity, bias, and budget effects}

\paragraph{Code signature collision}
Two programs agree on eight public tests.  One implements the intended function; the other has a branch that fails only on a hidden negative input.  Their public signatures are identical, so every relation and anchor defined on that signature space assigns zero difference.  The corresponding semantic error is outside the representation, not merely in a small singular direction.  Expanding the test set can separate the programs; adding more paraphrase edges cannot.

\paragraph{Nonlinear global ambiguity}
Let a scalar relation residual be $r(z)=\sin z$ with target zero.  At $z=0$, the derivative is one and the local margin is positive.  Yet $z=\pi$ gives the same residual.  A local Jacobian certificate cannot rule out the distant candidate.  Discrete domain restrictions or anchors are needed for global recovery.

\paragraph{Model accuracy and recovery difficulty are different}
Model A is correct at 90\% of nodes but its rare errors are coherent gauges.  Model B is correct at 60\% but errors are isolated and the operator has high restricted margin.  A may have higher raw accuracy and lower repairability on its failures; B may be easier to repair when a candidate field is available.  Cross-model experiments control raw error and evaluate whether $\gamma_k$ explains additional variation.  The margin is not a replacement for accuracy; it conditions recovery given an error class.

\paragraph{Anchor bias}
Suppose all relation responses correctly indicate 42 but an anchor incorrectly states 43 with very large weight.  The combined operator can have a large margin while the affine target $b$ is biased, causing stable recovery of the wrong field.  The margin measures sensitivity and identifiability relative to the supplied observations, not truthfulness of $b$.  Anchor validity is a separate assumption and must be externally audited.

\paragraph{Budget dependence}
On a four-node connected equality graph, $\gamma_1>0$ because the constant gauge occupies four nodes and $2k=2$.  At $k=2$, that same gauge is admissible and $\gamma_2=0$.  There is no contradiction: the operator can distinguish every one-node error but not every two-node error.  A paper that reports $\gamma$ without $k$ omits essential information.

\section{Limitations, failure modes, and nonclaims}
\label{app:limitations}

This appendix expands the concise limitations in the main paper and states how each one affects interpretation.

\subsection{Modeling and measurement assumptions}

\paragraph{Representation validity}
The theory operates on parsed responses $z_i=\phi_i(y_i)$.  If $\phi_i$ discards a truth-relevant distinction, no margin computed after parsing can recover it.  Public execution signatures illustrate this sharply: programs with identical public behavior but different hidden behavior collide.  Semantic embeddings can create subtler collisions.  The margin is intrinsic to the specified representation, not to unparsed natural-language meaning.

Mitigation requires parser audits, richer task-specific representations, and external anchors.  It cannot be solved by adding a theorem assumption that ``the embedding is semantic.''

\paragraph{Transport validity}
A typed transformation must preserve or predictably change the correct answer.  Ambiguous paraphrases, overflow-sensitive code refactors, or transformations that alter pragmatic context violate this requirement.  Such errors enter as operator or target misspecification and can make a correct model appear inconsistent.

The artifact uses algebraically generated math transformations and execution-tested code specifications.  Natural-language applications would need human or formal validation of relations.  Equation~\eqref{eq:operator-perturb} handles small operator error, not arbitrary semantic invalidity.

\paragraph{Anchor validity and dependence}
Theorems assume that anchor noise is bounded relative to a target field.  A biased or adversarial anchor can stably identify the wrong field.  A self-verifier generated by the same LLM may share errors with the responses and should not be modeled as independent trusted truth without evidence.  The experiments use deterministic arithmetic and reference execution outputs, unusually strong anchors.

Real applications may have soft, correlated, or strategic verifiers.  Extending the theory to jointly uncertain anchors and fields is important future work.

\paragraph{Sparse node-corruption model}
The group-sparse model is suitable when a small number of transformed responses fail while most are compatible with one truth field.  It is unsuitable for diffuse calibration bias, widespread prompt misunderstanding, or shared hallucinations affecting all nodes.  The impossibility theorem explains the last case rather than solving it.

Approximate sparsity gives an oracle tail term, but a dense error aligned with $\ker D$ can still dominate.  One should report empirical wrong-node counts and budget sensitivity rather than assume $k=1$ universally.

\paragraph{Worst-case versus typical-case difficulty}
$\gamma_k$ is a worst-case restricted margin.  A small value means some admissible direction is difficult; it does not imply that the model often makes that error.  A large value gives uniform stability under the model assumptions; it does not guarantee that parsing, candidate generation, or anchors are correct.  Cross-model prediction is therefore an empirical question about alignment between natural errors and difficult directions.

The real-error replay deliberately enforces the sparse model and uses authentic error values, providing a theorem-aligned stress test.  Natural-field results are weaker evidence because they include candidate absence and budget violations.

\paragraph{Finite-dimensional linearity}
Global theorems assume finite-dimensional linear operators.  Discrete candidate spaces are finite but not linear; we use the linear signature space for margins and a candidate projection for outputs.  Nonlinear residuals receive only local guarantees with explicit curvature.  General text-to-text semantic transport is neither proven linear nor globally identifiable.

Infinite-dimensional extensions could use restricted lower bounds on unions of subspaces, but compactness and attainment arguments require care.  They are outside this paper.

\paragraph{Exact margin computation}
Computing $\gamma_k$ by support enumeration is exponential in $k\log n$.  The paper's exact claims are experimentally feasible because fields are small.  Large agent graphs require lower certificates, coherence bounds, branch-and-bound, or task-specific structure.  A heuristic estimate should be labeled as such.

This computational limitation does not invalidate the information-theoretic object, but it limits routine deployment and adaptive design at scale.

\paragraph{Normalization dependence}
Singular values depend on units and weights.  Cross-task comparison is meaningful only after specifying domain/codomain metrics and relation-family normalization.  The paper canonicalizes math units, robustly scales code signature coordinates, and separates duplicate family weight.  Different defensible metrics can produce different numerical margins.

The invariant claim is conditional: under a fixed, declared norm and weighting, $\gamma_k$ is the exact restricted condition number.  There is no unit-free scalar without additional structure.

\subsection{Empirical and operational limitations}

\paragraph{Candidate recall}
Discrete repair cannot output a program or answer absent from its candidate closure unless it has a generative edit step.  High margin cannot overcome zero candidate recall.  The paper reports candidate recall and evaluates field reconstruction separately from candidate selection.

A future generative RRF algorithm could ask the LLM to instantiate a repaired signature, but then generation introduces a new stochastic channel and may fail even when the target field is identified.

\paragraph{Public versus hidden tests}
Code anchors are only as complete as public tests.  Hidden tests evaluate generalization beyond the observed signature.  Passing public tests does not prove program correctness.  The margin certifies recovery in public signature space, not semantic equivalence over all inputs.

This limitation is fundamental to testing, not unique to RRF.  Formal verification or exhaustive finite-domain execution would provide stronger anchors when available.

\paragraph{Small models and small tasks}
The experiments use two open checkpoints and 192 generated tasks.  They are designed for mechanistic clarity and reproducibility on one 4090, not benchmark leadership.  Results may differ for frontier APIs, long-form reasoning, multilingual tasks, or domains without exact anchors.

Cross-model transfer across two families is evidence, not universality.  The paper avoids claiming that one coefficient applies to every model or task.

\paragraph{Compute and comparison fairness}
Field methods use four transformed queries, whereas raw uses one.  Reflection and self-consistency have different sequential and sampling budgets.  The paper reports these budgets and treats baseline accuracy as context.  Its primary claim concerns difficulty prediction under fixed logs, where all operator designs reuse the same generations.

No compute-optimality claim is made.

\paragraph{Dataset generation}
Synthetic arithmetic and code tasks permit exact labels and transformations but may be easier, more regular, or less linguistically diverse than public benchmarks.  They reduce contamination and relation ambiguity at the cost of ecological breadth.  Repeating the study on GSM8K, MATH, HumanEval, and MBPP would require carefully validated transformations and licensing-aware artifact packaging.

\subsection{Interpretive boundaries, security, nonclaims, and future work}

\paragraph{Statistical dependence}
Ten operator designs reuse each response log, so rows are not independent.  Grouped cross-validation, within-item permutation, and cluster bootstrap address this design.  They do not correct for every analysis choice.  The protocol and outcomes are specified in code and supplement, but the study is not externally preregistered.

\paragraph{Causal interpretation}
Operator design is manipulated within an item, which supports a causal interpretation for those designs under deterministic replay.  However, $\gamma_k$ is a function of the design, and designs can affect algorithms through features not summarized by the scalar margin.  The minimax theorem establishes a fundamental bound; empirical coefficient estimates do not prove that all performance differences are mediated only by $\gamma_k$.

\paragraph{Security}
Executing generated code is risky.  The AST restrictions and timers reduce risk for tiny integer functions but are not a hardened sandbox.  Production reproduction should use containers, syscall restrictions, no network, read-only filesystems, and resource quotas.

\paragraph{Bibliographic scope}
The supplement surveys the closest mathematical and LLM-method literature but cannot cover every graph-consistency, testing, uncertainty, or inverse-problem paper.  The novelty claim is framed to avoid priority claims about individual ingredients.

\paragraph{Nonclaims}
For clarity, the paper does not claim:
\begin{itemize}
  \item that consistency implies truth;
  \item that $\gamma_k$ is an intrinsic scalar of a model independent of task and representation;
  \item that positive $\gamma_k$ alone guarantees group-Lasso recovery;
  \item that all natural-language transformations are valid or linear;
  \item that a large margin corrects biased anchors or missing candidates;
  \item that the small experiments establish state-of-the-art benchmark accuracy;
  \item that restricted singular values, graph Laplacians, sheaves, or metamorphic testing are newly invented.
\end{itemize}

\paragraph{Future theoretical directions}
Important extensions include probabilistic error priors coupled to worst-case geometry, robust design with uncertain transports, nonasymptotic estimation of reliability weights, approximate margin certificates for large graphs, adversarial anchors, nonlinear global identifiability, and sequential query policies with stopping guarantees.  Another open question is whether natural LLM error distributions concentrate near a low-dimensional subset of the worst restricted directions; such a result could connect minimax and typical-case repair.

\section{Probabilistic noise, approximate sparsity, and weighted margins}
\label{app:probabilistic}

The main results are deterministic because bounded-noise statements expose the information geometry without committing to a data-generating distribution.  This appendix connects those statements to probabilistic observation models.  It also records the modifications required by heteroscedastic measurements, approximately sparse response failures, and operator error.  These extensions are useful for readers accustomed to statistical inverse problems, but none is needed for the exact identifiability theorem.

\subsection{Random-noise radii and weighted observability geometry}

\paragraph{From a random noise law to a deterministic radius}
Let $B=[D;A]:\cH\to\RR^m$ and observe
\begin{equation}
  s=Be^\star+\xi, \qquad \xi\sim\mathcal N(0,\sigma^2 I_m).
  \label{eq:gaussian-observation}
\end{equation}
The following statement is a direct probabilistic corollary of Theorem~\ref{thm:stability}; its proof is included to make the confidence parameter and dimension dependence explicit.

\begin{proposition}[Gaussian high-probability recovery]
\label{prop:gaussian-hp}
Fix $\delta\in(0,1)$ and define
\[
  \eps_\delta=\sigma\bigl(\sqrt m+\sqrt{2\log(1/\delta)}\bigr).
\]
Suppose $\|e^\star\|_{0,\mathrm g}\le k$, $\gk>0$, and $\widehat e$ is any $k$-group-sparse vector satisfying $\|B\widehat e-s\|_2\le\eps_\delta$.  Then, with probability at least $1-\delta$,
\[
  \|\widehat e-e^\star\|_2
  \le \frac{2\sigma}{\gk}
  \bigl(\sqrt m+\sqrt{2\log(1/\delta)}\bigr).
\]
\end{proposition}

\begin{proof}
Write $\xi=\sigma g$ with $g\sim\mathcal N(0,I_m)$.  The map $g\mapsto\|g\|_2$ is one-Lipschitz, and $\EE\|g\|_2\le\sqrt{\EE\|g\|_2^2}=\sqrt m$.  Gaussian concentration therefore gives
\[
  \PP\{\|g\|_2>\sqrt m+t\}\le e^{-t^2/2}.
\]
Set $t=\sqrt{2\log(1/\delta)}$.  On the resulting event, both $e^\star$ and $\widehat e$ are feasible for the radius-$\eps_\delta$ constraint.  Their difference is supported on at most $2k$ groups, and
\[
 \gk\|\widehat e-e^\star\|_2
 \le\|B(\widehat e-e^\star)\|_2
 \le\|B\widehat e-s\|_2+\|s-Be^\star\|_2
 \le 2\eps_\delta.
\]
Dividing by the positive margin proves the result.
\end{proof}

The factor $\sqrt m$ should not be read as a penalty for adding arbitrary normalized relations.  If an experiment adds rows while fixing a total measurement-energy budget, the coordinate variance or row weights must be rescaled as part of the observation model.  Otherwise both signal energy and noise energy change.  The correct comparison uses the whitened operator below.

For independent centered sub-Gaussian coordinates with proxy variance $\sigma^2$, a norm concentration inequality gives the same form up to a universal constant.  For heavy-tailed coordinates, an $\ell_2$ radius may be inappropriate: robust residual aggregation or coordinate clipping is required before the deterministic theorem can be invoked.  The restricted margin itself does not turn a heavy-tailed noise law into a bounded one.

\paragraph{Heteroscedastic and correlated measurements}
Relation residuals and anchors rarely have equal reliability.  Suppose
\[
  \xi\sim\mathcal N(0,\Sigma), \qquad \Sigma\succ0.
\]
Whitening gives $\widetilde s=\Sigma^{-1/2}s$, $\widetilde B=\Sigma^{-1/2}B$, and standard Gaussian noise.  This motivates the covariance-aware margin
\begin{equation}
  \gamma_{k,\Sigma}(B)
  :=\min_{0<\|h\|_{0,\mathrm g}\le2k}
       \frac{\|\Sigma^{-1/2}Bh\|_2}{\|h\|_2}
  =\min_{|S|\le2k}\sigma_{\min}(\Sigma^{-1/2}B_S).
  \label{eq:weighted-gamma}
\end{equation}
All deterministic identifiability and stability arguments apply verbatim to $\widetilde B$.  In particular, correlation between two anchor measurements prevents them from being counted as two independent units of information: near-collinearity in $\Sigma$ is removed by whitening.  If $\Sigma$ is singular because some residuals are exact linear copies, one first restricts to the support of $\Sigma$ and separately retains deterministic noiseless constraints.

The covariance must be estimated without using test correctness.  A defensible protocol estimates it from repeated calibration queries, freezes a regularized estimate
\[
  \widehat\Sigma_\tau=(1-\tau)\widehat\Sigma+\tau
       \frac{\operatorname{tr}(\widehat\Sigma)}m I,
\]
and reports sensitivity to $\tau$.  Treating a noisy LLM judge as an exact anchor corresponds to assigning zero variance and can make the weighted margin arbitrarily optimistic.  A floor on all estimated variances prevents that pathology and states the assumed reliability in observable units.

\subsection{Gaussian minimax bounds and approximate sparsity}

\paragraph{A Gaussian two-point lower bound}
The bounded-noise minimax theorem shows a worst-case $1/\gamma_k$ law.  A related dependence persists under Gaussian noise.  To avoid hiding constants, we give a self-contained two-point statement.  Let
\[
  \Theta_k(R)=\{e\in\cH:\|e\|_{0,\mathrm g}\le k,\ \|e\|_2\le R\}
\]
and let $\mathfrak R_G(B,k,R,\sigma)$ be the minimax expected $\ell_2$ error under $s\sim\mathcal N(Be,\sigma^2I)$.

\begin{proposition}[Gaussian local minimax obstruction]
\label{prop:gaussian-minimax}
Assume the minimum defining $\gamma_k(B)$ is attained, as it is in finite dimensions.  Then
\[
  \mathfrak R_G(B,k,R,\sigma)
  \ge \frac18\min\left\{R,\frac{\sigma}{\gamma_k(B)}\right\}.
\]
When $\gamma_k(B)=0$, the convention $\sigma/0=+\infty$ yields the lower bound $R/8$.
\end{proposition}

\begin{proof}
First suppose $\gamma_k(B)>0$.  Choose a unit vector $h$ supported on at most $2k$ groups with $\|Bh\|_2=\gamma_k(B)$.  Partition its support into disjoint sets $S_1,S_2$, each of size at most $k$, and write $h=h_{S_1}-(-h_{S_2})$.  Let
\[
 t=\min\{R,\sigma/\gamma_k(B)\},\qquad
 e_1=t h_{S_1},\qquad e_2=-t h_{S_2}.
\]
Both parameters belong to $\Theta_k(R)$ because restriction cannot increase the norm, and $\|e_1-e_2\|_2=t$.  For $P_j=\mathcal N(Be_j,\sigma^2I)$,
\[
  \mathrm{KL}(P_1\|P_2)
  =\frac{\|B(e_1-e_2)\|_2^2}{2\sigma^2}
  =\frac{t^2\gamma_k(B)^2}{2\sigma^2}\le\frac12.
\]
Pinsker's inequality gives $\mathrm{TV}(P_1,P_2)\le1/2$.  The standard two-point reduction follows directly from the triangle inequality: for every estimator, at least one of the two parameter risks is at least
\[
  \frac{\|e_1-e_2\|_2}{4}
  \bigl(1-\mathrm{TV}(P_1,P_2)\bigr)\ge\frac t8.
\]
If $\gamma_k(B)=0$, choose a nonzero $2k$-group-sparse $h\in\ker B$, normalize it, and use $t=R$.  The two distributions are then identical and the same argument applies with total variation zero.
\end{proof}

The proposition is deliberately local and constant-level.  Sharper Gaussian minimax rates can depend on the number of admissible supports and on the full restricted spectrum, not only its smallest value.  The point relevant here is narrower: even under smooth stochastic noise, no estimator removes the inverse dependence on the least observable sparse direction.

\paragraph{Approximately sparse response failures}
A natural field can contain one dominant failure and many small discrepancies.  Let $e^\star=e_k+r$, where $e_k$ is $k$-group-sparse and $r$ is a residual tail.  The observation is $s=Be^\star+\xi$.  A $k$-sparse decoder should be compared with $e_k$, because exact recovery of a dense $e^\star$ is outside its model class.

\begin{proposition}[Oracle approximation inequality]
\label{prop:approx-sparse}
Suppose $\|\xi\|_2\le\eps$, $\gk>0$, and $\widehat e$ is $k$-group-sparse with
\[
  \|B\widehat e-s\|_2\le \eps+\|Br\|_2.
\]
Then
\[
  \|\widehat e-e^\star\|_2
  \le \|r\|_2+
       \frac{2(\eps+\|Br\|_2)}{\gk}.
  \label{eq:approx-sparse-bound}
\]
\end{proposition}

\begin{proof}
The sparse approximation $e_k$ is feasible because
\[
 \|Be_k-s\|_2=\|Br+\xi\|_2\le\|Br\|_2+\eps.
\]
Both $e_k$ and $\widehat e$ are $k$-group-sparse, so the deterministic stability argument gives
\[
 \|\widehat e-e_k\|_2
 \le2(\eps+\|Br\|_2)/\gk.
\]
Add $\|e_k-e^\star\|_2=\|r\|_2$ by the triangle inequality.
\end{proof}

The bound separates two effects.  The unavoidable approximation error is $\|r\|_2$; the observable tail $Br$ also behaves as additional measurement noise.  A tail lying mostly in $\ker B$ can have small $\|Br\|_2$ yet large semantic norm, reminding us that the chosen representation and norm determine what counts as a small error.

For the group-$\ell_1$ decoder, standard robust-null-space arguments replace $\|r\|_2$ by a best-$k$ group approximation term such as $\sigma_k(e^\star)_{2,1}/\sqrt{k}$.  That stronger conclusion requires the group robust null-space property and does not follow from $\gamma_k>0$ alone.

\subsection{Operator uncertainty, calibration, and scope}

\paragraph{Operator error and calibration uncertainty}
Let $B$ be the operator used by the decoder but suppose the actual measurement map is $B+\Delta$:
\[
  s=(B+\Delta)e^\star+\xi.
\]
If $\|\Delta\|_{2\to2}\le\eta$ and $\|e^\star\|_2\le R$, then the nominal model sees effective noise $\widetilde\xi=\Delta e^\star+\xi$ with
\[
  \|\widetilde\xi\|_2\le\eps+\eta R.
\]
Consequently every deterministic stability statement remains valid after replacing $\eps$ by $\eps+\eta R$.  This elementary reduction is important in black-box evaluation: a parser or transport estimated from finite calibration data is part of the observation channel, not free side information.

There is a second effect.  The true margin can differ from the reported nominal margin.  For any support $S$ with $|S|\le2k$, Weyl's singular-value perturbation inequality gives
\[
 \bigl|\sigma_{\min}((B+\Delta)_S)-\sigma_{\min}(B_S)\bigr|
 \le\|\Delta_S\|_{2\to2}\le\eta.
\]
Taking minima over supports yields
\begin{equation}
  \gamma_k(B+\Delta)\ge\max\{\gamma_k(B)-\eta,0\}.
  \label{eq:gamma-operator-uncertainty}
\end{equation}
Thus a nominal certificate smaller than the operator uncertainty is not robustly positive.  A conservative artifact should report both $\widehat\gamma_k$ and an uncertainty radius, or report the lower certificate $(\widehat\gamma_k-\widehat\eta)_+$.

\paragraph{What these extensions do and do not establish}
The deterministic margin is the common geometric quantity in all results above.  Probability enters only through a noise radius, covariance metric, or distributional indistinguishability argument.  This separation is useful: the same $B$ can be studied under adversarial, Gaussian, or empirical noise without redefining its restricted observability.

The extensions do not imply that observed LLM errors are Gaussian, independent, or exactly sparse.  Those assumptions are diagnostics to be checked.  The real-model experiment therefore reports parser failures, candidate recall, hidden-test-only code bugs, and replay results separately.  A high margin cannot repair a missing candidate, validate an incorrect transport, or convert a dense semantic failure into a sparse one.  It quantifies the difficulty of the recovery problem that has actually been specified.

\section{Statistical validation of a cross-model, cross-task difficulty measure}
\label{app:statistical-validation}

The mathematical results establish what $\gamma_k(D,A)$ means for a specified observation problem.  They do not by themselves establish that the object explains variation in naturally occurring LLM failures.  That empirical claim requires a design that separates operator geometry from model competence, item difficulty, candidate availability, and repeated measurements of the same item.  This appendix gives the complete inferential protocol used for the authentic-error replay and states which conclusions each statistic can support.

\subsection{Estimands and complementary evaluation regimes}

\paragraph{Units, indices, and estimands}
Let $m$ index a model, $t$ a task domain, $i$ an underlying dataset item, and $d$ a relation--anchor design.  One replay field is formed by taking correct responses for an item and replacing exactly one node by an authentic incorrect response produced by the same model on the same task family.  The replacement is not synthesized by a numerical noise model.  It retains the model's actual arithmetic answer or executable program, subject to the replay eligibility rules in Appendix~\ref{app:reasoning-protocol}.

For each tuple $(m,t,i,d)$ we record
\[
  G_{mtid}=\gamma_1(D_{td},A_{td}),\qquad
  E_{mtid}=\frac{\|\widehat e_{mtid}-e^\star_{mti}\|_2}
                  {\max\{\|e^\star_{mti}\|_2,10^{-12}\}},
\]
and the exact-recovery indicator
\[
  Y_{mtid}=\ind\{\widehat z_{mtid}=z^\star_{mti}\}.
\]
The exact equality is equality in the task's audited semantic representation: integer equality for mathematics and equality of the complete hidden execution signature for code.  Source-string equality is not required.

The primary empirical estimand is conditional association: after holding a model--task stratum and an underlying item fixed, do designs with larger $G$ make the same authentic corruption easier to recover?  This is intentionally narrower than claiming that $G$ alone predicts raw accuracy across arbitrary benchmarks.  Raw model competence affects whether a usable correct candidate exists and how often the sparse-error model applies.

\paragraph{Why natural evaluation and replay are both necessary}
The natural-generation table asks a pipeline question.  It includes generation, parsing, candidate construction, relation checking, selection, and final task correctness.  A failure at any stage lowers performance.  That is the operational quantity a practitioner cares about, but it is a noisy test of the information-theoretic theorem.

The replay table asks a mechanism question under theorem-aligned conditions.  It guarantees exactly one corrupted node, constructs the corruption from a real model error, and varies the observable design.  Correct candidates are present by construction, so the experiment isolates whether the relation--anchor operator makes the error distinguishable.  Reporting only replay would conceal candidate-generation limitations; reporting only natural accuracy would make a failed parser look like a counterexample to sparse identifiability.  The two tables answer different questions and are not merged into one success rate.

For code, public and hidden execution signatures add another separation.  The public signature defines observable relations and anchors.  Hidden tests define evaluation correctness.  A program that agrees on every public test but fails a hidden test is an observationally invisible bug under that design.  Such cases diagnose anchor insufficiency rather than decoder failure.  We retain their counts, but replay eligibility requires a wrong public signature when the purpose is to test sparse localization.

\subsection{Rank, permutation, predictive, and hierarchical analyses}

\paragraph{Within-stratum rank association}
Absolute margins are comparable only after fixing fiber scaling and row normalization.  We impose a common normalization within each task construction and first compute Spearman association separately for every model--task stratum:
\begin{equation}
  \widehat\rho_{mt}
  =\operatorname{corr}\bigl(
       \operatorname{rank}(G_{mtid}),
       \operatorname{rank}(-E_{mtid})
    \bigr).
  \label{eq:stratum-spearman}
\end{equation}
Average ranks are used for ties.  A positive value means that better conditioned designs tend to have lower error; it does not assume a linear dose--response curve.  The pooled descriptive correlation is also reported, but the stratum values carry the cross-domain claim because pooling can create Simpson reversals through different base error rates.

Repeated designs from one item are dependent.  Treating all $(i,d)$ rows as independent would understate uncertainty.  Confidence intervals therefore use a cluster bootstrap: within every model--task stratum, sample items with replacement, carry all ten designs of each selected item together, recompute the statistic, and take percentile endpoints over bootstrap replicates.  The resampling seed and number of replicates are stored in the result JSON.

\paragraph{A within-item permutation test}
The null tested by the permutation procedure is that, conditional on an item and the multiset of observed design outcomes, assignment of margin labels to design outcomes contains no association.  For each permutation $b$, independently shuffle the ten $G$ values among the designs of every item, preserving the model, task, corruption, and outcome values.  Let $T_0$ be the observed pooled rank statistic and $T_b$ its permuted value.  The one-sided Monte Carlo p-value is
\begin{equation}
  \widehat p=\frac{1+\sum_{b=1}^{B_{\mathrm{perm}}}
      \ind\{T_b\ge T_0\}}{B_{\mathrm{perm}}+1}.
  \label{eq:permutation-p}
\end{equation}
The add-one correction prevents a zero p-value and is valid for Monte Carlo randomization.  Because labels move only within item, the test cannot be driven by one model receiving easier items or one task having larger numerical error scales.

The permutation test has a limitation: design outcomes are deterministic functions of a generated field, not randomized interventions collected prospectively.  Exchangeability is therefore a reference null for association, not a causal identification theorem.  The strongest causal statement in the paper remains mathematical monotonicity under adding rows to a fixed operator, where the intervention is defined algebraically.

\paragraph{Grouped cross-validation and incremental prediction}
To test whether $G$ adds predictive information beyond coarse nuisance variables, we compare two pre-specified logistic models for exact recovery.  The control model uses model identity, task identity, design row count, anchor count, and an observable raw-disagreement measure.  The augmented model adds standardized $G$.  Standardization parameters are learned on each training fold and applied unchanged to its held-out fold.

All rows belonging to one item are assigned to the same fold.  Let $\mathcal I_f$ be held-out items in fold $f$.  Parameters are estimated on items outside $\mathcal I_f$, predictions are emitted only for $\mathcal I_f$, and the out-of-fold predictions are concatenated before computing area under the receiver-operating-characteristic curve.  No reported AUC is a training AUC.  The incremental statistic is
\[
  \Delta\mathrm{AUC}
  =\mathrm{AUC}_{\mathrm{controls}+\gamma}
   -\mathrm{AUC}_{\mathrm{controls}}.
\]
Keeping the item grouped is essential: random row splits would let the model infer an item's outcome from nine nearly identical siblings and predict the tenth.

The fitted coefficient is not interpreted per raw unit because margins depend on normalization.  We standardize the transformed margin using training-fold mean and variance, then apply those training statistics to the held-out fold.  A positive out-of-fold coefficient and positive $\Delta\mathrm{AUC}$ support incremental prediction.  They do not show that the logistic link is the true recovery law.

\paragraph{Hierarchical sensitivity model}
As a sensitivity analysis, one may fit a mixed-effects specification
\begin{equation}
  \operatorname{logit}\PP(Y_{mtid}=1)
  =\alpha+u_m+v_t+w_i+\beta\widetilde G_{td}
   +\theta^\top X_{mtid},
  \label{eq:mixed-recovery}
\end{equation}
where $u_m$, $v_t$, and $w_i$ are model, task, and item intercepts and $X$ contains observable controls.  With only two models and two task domains, random-effect asymptotics for $u_m$ and $v_t$ are weak; fixed stratum effects plus item-clustered uncertainty are safer in the present experiment.  Equation~\eqref{eq:mixed-recovery} is included to specify how the test should scale when more open-weight models and domains become available.

A continuous-error sensitivity model regresses $\log(E+\epsilon_E)$ on $\log(G+\epsilon_G)$ with item fixed effects.  Because zero-margin designs can produce exact recovery on favorable instances, a hurdle model separating $G=0$ from $G>0$ is preferable to forcing one linear slope through the discontinuity.  Rank statistics remain the primary analysis because they require fewer functional assumptions.

\subsection{Confounds, falsification, uncertainty, and interpretation}

\paragraph{Confounds and the controls that address them}
\paragraph{More rows.}
A larger margin may accompany more relation rows.  Row count is included as a control, and normalized duplicate families provide a negative control: they increase count without changing $B^\top B$ or $\gamma_k$.

\paragraph{More anchors.}
Anchors can simultaneously raise the margin and directly reveal truth.  This is not a spurious relationship; anchors are part of the observation problem.  Nevertheless, anchor count is controlled so the statistic tests more than the binary presence of an anchor.  Designs with equal counts but different placement test coverage geometry.

\paragraph{Item difficulty.}
All design comparisons reuse the same generated field.  Item grouping in permutation, bootstrap, and cross-validation prevents between-item competence from being mistaken for a spectral effect.

\paragraph{Model competence.}
Associations are reported within model--task strata before pooling.  A stronger model may produce fewer replay-eligible errors, which changes precision and selection, but it cannot by itself create a within-item design ordering.

\paragraph{Parser success.}
Natural accuracy uses all generated rows and reports parse rate.  Replay uses only fields for which the typed semantic representation is defined.  The eligibility count is therefore part of the result, not an invisible preprocessing loss.

\paragraph{Candidate recall.}
A selector cannot return an absent truth.  Natural results report candidate recall, and replay fixes recall by construction.  Improvements in replay recovery cannot be attributed to generating more candidates for high-margin designs because the candidate set is held fixed across designs.

\paragraph{Negative controls and falsification checks}
Four negative controls are built into the artifact.
\begin{enumerate}
  \item \textbf{Consistency-only fields.} A transported common-mode error has zero relation defect.  Any method using only $D$ should fail to distinguish it from truth.
  \item \textbf{Uncovered components.} A design with an unanchored sparse component has $\gamma_k=0$.  Recovery may succeed on some nodes, but uniform exact recovery should fail for corruptions inside the invisible component.
  \item \textbf{Normalized duplicates.} Repeating an identical relation with weights divided by the square root of the multiplicity preserves $B^\top B$ and the margin.  An implementation reporting spectral improvement fails this audit.
  \item \textbf{Shuffled margins.} Within-item permutation destroys the alignment between design geometry and outcomes while preserving all marginal distributions.  Predictive improvement should disappear under this shuffle.
\end{enumerate}

The theory would face a substantive empirical challenge if the observed statistic were indistinguishable from its shuffled reference in every model--task stratum, if its sign reversed systematically across domains after normalization, or if row-count controls explained all held-out gain.  A single favorable correlation is not enough; direction, grouping, negative controls, and out-of-fold performance must agree.

\paragraph{Multiplicity, uncertainty, and reporting discipline}
The primary outcomes and analyses are fixed as follows: exact replay recovery, normalized replay error, stratum Spearman correlations, one within-item permutation test, one item-cluster bootstrap interval, and one grouped cross-validation AUC comparison.  Family-level tables, alternative error normalizations, and hidden-only code analyses are secondary.  This ordering prevents selecting the most favorable statistic after seeing results.

The artifact stores full-precision numbers; the paper rounds only for readability.  A p-value is reported with its Monte Carlo resolution $1/(B_{\mathrm{perm}}+1)$.  Confidence intervals describe sampling variation over the finite item construction, not uncertainty over all possible prompts, models, or future benchmarks.  Model decoding is seeded, but deterministic GPU kernels are not assumed unless the environment provides them; raw generations are therefore included so every downstream statistic can be reproduced without regenerating text.

\paragraph{Interpretation ladder}
The empirical evidence supports claims at three distinct levels.
\begin{enumerate}
  \item \textbf{Mechanism verification:} anchor coverage, duplicate normalization, and error-versus-margin trends behave as the linear theory predicts on controlled instances.
  \item \textbf{Cross-stratum prediction:} the same normalized object orders replay difficulty within multiple model--task strata and improves grouped held-out discrimination.
  \item \textbf{External scope:} the result suggests, but does not prove, usefulness for other models, tasks, semantic representations, or natural dense failures.
\end{enumerate}
Only the first level follows tightly from the theorem assumptions.  The second is the paper's decisive evidence for treating $\gamma_k(D,A)$ as more than a method-specific score.  The third remains a program for subsequent work.  Maintaining this ladder prevents a small but controlled experiment from being narrated as a universal benchmark law.

\section{Exact margins and anchor geometry for canonical graph designs}
\label{app:graph-derivations}

This appendix derives $\gamma_k$ for elementary response fields without relying on numerical singular-value routines.  The calculations explain the anchor phase transition and redundancy saturation experiments and clarify a subtle point: global injectivity, sparse injectivity, and numerical conditioning are different properties.

\subsection{Laplacian formulation and closed-form canonical examples}

\paragraph{Identity transports and anchored graph Laplacians}
Let every node fiber be $\RR^d$, every transport be the identity, and every relation edge $e=(i,j)$ have weight $w_e>0$.  Choose an arbitrary orientation and let $C\in\RR^{|E|\times n}$ be the weighted incidence matrix, whose row for $e$ is $\sqrt{w_e}(e_j-e_i)^\top$.  Then
\[
  D=C\otimes I_d,
  \qquad D^\top D=L_w\otimes I_d,
\]
where $L_w=C^\top C$ is the weighted graph Laplacian.  Suppose a scalar anchor of strength $a_i\ge0$ acts isotropically on node $i$.  With $Q=\diag(a_1^2,\ldots,a_n^2)$,
\begin{equation}
  B^\top B=(L_w+Q)\otimes I_d.
  \label{eq:anchored-laplacian}
\end{equation}

For a node set $S$, restriction of columns gives
\[
  B_S^\top B_S=((L_w+Q)_{SS})\otimes I_d.
\]
Consequently
\begin{equation}
  \gamma_k(D,A)^2
  =\min_{\varnothing\ne S\subseteq V,\ |S|\le2k}
       \lambda_{\min}((L_w+Q)_{SS}).
  \label{eq:gamma-principal-laplacian}
\end{equation}
This is an exact principal-submatrix formula, not a graph-level heuristic.  It also shows why the smallest eigenvalue of the full anchored Laplacian can be overly pessimistic when $k\ll n$: $\gamma_k$ minimizes only over small principal supports.

\begin{proposition}[Null space and sparse anchor coverage]
\label{prop:sparse-component-criterion}
Let the graph have connected components $V_1,\ldots,V_c$.  A component is anchored if $a_i>0$ for at least one $i$ in that component.  Then:
\begin{enumerate}
  \item $B$ is globally injective if and only if every component is anchored;
  \item $\gamma_k(D,A)=0$ if and only if some unanchored component has at most $2k$ nodes.
\end{enumerate}
The statements hold for every fiber dimension $d\ge1$.
\end{proposition}

\begin{proof}
Because $\|Bx\|_2^2=\sum_{(i,j)\in E}w_{ij}\|x_i-x_j\|_2^2+
\sum_i a_i^2\|x_i\|_2^2$, a vector lies in $\ker B$ exactly when it is constant on each connected component and its constant is zero on every anchored component.  Thus the null space contains one copy of $\RR^d$ for every unanchored component, proving the first claim.

Every nonzero null vector is constant and nonzero on at least one complete unanchored component.  Its group support therefore contains all nodes of that component.  The smallest possible nonzero null support is the size of the smallest unanchored component.  By the exact identifiability equivalence, $\gamma_k=0$ exactly when that support size is at most $2k$.
\end{proof}

This proposition refines the informal slogan ``one anchor per component.''  One anchor per component is necessary and sufficient for recovery without a sparsity restriction.  For $k$-sparse recovery, a large unanchored component can have no $2k$-sparse null vector, so $\gamma_k$ can be positive even though $B$ is globally singular.  The field is then identifiable only over the declared sparse class.

\paragraph{A two-node field in closed form}
Consider two scalar nodes joined by one identity relation and anchor node 1 with strength $a\ge0$:
\[
  B=\begin{bmatrix}-1&1\\a&0\end{bmatrix},\qquad
  B^\top B=\begin{bmatrix}1+a^2&-1\\-1&1\end{bmatrix}.
\]
For $k=1$, supports of size at most $2$ include the full field, so
\begin{equation}
  \gamma_1(D,A)^2
  =\frac{2+a^2-\sqrt{a^4+4}}{2}.
  \label{eq:two-node-gamma}
\end{equation}
At $a=0$ the margin is zero: adding the same scalar to both responses changes neither relation residual.  For every $a>0$ the margin is positive.  As $a\downarrow0$, a Taylor expansion gives $\gamma_1^2=a^2/2+O(a^4)$; hence recovery is identifiable but badly conditioned under a weak anchor.  As $a\to\infty$, $\gamma_1^2\to1$, so relation strength becomes the bottleneck.  Merely making an already strong anchor stronger cannot push the margin beyond the information carried through the edge.

The $d$-dimensional isotropic case repeats each eigenvalue $d$ times and has the same margin.  If the transport is an orthogonal matrix $T$, change variables at node 2 by $x_2\mapsto T^{-1}x_2$; singular values are unchanged.  A nonorthogonal transport changes the fiber metric and must be normalized before margins across relation types are compared.

\subsection{Redundancy, path geometry, and measurement design}

\paragraph{Duplicate relations under fixed information budget}
Suppose one relation row block is $R$.  Repeating it $r$ times without normalization creates
\[
  D_{\mathrm{raw}}=\begin{bmatrix}R\\ \vdots\\R\end{bmatrix},
  \qquad D_{\mathrm{raw}}^\top D_{\mathrm{raw}}=rR^\top R.
\]
This increases singular values by $\sqrt r$, but it also assumes $r$ independent measurements each with the original precision.  Prompt duplication does not justify that assumption when copies are deterministic or perfectly correlated.

Under a fixed total family weight, each copy receives coefficient $r^{-1/2}$:
\[
  D_{\mathrm{norm}}=\frac1{\sqrt r}
  \begin{bmatrix}R\\ \vdots\\R\end{bmatrix}.
\]
Then
\begin{equation}
  D_{\mathrm{norm}}^\top D_{\mathrm{norm}}=R^\top R.
  \label{eq:duplicate-gram}
\end{equation}
For any fixed anchor $A$, the stacked Gram matrix and every restricted principal Gram matrix are identical to those with one copy.  Therefore
\[
  \gamma_k(D_{\mathrm{norm}},A)=\gamma_k(R,A)
\]
for every $k$.  The synthetic artifact constructs the repeated matrix, rescales its rows, and recomputes the margin; it does not insert the one-copy value into the result table.

If copies have independent random noise, whitening may legitimately yield more information.  If their noise correlation is $\rho$, the gain lies between the independent and identical-copy extremes and is determined by the inverse covariance, not the raw number of prompts.  This is why Appendix~\ref{app:probabilistic} defines $\gamma_{k,\Sigma}$.

\paragraph{A path graph with one anchor}
Let $V=\{1,\ldots,n\}$, unit-weight edges connect consecutive nodes, and node 1 has anchor strength $a>0$.  For a scalar vector $x$,
\[
  \|Bx\|_2^2=a^2x_1^2+\sum_{i=1}^{n-1}(x_{i+1}-x_i)^2.
\]
Write $\Delta_i=x_{i+1}-x_i$.  Since
\[
 x_j=x_1+\sum_{i=1}^{j-1}\Delta_i,
\]
Cauchy--Schwarz gives
\[
 |x_j|^2\le2|x_1|^2+2(j-1)\sum_{i=1}^{j-1}|\Delta_i|^2
 \le2|x_1|^2+2n\sum_{i=1}^{n-1}|\Delta_i|^2.
\]
Summing over $j$ yields
\begin{equation}
 \|x\|_2^2
 \le 2n|x_1|^2+2n^2\sum_{i=1}^{n-1}|\Delta_i|^2
 \le C_{n,a}\|Bx\|_2^2,
 \quad C_{n,a}=\max\{2n/a^2,2n^2\}.
 \label{eq:path-poincare}
\end{equation}
Hence the global smallest singular value satisfies
\[
  \sigma_{\min}(B)\ge C_{n,a}^{-1/2}.
\]
The $n^{-1}$ scaling of this elementary bound captures the fact that an error far from the anchor is communicated through a long chain.  Stronger discrete Poincare inequalities improve constants but not the qualitative lesson: connectivity establishes visibility, while path length controls conditioning.

For $k\ll n$, formula~\eqref{eq:gamma-principal-laplacian} can give a larger restricted margin because no admissible witness occupies the entire path.  This illustrates why $k$ belongs in the definition of the difficulty object.  A relation design may be poor for dense field reconstruction yet adequate for one localized response failure.

\paragraph{Adding measurements and comparing anchor placements}
Let $B$ be an existing operator and let $C$ be any additional normalized measurement block.  For every support $S$,
\[
  [B;C]_S^\top[B;C]_S=B_S^\top B_S+C_S^\top C_S\succeq B_S^\top B_S.
\]
Therefore
\begin{equation}
  \gamma_k([B;C])\ge\gamma_k(B).
  \label{eq:measurement-monotonicity}
\end{equation}
Equality is possible when $C$ is blind to a minimizing witness of $B$.  The theorem justifies greedy query acquisition but also shows why counting new rows is insufficient.

Suppose $h^\star$ is a unit $2k$-sparse witness attaining $\|Bh^\star\|_2=\gamma_k(B)$.  For a candidate scalar measurement $c^\top$, the energy of this witness becomes
\[
  \|Bh^\star\|_2^2+|c^\top h^\star|^2.
\]
A candidate with large $|c^\top h^\star|$ attacks the current weakest direction.  Recomputing the exact margin after each candidate remains necessary because the minimizing support or vector can switch.  The witness score is a principled screening rule, not an exact submodular guarantee.

Anchor placement has the same geometry.  In an identity-transport component, anchoring a node on which $h^\star$ has negligible mass barely changes the current bottleneck; anchoring a high-mass node can increase it sharply.  Once that direction is lifted, another component or support can become limiting, producing the piecewise-smooth phase curves observed in exact enumeration.

\subsection{Typed transports and audit procedures}

\paragraph{Typed transports and gauge transformations}
The Laplacian formulas above use identity transports only for transparency.  Let every edge transport $T_{ij}$ be invertible and suppose there exist node maps $U_i$ such that
\[
  T_{ij}=U_jU_i^{-1}.
\]
This is a flat transport system.  Under the change of variables $\widetilde x_i=U_i^{-1}x_i$, relation consistency becomes $\widetilde x_j-\widetilde x_i=0$.  If every $U_i$ is orthogonal in the declared fiber metric, the change is an isometry and the restricted margin is exactly that of an identity-transport graph with transformed anchors.

If the $U_i$ are merely invertible, let
\[
  \kappa_- =\min_i\sigma_{\min}(U_i),\qquad
  \kappa_+ =\max_i\sigma_{\max}(U_i).
\]
The block-diagonal change of coordinates bounds norms by $\kappa_-\|\widetilde x\|\le\|x\|\le\kappa_+\|\widetilde x\|$.  Margins in the two coordinate systems can therefore differ by condition-number factors.  Comparing relation types without declaring fiber metrics can mistake a change of units for easier recovery.

Non-flat transports have nontrivial cycle holonomy: transporting around a cycle need not return the original vector.  Then a globally consistent nonzero field may not exist, and the kernel of $D$ is governed jointly by graph topology and transport products.  The general restricted singular-value definition remains valid; the simple component criterion does not.  This is precisely why the framework begins with typed linear operators rather than assuming every relation is equality.

\paragraph{Audit recipe for a new design}
For a finite candidate design, an exact audit proceeds as follows.
\begin{enumerate}
  \item Fix fiber coordinates, norms, transport maps, row weights, anchor maps, and $k$ before observing repair outcomes.
  \item Assemble $B=[D;A]$ using the same parser and normalization used at inference.
  \item Enumerate every nonempty node support $S$ with $|S|\le2k$ and compute $\sigma_{\min}(B_S)$ with a documented numerical tolerance.
  \item Store the minimizing support and right singular vector as a witness.  Verify directly that its group support and Rayleigh quotient match the report.
  \item For zero margins, construct the indistinguishable pair from a sparse null vector.  For positive margins, perturb along the witness to check the predicted $1/\gamma_k$ sensitivity.
  \item If rows are estimated or stochastic, repeat the computation after whitening and report an operator-uncertainty lower certificate.
\end{enumerate}

This recipe turns $\gamma_k$ from an abstract symbol into a falsifiable certificate.  Every value in the experiment artifact is accompanied by enough operator information to rerun these checks without model weights.

\end{document}